\documentclass[letterpaper]{article} % DO NOT CHANGE THIS
\usepackage{aaai2026}  % DO NOT CHANGE THIS
\usepackage{times}  % DO NOT CHANGE THIS
\usepackage{helvet}  % DO NOT CHANGE THIS
\usepackage{courier}  % DO NOT CHANGE THIS
\usepackage[hyphens]{url}  % DO NOT CHANGE THIS
\usepackage{graphicx} % DO NOT CHANGE THIS
\usepackage{natbib}  % DO NOT CHANGE THIS AND DO NOT ADD ANY OPTIONS TO IT
\usepackage{caption} % DO NOT CHANGE THIS AND DO NOT ADD ANY OPTIONS TO IT
\usepackage{newfloat}
\usepackage{listings}
\usepackage{amsmath}
\usepackage{amssymb}
\usepackage{mathtools}
\usepackage{amsthm}
\usepackage{graphicx}
\usepackage{multirow}
\usepackage{array}
\usepackage{booktabs}
\usepackage{listings}
\usepackage{makecell}
\usepackage{subcaption}
\usepackage{comment}
\usepackage{booktabs}
\usepackage{bm}
\usepackage{siunitx}
\usepackage{amsmath}
\DeclareMathOperator*{\argmin}{arg\,min}
\usepackage{enumitem}
\usepackage{caption}
\usepackage{subcaption}
\usepackage{placeins}
\usepackage{algorithm}      % 提供 algorithm 环境
\usepackage{algpseudocode}  % 提供 algorithmic 环境
\newtheorem{theorem}{Theorem} % 定义theorem环境
\theoremstyle{definition}
\newtheorem{definition}{Definition}
\usepackage{siunitx}
\usepackage{xr-hyper}

\usepackage{xr}
\theoremstyle{definition}
\newtheorem{lemma}{Lemma}
\DeclareCaptionStyle{ruled}{labelfont=normalfont,labelsep=colon,strut=off} % DO NOT CHANGE THIS
\floatstyle{ruled}
\newfloat{listing}{tb}{lst}{}
\floatname{listing}{Listing}
\makeatletter
\@ifundefined{isChecklistMainFile}{
	\newif\ifreproStandalone
	\reproStandalonetrue
}{
	\newif\ifreproStandalone
	\reproStandalonefalse
}
\makeatother

\ifreproStandalone

\usepackage{times}
\usepackage{helvet}
\usepackage{courier}
\usepackage{xcolor}
\title{Beyond MSE: Ordinal Cross-Entropy for Probabilistic Time Series Forecasting}
\author{
	Jieting Wang,
	Huimei Shi, 
	Feijiang Li,
	Xiaolei Shang\\
}
\affiliations{
	\textsuperscript{\rm 1}Institute of Big Data Science and Industry, Shanxi University \\
	\textsuperscript{\rm 2}Key Laboratory of Evolutionary Science Intelligence of Shanxi Province, Taiyuan, China \\
}

\usepackage{bibentry}
\begin{document}

\maketitle

\begin{abstract}
Time series forecasting is an important task that involves analyzing temporal dependencies and underlying patterns (such as trends, cyclicality, and seasonality) in historical data to predict future values or trends. Current deep learning-based forecasting models primarily employ Mean Squared Error (MSE) loss functions for regression modeling. Despite enabling direct value prediction, this method offers no uncertainty estimation and exhibits poor outlier robustness.
To address these limitations, we propose OCE-TS, a novel ordinal classification approach for time series forecasting that replaces MSE with Ordinal Cross-Entropy (OCE) loss, preserving prediction order while quantifying uncertainty through probability output. Specifically, OCE-TS begins by discretizing observed values into ordered intervals and deriving their probabilities via a parametric distribution as supervision signals. Using a simple linear model, we then predict probability distributions for each timestep. The OCE loss is computed between the cumulative distributions of predicted and ground-truth probabilities, explicitly preserving ordinal relationships among forecasted values.
Through theoretical analysis using influence functions, we establish that cross-entropy (CE) loss exhibits superior stability and outlier robustness compared to MSE loss.
Empirically, we compared OCE-TS with five baseline models—Autoformer, DLinear, iTransformer, TimeXer, and TimeBridge—on seven public time series datasets. Using MSE and Mean Absolute Error (MAE) as evaluation metrics, the results demonstrate that OCE-TS consistently outperforms benchmark models.
The codeis publicly available at: https://github.com/Shi-hm/OCE-TS.
\end{abstract}


\section{Introduction}
Long-term time series forecasting plays a vital role in practical applications such as industrial monitoring, traffic management, and financial risk control. Traditional forecasting methods including Autoregressive Integrated Moving Average (ARIMA) \cite{box2015time}, Error-Trend-Seasonal (ETS) models \cite{hyndman2008forecasting}, and Prophet \cite{taylor2018forecasting} typically require manual feature engineering and make certain assumptions about data patterns.

Recent advances in deep learning have significantly improved time series forecasting performance \cite{li2019enhancing}. Modern architectures can automatically learn complex temporal features and long-term dependencies, achieving state-of-the-art performance in various forecasting tasks \cite{lim2021temporal}. Particularly for long-term forecasting tasks, current research demonstrates that Transformer-based models \cite{zhou2021informer,wu2021autoformer}, modern Multilayer Perceptron (MLP) architectures \cite{zeng2023transformers}, and enhanced Long Short-Term Memory (LSTM) networks \cite{greff2017lstm} can all effectively extract multi-scale features from time series through end-to-end training. These developments have been further validated in comprehensive benchmarks \cite{zhou2022fedformer}, establishing neural methods as the new paradigm in time series forecasting.

Nevertheless, the predominant use of MSE loss in neural network-based forecasting introduces fundamental constraints: the regression framework provides no inherent uncertainty quantification, and the quadratic loss term amplifies the influence of anomalous observations.

Researchers have developed specialized approaches that each target specific aspects of the problem in regression. For uncertainty quantification, probabilistic methods like Gaussian Processes~\cite{seeger2004gaussian} offer principled Bayesian uncertainty estimates but struggle with computational scalability, while quantile regression~\cite{koenker2001quantile} provides prediction intervals but may violate natural ordering constraints. Deep generative models~\cite{kingma2013auto} can capture complex distributions but require substantial training data and careful tuning.
On the robustness front, both Huber loss~\cite{huber1992robust} and correntropy-based methods~\cite{liu2007correntropy} can effectively mitigate outlier sensitivity while lacking probabilistic uncertainty quantification. The adaptive weighting scheme~\cite{kendall2018multi}, despite enabling dynamic sample weighting, requires joint optimization of multiple interacting hyperparameters that critically determine its performance.
However, no existing method simultaneously addresses uncertainty quantification and robustness.

To this end, we propose OCE-TS as an integrated solution combining probabilistic uncertainty quantification with robust outlier handling while preserving ordinal relationships and maintaining computational efficiency. Since time series values exhibit meaningful ordinal structure, the proposed OCE-TS method reformulates the regression task as an ordinal classification problem.
Ordinal classification~\cite{gutierrez2015ordinal} is a specialized machine learning task designed to handle categorical labels with inherent ordinal relationships. Its core objective is to predict discrete yet ordered categories (e.g., user ratings 1-5 stars) while preserving the sequential relationships between categories. Through ordinal classification, we obtain probabilistic outputs for each predicted value, naturally quantifying prediction uncertainty.
This approach is also inspired by the concepts of stability and sample-wise reliability, which have been further explored in recent studies \cite{li2023fuzzy, wang2023rss,   wang2020learning,wang2025stabilizing}, helping to enhance robustness against noise and irregular patterns.

To optimize the ordinal classification model, OCE-TS uses ordinal cross-entropy (OCE) loss \cite{Niu2016ordinal,hu2010information,wang2022generalization}.
The OCE loss modifies conventional cross-entropy through order-preserving mechanisms. By evaluating discrepancies in cumulative distribution functions (CDFs) of predicted versus true labels, it guarantees that predictions adhere to the inherent hierarchy of ordinal categories.

The main contributions of this paper are as follows:
\begin{itemize}
	\item We propose OCE-TS, a novel approach that combines linear model with ordinal cross-entropy to solve time series value prediction problems while preserving temporal ordering constraints.
	\item We compare MSE (regression) and cross-entropy (classification) loss stability via influence function analysis, demonstrating  that classification achieves superior stability when predicted probabilities are uniformly distributed and feature representations exhibit isotropic property;
	\item We verify OCE-TS's effectiveness across multiple benchmarks, which provides empirical guarantees for a new learning paradigm for time series forecasting.
\end{itemize}

Proofs and supplementary experiments are provided in the Appendix.

\section{Related Work}
Building upon ordinal classification for time series forecasting, our work intersects with three key domains: distributional regression (already discussed in Introduction), time series forecast, and ordinal classification. We now survey the latter two research areas.
\subsection{Long-term time series forecasting}

Long-term time series forecasting remains highly challenging due to extended prediction horizons, error accumulation, and distribution drift.
Moreover, stability and sample-wise reliability \cite{li2019clustering, wang2022generalization} offer useful insights for improving forecasting robustness.
Recent research primarily focuses on four key directions.
First, for network architecture innovation, research emphasizes enhanced long-term dependency and periodicity modeling, Transformers \cite{vaswani2017attention} and their variants (e.g., Informer \cite{zhou2021informer}, Autoformer \cite{wu2021autoformer}, FEDformer \cite{zhou2022fedformer}, PatchTST \cite{nie2022time}) have become predominant. Meanwhile, lightweight models (e.g., DLinear \cite{zeng2023transformers}, LightTS \cite{campos2023lightts}) achieve comparable performance with lower computational costs in specific scenarios.
Second, for time series representation learning and knowledge distillation, self-supervised learning (e.g., TS2Vec \cite{yue2022ts2vec}, TNC \cite{tonekaboni2021unsupervised}) and knowledge distillation methods (e.g., TCN-Distill \cite{gao2021distilling}, TimeNet) are leveraged to enhance generalization capability and robustness while reducing dependence on labeled data.
Third, for variable and feature hierarchical modeling, methods explicitly capture inter-variable dependencies and feature importance, primarily through cross-variable attention mechanisms (e.g., MTGNN \cite{wu2020connecting}, ASTGCN \cite{Guo2019AttentionBasedST}), learnable variable selection modules, and adaptive feature transformations to improve multivariate data modeling.
Finally, for large models and zero-shot inference, research explores large language models (LLMs \cite{gruver2023llmtime}) for contextual learning (e.g., Time-LLM \cite{jin2023time}), enabling zero-shot forecasting and cross-domain transfer through time series tokenization and prompt tuning techniques.

\subsection{Ordinal Classification Task}
Ordinal classification faces challenges including strong label-order dependency, ambiguous category intervals, and limited model adaptability. Existing research primarily focuses on three directions. First, loss function design, where order-aware loss functions (e.g., threshold-based \cite{Niu2016ordinal} and margin-based \cite{pitawela2025cloc} losses) are introduced to explicitly model ordinal relationships; second, model architecture improvements, including approaches like Support Vector Ordinal Regression (SVOR) \cite{Gu2015Incremental} and COrrelation ALignment (CORAL) \cite{Shi2023RankConsistent}, which address label ordering via explicit threshold partitioning and implicit ordinal encoding, while attention mechanisms \cite{vaswani2017attention} and Bayesian models \cite{chu2005gaussian} demonstrate strong performance in capturing long-range dependencies and modeling uncertainty; and finally, probabilistic ordinal modeling, which adopts frameworks like ordered Probit/Logit \cite{liu2019probabilistic} and cumulative link models \cite{gutierrez2015ordinal} to characterize the conditional probability distribution over ordered categories, incorporating monotonicity constraints to reflect the inherent ordinal structure.
Recent work on stability and reliability in multi-view learning by \cite{Li2025} further enhances the robustness of ordinal classification methods.

\section{Time Series Forecasting Problem}
Time series forecasting aims to predict future values based on historical observations. Given a lookback window $\boldsymbol{X}_t = (X_t, X_{t-1}, \dots, X_{t-w+1})^\top \in \mathbb{R}^w$ of $w$ historical observations, we predict the $H$-step future trajectory $\boldsymbol{Y} = (Y_{t+1}, \dots, Y_{t+H})^\top \in \mathbb{R}^H$ through
$\hat{\boldsymbol{Y}} = g(\boldsymbol{X}_t)$,
where $g:\mathbb{R}^w \to \mathbb{R}^H$ is the forecasting function. For multivariate series with $M$-dimensional observations $\boldsymbol{x}_t \in \mathbb{R}^M$, this extends to predicting $\boldsymbol{Y} = \{\boldsymbol{y}_{t+1}, \dots, \boldsymbol{y}_{t+H}\} \in \mathbb{R}^{H \times M}$ from input $\boldsymbol{X}_t = \{\boldsymbol{x}_t, \boldsymbol{x}_{t-1}, \dots, \boldsymbol{x}_{t-w+1}\} \in \mathbb{R}^{w \times M}$.

The forecasting quality is evaluated via:
\begin{equation}
	\text{MSE} = \frac{1}{H}\|\boldsymbol{Y} - \hat{\boldsymbol{Y}}\|_F^2 = \frac{1}{H}\sum_{j=1}^H \|\boldsymbol{y}_{t+j} - \hat{\boldsymbol{y}}_{t+j}\|_2^2.
\end{equation}

%Time series forecasting aims to predict the future $H$-step trajectory $\boldsymbol{X}_{T+1:T+H} = \{\boldsymbol{x}_{T+1}, \dots, \boldsymbol{x}_{T+H}\}$ based on a historical observation sequence $\boldsymbol{X}_{1:T} = \{\boldsymbol{x}_1, \dots, \boldsymbol{x}_T\} \in \mathbb{R}^{T \times M}$, where each $\boldsymbol{x}_t \in \mathbb{R}^M$ denotes an $M$-dimensional observation at time $t$. This is commonly referred to as an input-$T$-predict-$H$ task, and becomes particularly challenging in long-term forecasting scenarios where $H \gg 1$.
%
%To address this, the model learns a mapping function \( F(\cdot) \) that projects the historical sequence to future predictions:
%\begin{equation}
%	\hat{\boldsymbol{X}}_{T+1:T+H} = F(\boldsymbol{X}_{1:T}; H; \theta),
%\end{equation}
%where \( \theta \) denotes model parameters and \( H \) optionally specifies the prediction horizon.
%
%Forecast accuracy is typically evaluated using the Mean Squared Error (MSE), defined as:
%\begin{align}
%	\text{MSE} &= \frac{1}{H} \|\boldsymbol{X}_{T+1:T+H} - \hat{\boldsymbol{X}}_{T+1:T+H}\|_F^2 \\
%	&= \frac{1}{H} \sum_{t=T+1}^{T+H} \|\boldsymbol{x}_t - \hat{\boldsymbol{x}}_t\|_2^2.
%\end{align}

Minimizing MSE corresponds to fitting the conditional expectation of the true data, ensuring unbiased predictions. However, MSE only supports point estimates and fails to capture uncertainty, motivating reformulations of the forecasting task as a probability prediction problem.

\section{Regression-to-Classification Framework}

The core idea of transforming regression into classification involves shifting from predicting point estimates $\hat{Y}_{t+j}$ through MSE minimization to learning complete predictive distributions $\bm{\pi}(\hat{Y}_{t+j})$, where $\hat{Y}_{t+j} \in \mathbb{R}$ represents the original continuous predicted variable and $\bm{\pi}(\hat{Y}_{t+j}) \in \mathbb{R}^K$ denotes a discrete probability distribution vector satisfying $\sum_{k=1}^K \pi_k = 1$.
The probabilistic regression framework transforms continuous prediction into a distribution classification problem through three key operations.

First, the true distribution is constructed by discretizing the ground truth $Y_{t+j} \in \mathbb{R}$ into bin probabilities $p_k = F_Y(u_k) - F_Y(\ell_k)$ for $k=1,...,K$, where $F_Y$ is the empirical cumulative distribution function (CDF) and $\mathcal{B}_k = [\ell_k, u_k)$ forms a partition of the target value space.

The predictive model $g_\theta$ with parameters $\theta$ then learns to output a probability distribution $\bm{\pi}(\bm{X}_t;\theta) = (\pi_1,...,\pi_K)$, where each $\pi_k = \mathbb{P}(\hat{Y}_{t+j} \in \mathcal{B}_k|\bm{X}_t)$ represents the predicted probability for bin $k$. This model is trained by minimizing the distributional discrepancy $\mathcal{L}(\theta) = D\big(\bm{p}(Y_{t+j}), \bm{\pi}(\bm{X}_t;\theta)\big)$ through the optimization objective $\theta^* = \argmin_{\theta} \mathbb{E}_{\bm{X}_t,Y_{t+j}} \left[ D\big(\bm{p}(Y_{t+j}), \bm{\pi}(\bm{X}_t;\theta)\big) \right]$, where $D$ measures the divergence (e.g., cross-entropy) between the true distribution $\bm{p}$ and predicted distribution $\bm{\pi}$.

Finally, continuous point estimates $\hat{Y}_{t+j}$ are recovered via  $\hat{Y}_{t+j} = \mathcal{G}\big({(v_k, \pi_k(\bm{X}_t;
	\theta^*))}_{k=1}^K\big)$, where $\mathcal{G}$ transforms the discrete probability distribution back to the continuous space using representative values $v_k$ for each bin $\mathcal{B}_k$.
%The probabilistic regression framework consists of four fundamental components:
%
%\begin{enumerate}
%    \item \textbf{True Distribution Construction}:
%    For ground truth $Y_{t+j} \in \mathbb{R}$, we compute the true bin probabilities:
%    \begin{equation}
	%        p_k = F_Y(u_k) - F_Y(\ell_k), \quad k=1,...,K
	%    \end{equation}
%    where $F_Y$ is the empirical CDF of $Y$, and $\mathcal{B}_k = [\ell_k, u_k)$ form a partition of the target space.
%
%    \item \textbf{Predictive Model}:
%    The model $f_\theta$ with parameters $\theta$ outputs a probability distribution:
%    \begin{equation}
	%        \bm{\pi}(x_t;\theta) = (\pi_1,...,\pi_K), \quad \pi_k = \mathbb{P}(\hat{Y}_{t+j} \in \mathcal{B}_k|x_t)
	%    \end{equation}
%
%    \item \textbf{Loss Minimization}:
%    The model is trained by minimizing the distributional discrepancy:
%    \begin{equation}
	%        \mathcal{L}(\theta) = D\big(\bm{p}(Y_{t+j}), \bm{\pi}(x_t;\theta)\big)
	%    \end{equation}
%    where $D$ is a divergence measure (e.g., cross-entropy) between true and predicted distributions.
%    The complete learning objective is:
%    \begin{equation}
	%    \theta^* = \argmin_{\theta} \mathbb{E}_{(x_t,Y_{t+j})} \left[ D\big(\bm{p}(Y_{t+j}), \bm{\pi}(x_t;\theta)\big) \right]
	%     \end{equation}
%
%    \item \textbf{Decoding Operation}:
%    Point estimates are recovered through:
%    \begin{equation}
	%        \hat{Y}_{t+j} = \mathcal{G}\big(\{(v_k, \pi_k(x_t; \theta^*))\}_{k=1}^K\big)
	%    \end{equation}
%\end{enumerate}

\section{Ordinal Cross-Entropy Loss}
This section presents the definition of OCE, followed by a comparative example illustrating the fundamental difference with standard cross-entropy (CE) and a novel influence function analysis that contrasts their sensitivity characteristics with MSE. See Appendix~\ref{app:CE} for the definition of CE.
\begin{definition}[Ordinal Cross-Entropy Loss] \cite{hu2010information}
	Let $Y \in \{1,\dots,K\}$ be the true class label and $\hat{Y}$ the predicted class. Given the predicted probability distribution $\bm{q} = [q_1,\dots,q_K]$ where $q_k = \mathbb{P}(\hat{Y} = k)$ and the true probability distribution $\bm{p} = [p_1,\dots,p_K]$. We define the cumulative probabilities as:
	$	\mathbb{P}_{\text{pred}}(\hat{Y} \leq k) = \sum_{i=1}^k q_i ,$
	$	\mathbb{P}_{\text{pred}}(\hat{Y} > k) = 1 - \mathbb{P}_{\text{pred}}(\hat{Y} \leq k).$
	Let $\mathbb{P}_{\text{true}}(Y \leq k)$ represent the ground-truth cumulative distribution.
	For hard labels: $\mathbb{P}_{\text{true}}(Y \leq k) = \mathbb{I}\{Y \leq k\}$, where $\mathbb{I}$ is the indictor function.
	The ordinal cross-entropy loss is then:
	\begin{equation}\label{eq:oce_loss}
		\begin{split}
			\mathcal{L}_{\text{OCE}}(Y, \hat{Y}) = &-\sum_{k=1}^{K-1} \Bigl[
			\mathbb{P}_{\text{true}}(Y \leq k) \log \mathbb{P}_{\text{pred}}(\hat{Y} \leq k) \\
			&+ \mathbb{P}_{\text{true}}(Y > k) \log \mathbb{P}_{\text{pred}}(\hat{Y} > k) \Bigr].
		\end{split}
	\end{equation}
\end{definition}

\subsection{Comparative Analysis of OCE and CE}
Table~\ref{tab:loss_comparison} compares standard cross-entropy and ordinal cross-entropy (OCE) performance across various prediction scenarios. In Case 1, where the true distribution is [0.8,0.1,0.1], although predicted distribution B achieves a lower CE value (0.449 vs 0.518), its OCE loss is higher (1.528 vs 1.396). This indicates that CE may underestimate the ordinal errors in predicted distribution B. Similarly, in Case 2 with a true distribution of [0.2,0.1,0.7], while predicted distribution A shows a slightly better CE value (0.604 vs 0.624), its OCE loss (2.029) is higher, further validating the limitations of CE in evaluating ordinal relationships.

\begin{table}[!htbp]
	\centering
	\small
	\setlength{\tabcolsep}{6pt}  % 增加列间距
	\renewcommand{\arraystretch}{1.2}  % 增加行高
	\begin{tabular}{@{}llccc@{}}
		\toprule
		\textbf{Scenario} & \textbf{Type} & \textbf{Distribution} $\bm{p}$ & \textbf{CE} & \textbf{OCE} \\
		\midrule
		\multirow{3}{*}{\textbf{Case 1}}
		& True & $[0.8, 0.1, 0.1]$ & -- & -- \\
		& Pred A & $[0.3, 0.5, 0.2]$ & 0.518 & 1.396 \\
		& Pred B & $[0.4, 0.1, 0.5]$ & 0.449 & 1.528 \\
		\cmidrule(r){1-5}
		
		\multirow{3}{*}{\textbf{Case 2}}
		& True & $[0.2, 0.1, 0.7]$ & -- & -- \\
		& Pred A & $[0.6, 0.2, 0.2]$ & 0.604 & 2.029 \\
		& Pred B & $[0.3, 0.5, 0.2]$ & 0.624 & 1.720 \\
		\bottomrule
	\end{tabular}
	\caption{Comparison of CE and OCE}
	\label{tab:loss_comparison}
\end{table}

Notably, the OCE loss shows greater sensitivity to deviations in class ordering between predicted and true distributions. Therefore, it is better suited for classification tasks emphasizing ordinal relationships, providing a more accurate reflection of the model’s ordering performance.

\subsection{Comparative Analysis of OCE and MSE}
This subsection presents a comparative influence function analysis of MSE and CE losses, focusing on their differential sensitivity to feature scaling characteristics and prediction residuals. The analysis framework applies uniformly to all cross-entropy-type losses (including both CE and OCE), as they share identical functional forms - differing only in their probability computation methods (plain softmax for CE versus cumulative probabilities for OCE).

Influence functions are widely used in robustness analysis, outlier detection, and training sample evaluation.Existing research  focuses on computing IF~\cite{pmlr-v70-koh17a,Schioppa2022Scaling} and using them for sample selection~\cite{Pruthi2020Estimating,Basu2020Influence}, whereas this paper analyzes the stability of MSE and CE losses through influence functions. We analyze and compare MSE's and CE's influence functions via simple linear models.
We consider the standard linear regression model $y = \bm{x}^\top \bm{\theta} + \epsilon$, where $\bm{x} \in \mathbb{R}^d$ is the $d$-dimensional feature vector, $\bm{\theta} \in \mathbb{R}^d$ is the $d$-dimensional parameter vector, $\epsilon \in \mathbb{R}$ is the scalar noise term, and $y \in \mathbb{R}$ is the scalar response.
The mean squared error loss function for a data point $\bm{z} = (\bm{x}, y)$ is given by: $L(\bm{z}, \bm{\theta}) = \frac{1}{2}(y - \bm{x}^\top \bm{\theta})^2 \in \mathbb{R}$.

Consider a $K$-class classification problem with a parametric softmax model. Let $\bm{x} \in \mathbb{R}^d$ be an input feature vector and $y \in \{1,...,K\}$ its corresponding class label. The model's predicted probabilities are given by:
\begin{equation}
	p(y|\bm{x}; \bm{\beta}) = \sigma(\bm{x}^\top \bm{\beta})_y = \frac{e^{\bm{x}^\top \bm{\beta}_y}}{\sum_{k=1}^K e^{\bm{x}^\top \bm{\beta}_k}},
\end{equation}
where $\bm{\beta} = [\bm{\beta}_1,...,\bm{\beta}_K] \in \mathbb{R}^{d \times K}$ contains the model parameters. The cross-entropy loss for a single observation $\bm{z} = (\bm{x}, y)$ is:
$L(\bm{z}, \bm{\beta}) = -\sum_{k=1}^K \mathbb{I}_{y=k} \log \sigma(\bm{x}^\top \bm{\beta})_k$.

Influence Function (IF) measures the sensitivity of model parameters or outputs to a single training sample. Its basic form is given by:
\begin{equation}
	\text{IF}(\bm{z}) = -\bm{H}_{\bm{\theta}}^{-1} \nabla_{\bm{\theta}} \mathcal{L}(\bm{\theta}, \bm{z}),
\end{equation}
where $\bm{H}_{\bm{\theta}}$ denotes the Hessian matrix evaluated at the optimal parameters $\bm{\theta}$, and $\nabla_{\bm{\theta}} \mathcal{L}$ represents the gradient of the loss function with respect to the model parameters.

The influence function for MSE can be derived as:
\begin{equation}
	\text{IF}_{\text{MSE}}(\bm{z};\bm{\theta}) = \left(\mathbb{E}[\bm{x}\bm{x}^\top]\right)^{-1}\bm{x}(y - \bm{x}^\top\bm{\theta}) \in \mathbb{R}^d,
	\label{eq:MSE}
\end{equation}
which quantifies the $d$-dimensional effect of individual data points on parameter estimates.

The influence function for CE can be derived as:
\begin{equation}
	\text{IF}_{\text{CE}}(\bm{z};\bm{\beta}) = -\left( \mathbb{E}\left[ \bm{x}\bm{x}^\top \otimes \bm{P} \right] \right)^{-1} (\bm{x}  \otimes  \left( \sigma(\bm{x}^\top \bm{\beta}) - \mathbf{e}_y \right) ),
	\label{eq:CE}
\end{equation}
where $\bm{P} = \text{diag}(\sigma(\bm{x}^\top \bm{\beta})) - \sigma(\bm{x}^\top \bm{\beta})\sigma(\bm{x}^\top \bm{\beta})^\top$ is the softmax output covariance matrix ($K\times K$), $\mathbf{e}_y \in \mathbb{R}^K$ is the one-hot encoded true label, $\otimes$ represents the Kronecker product.

As shown in Equations~(\ref{eq:MSE}) and~(\ref{eq:CE}), the cross-entropy loss alleviates the influence of the covariance matrix and the residuals through transformations involving the \( \bm{P} \) matrix and the sigmoid function, respectively. Below, we present the specific ratio between the two effects:
\begin{theorem}[Influence Function Growth Rate Comparison]
	\label{IFratio}
	Suppose that the covariance matrix $\bm{\Sigma}_X = \mathbb{E}[\bm{x}\bm{x}^\top]$ is positive definite with $\lambda_{\min}(\bm{\Sigma}_X) > 0$ and
	the expected Softmax matrix $\bm{P}$ is positive definite with $\lambda_{\min}(\bm{P}) > 0$.
	Assume the input features have finite second moments ($\mathbb{E}[\|\bm{x}\|_2^2] < \infty$) and the sample size exceeds the feature dimension ($n > d$).
	For non-degenerate predictions where the MSE residual is non-zero ($y - \bm{x}^\top\bm{\theta} \neq 0$) and the CE probability deviation is non-trivial ($\bm{\sigma} - \bm{e}_y \neq \bm{0}$), the influence function ratio $R = \|\mathrm{IF}_{\mathrm{CE}}\|_2/\|\mathrm{IF}_{\mathrm{MSE}}\|_2$ satisfies the two-sided bound:
	\begin{equation}
		\frac{\|\bm{\sigma} - \bm{e}_y\|_2}{\kappa_2(\bm{\Sigma}_X)\lambda_{\max}(\bm{P})|y - \bm{x}^\top\bm{\theta}|}
		\leq R \leq
		\frac{\sqrt{2}\kappa_2(\bm{\Sigma}_X)}{\lambda_{\min}(\bm{P})|y - \bm{x}^\top\bm{\theta}|},
	\end{equation}
	where $\kappa_2(\bm{\Sigma}_X) = \|\bm{\Sigma}_X\|_2\|\bm{\Sigma}_X^{-1}\|_2$ represents the spectral condition number.
\end{theorem}

From Theorem \ref{IFratio}, we can observe that
the CE loss becomes more stable than MSE when the ratio of influence functions satisfies $R \leq 1$. This stability condition holds when the data covariance matrix is well-conditioned and model predictions avoid extreme confidence.The proof and additional discussion are provided in Appendix~\ref{app:IF}.

\section{Methodology}
We propose the Ordinal Cross-Entropy Time Series Forecasting (OCE-TS) method, a novel regression-to-classification framework that transforms continuous forecasting into an ordinal probability estimation problem. As shown in Figure~\ref{fig:truncated-gaussian}, the approach consists of three key components: Target-to-Probability Transformation (TPT) that discretizes continuous values into probability distributions across ordered bins; Deep Ordinal Classifier Training using a neural architecture with ordinal cross-entropy loss to preserve the inherent ordering of bins while learning temporal patterns; and Probability-to-Value Reconstruction (PVR) that converts the predicted class probabilities back to continuous values.

\begin{figure}[t]
	\centering
	\includegraphics[width=0.5\textwidth]{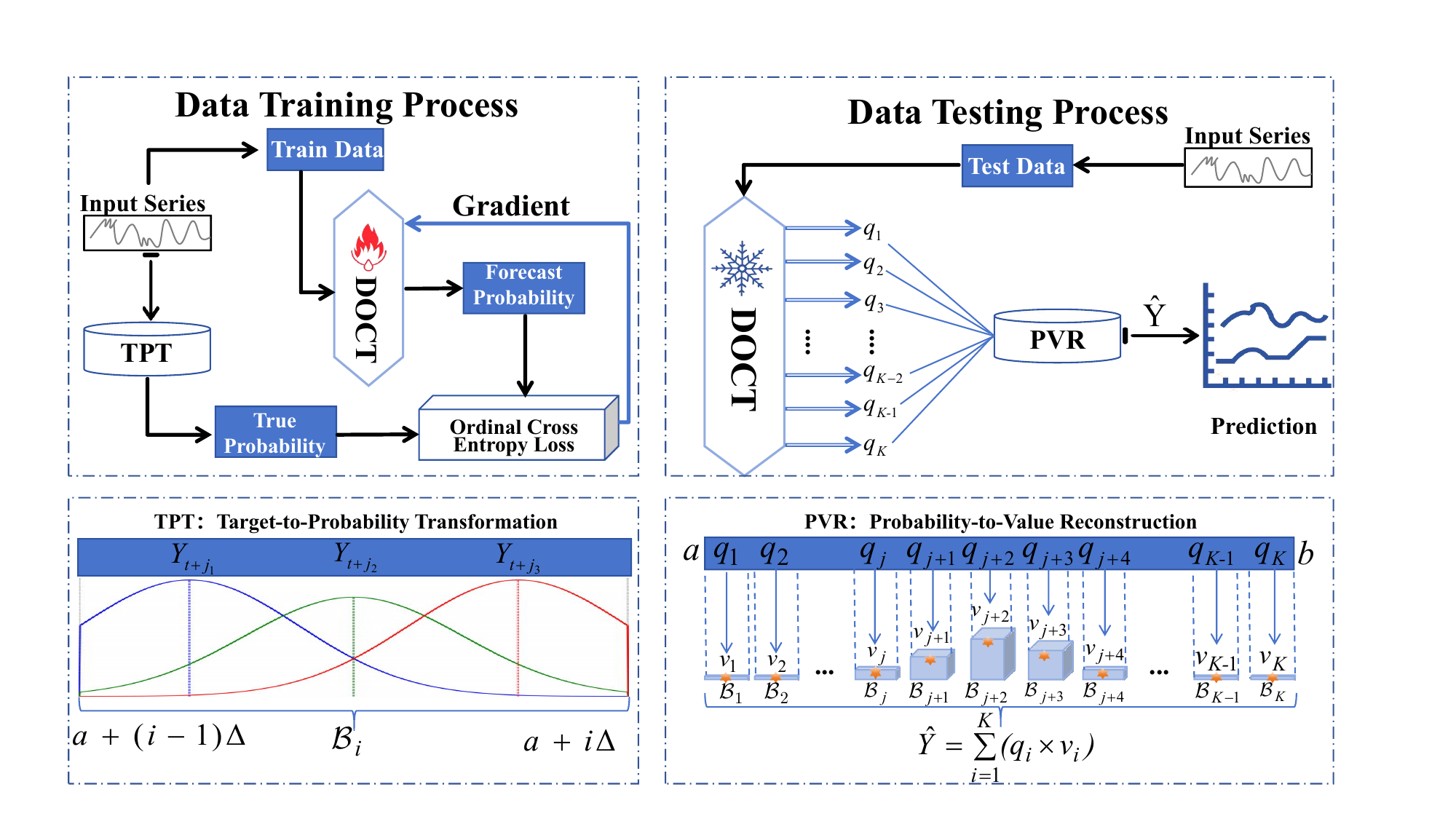}
	\caption{The Framework of Ordinal Cross-Entropy Loss-Based Time Series Forecasting (OCE-TS)}
	\label{fig:truncated-gaussian}
\end{figure}

\subsection{Target-to-Probability Transformation}

Given a true value $Y_{t+j} \in \mathbb{R}$ confined to the interval $[a,b]$, we construct its discrete probability distribution over $K$ bins as follows.
The target interval is divided into $K$ equal-width bins, where each bin $\mathcal{B}_k$, $k = 1,\ldots,K$ has lower and upper bounds defined by: $\mathcal{B}_k = [\ell_k, u_k)$, where $\ell_k = a + (k-1)\Delta$, $u_k = a + k\Delta$, and $\Delta = (b-a)/K$ for $k = 1,\ldots,K$. $\Delta$ represents the uniform bin width, ensuring complete coverage of $[a,b]$ without overlap.

For each observation $Y_{t+j}$, we center a truncated Gaussian distribution at $y_c = Y_{t+j}$:
\begin{equation}\label{eq:G_pdf}
	p(y) =
	\begin{cases}
		\frac{1}{Z\sigma\sqrt{2\pi}}\exp\left(-\frac{(y-y_c)^2}{2\sigma^2}\right) & y \in [a,b], \\
		0 & \text{otherwise}.
	\end{cases}
\end{equation}
The normalization factor $Z$ ensures unit probability mass within $[a,b]$:
\begin{equation}\label{eq:G_Z}
	Z = \frac{1}{2}\left[
	\operatorname{erf}\left(\frac{b-y_c}{\sigma\sqrt{2}}\right) -
	\operatorname{erf}\left(\frac{a-y_c}{\sigma\sqrt{2}}\right)
	\right].
\end{equation}
Here, $\operatorname{erf}(\cdot)$ denotes the Gauss error function, and $\sigma^2$ controls the dispersion of the distribution.

The probability mass of $Y_{t+j}$ for each bin $\mathcal{B}_k$ is obtained by integrating the CDF:
%\begin{align}
%p_k &= \mathbb{P}(y \in \mathcal{B}_k) \\ \notag
%&= \frac{1}{2Z}\left[
%\operatorname{erf}\left(\frac{u_k - y_c}{\sigma\sqrt{2}}\right) -
%\operatorname{erf}\left(\frac{\ell_k - y_c}{\sigma\sqrt{2}}\right)
%\right].
%\end{align}
\begin{align}
	p_k^{(j)} &= \mathbb{P}(Y_{t+j} \in \mathcal{B}_k) \label{eq:bin_prob_def} \\
	&= \frac{1}{2Z}\left[
	\operatorname{erf}\left(\frac{u_k - y_c}{\sigma\sqrt{2}}\right) -
	\operatorname{erf}\left(\frac{\ell_k - y_c}{\sigma\sqrt{2}}\right)
	\right]. \label{eq:bin_prob_calc}
\end{align}

This formulation guarantees normalization and non-negativity $\forall k$:
$\sum_{k=1}^K p_k^{(j)} = 1, \quad p_k^{(j)} \geq 0.$

Each observation $Y_{t+j}$ is mapped to a $K$-dimensional probability vector:
$\bm{p}^{(j)} = \left[ p_1^{(j)}, p_2^{(j)}, \ldots, p_K^{(j)} \right]^{T}, \quad
p_k^{(j)} = \mathbb{P}(Y_{t+j} \in \mathcal{B}_k),$
where the vector components sum to unity ($\|\bm{p}^{(j)}\|_1 = 1$) by construction.

\subsection{Deep Ordinal Classifier Training}

To obtain predictive probability distributions, this study adopts a probabilistic forecasting framework based on the Decomposition-Linear (DLinear) Model. DLinear \cite{zeng2023transformers} is a time series forecasting model that employs series decomposition and multi-layer linear mapping. Given an input window $\bm{X}_t \in \mathbb{R}^{w \times M}$ with $w$ time steps and $M$ features, the model first decomposes the series into trend and residual components through moving average smoothing:
\begin{equation}
	\bm{T}_t = \text{MA}_l(\bm{X}_t), \quad \bm{S}_t = \bm{X}_t - \bm{T}_t,
\end{equation}
where $\text{MA}_l(\cdot)$ is the moving average smoothing with window size $l$, and $\bm{S}_t$ captures residual variations. This decomposition separates temporal patterns at different scales.

The model processes each component through dedicated linear projections:
\begin{equation}
	\tilde{\bm{Y}} = \bm{W}_t \bm{T}_t^\top + \bm{W}_s \bm{S}_t^\top \in \mathbb{R}^{H \times M},
\end{equation}
where $\bm{W}_t, \bm{W}_s \in \mathbb{R}^{H \times w}$ are learnable projection matrices. For each prediction horizon $j$, the scalar output $\tilde{Y}_{t+j}$ is transformed into bin probabilities:
\begin{equation}
	\bm{q} = \text{softmax}(\bm{W}_o \tilde{Y}_{t+j} + \bm{b}_o),
\end{equation}
with parameters $\bm{W}_o \in \mathbb{R}^{K \times M}$ and $\bm{b}_o \in \mathbb{R}^K$ mapping to $K$ ordinal bins.
The complete parameter set $\theta = \{\bm{W}_t, \bm{W}_s, \bm{W}_o, \bm{b}_o\}$ is optimized by minimizing:
%\begin{equation}
%\mathcal{L}(\theta) = \mathbb{E}_{(\bm{X}_t,Y_{t+j})} \left[ \mathcal{L}_{\text{OCE}}(\bm{p}(Y_{t+j}), \bm{q}(\bm{X}_t;\theta)) \right]
%\end{equation}
\begin{equation}\label{eq:L}
	\mathcal{L}(\theta) = \frac{1}{N}\frac{1}{H}\sum_{i=1}^{N}\sum_{j=1}^{H} \mathcal{L}_{\text{OCE}}\left(\bm{p}(Y_{t_i+j}), \bm{q}(\bm{X}_{t_i};\theta)\right)
\end{equation}
where $\mathcal{L}_{\text{OCE}}$ is the ordinal cross-entropy loss (Eq.~(\ref{eq:oce_loss})), $Y_{t_i+j}$ is the ground truth at $t_i+j$, $N$ is the number of training windows, $H$ the prediction horizon, and $i$, $j$ index training samples and forecast steps, respectively.

The model parameters are optimized by minimizing the loss function $\mathcal{L}(\theta)$ to obtain the optimal parameters:
\begin{equation}
	\theta^* = \mathop{\arg\min}\limits_{\theta} \mathcal{L}(\theta),
	\label{eq:param_optim}
\end{equation}
 after obtaining the optimal parameters $\theta^*$, the model's prediction can be expressed as:
\begin{equation}
	q_t = \bm{q}(\bm{X}_{t}; \theta^*),
	\label{eq:q_compute}
\end{equation}
where $q_t$ is the model's prediction at time $t$ for $Y_{t+j}$.

\subsection{Probability-to-Value Reconstruction}
The final prediction $\hat{Y}_{t+j}$ is computed as the probability-weighted average of bin representatives $\{v_k\}_{k=1}^K$:
\begin{equation}
	\hat{Y}_{t+j} = \mathcal{G}\big(\{(v_k, q_k)\}_{k=1}^K\big) = \sum_{k=1}^K v_k \cdot q_k,
\end{equation}
where $v_k$ is the predefined center of bin $\mathcal{B}_k$.

The above end-to-end probabilistic forecasting framework is trained using ordinal cross-entropy loss to preserve bin ordering while enabling both accurate point forecasts and uncertainty quantification.

\section{Experimental Analysis}
This section evaluates OCE-TS against five baselines on seven datasets. We analyze key factors like loss functions, parameters and distribution assumptions to validate our method. Additional experimental results are provided in the Appendix~\ref{app:experimental results}.
\subsection{Experiments Settings}
\textbf{Dataset} We conduct experiments on seven publicly available time-series datasets, including ETT (ETTh1, ETTh2, ETTm1, ETTm2), Exchange, ILI, and Weather. The detailed dataset descriptions are provided in Appendix~\ref{app:dataset}.

\textbf{Baselines.} We adopt five representative models for long-term time series forecasting as our baselines, including three Transformer-based models: Autoformer~\cite{wu2021autoformer}, iTransformer~\cite{liu2023itransformer}, TimeXer~\cite{wang2024timexer} and TimeBridge~\cite{liu2024timebridge}, as well as one MLP-based model: DLinear~\cite{zeng2023transformers}.

\textbf{Experimental Environment}
The experiments are conducted on a machine equipped with an NVIDIA GeForce RTX 4060 GPU, an Intel(R) Core(TM) i7-14700F CPU, and 16 GB of RAM. The implementation requires Python 3.8.20 and PyTorch 2.0.0 compiled with CUDA 11.8 support, leveraging GPU acceleration during model training.

\textbf{Evaluation Metrics}
The MSE and MAE are employed to evaluate model performance in  time series forecasting. Lower values indicate better predictive performance. The definition of MAE is provided in Appendix~\ref{app:MAE}.

\textbf{Model Configuration and Training Details}
The input time series are normalized to $[-1, 1]$ or $[0, 1]$ during training and validation, and denormalized during testing for consistent metric evaluation. The support $[a, b]$ of the truncated Gaussian distribution is set to match the normalization range.Models are trained with a batch size of 32 for 15 epochs using the Adam optimizer (initial learning rate 0.005), employing dynamic learning rate adjustment based on validation performance, with an early stopping mechanism (patience=5 epochs). The input window size $w$ is set to 336, and prediction lengths $H$ are selected from $\{96, 192, 336, 720\}$. Details are provided in Appendix~\ref{app:Model}.

\begin{table*}[!t]
	\centering
	\renewcommand{\arraystretch}{0.8}
	\setlength{\tabcolsep}{0.3pt}
	\fontsize{8}{8}\selectfont% 缩小字体
	\begin{tabular}{@{}lccccccccccccc@{}}
		\toprule
		\multirow{2}{*}{\textbf{Dataset}} & \multirow{2}{*}{\textbf{H}} & \multicolumn{6}{c}{\textbf{MSE}} & \multicolumn{6}{c}{\textbf{MAE}} \\
		\cmidrule(lr){3-8} \cmidrule(lr){9-14}
		& & \makecell{Autoformer\\(2021)} & \makecell{Dlinear\\(2022)} & \makecell{Itransformer\\(2024)} & \makecell{Timexer\\(2024)} & \makecell{TimeBridge\\(2024)} & \textbf{Our}
		& \makecell{Autoformer\\(2021)} & \makecell{Dlinear\\(2022)} & \makecell{Itransformer\\(2024)} & \makecell{Timexer\\(2024)} & \makecell{TimeBridge\\(2024)} & \textbf{Our} \\
		\midrule
		
		% ETTh1数据集
		\multirow{4}{*}{ETTh1} & 96  & 0.453 & 0.375 & 0.386 & 0.377& 0.358 & \textbf{0.341} & 0.453 & 0.399 & 0.405 & 0.397 & \textbf{0.392} & 0.395 \\
		& 192 & 0.504 & 0.405 & 0.441 & 0.425 & \textbf{0.388}& 0.416 & 0.482 & 0.416 & 0.436 & 0.426 & \textbf{0.411}& 0.445 \\
		& 336 & 0.505 & 0.439 & 0.487 & 0.457 & \textbf{0.401}& 0.491 & 0.484 & 0.443 & 0.458 & 0.441 & \textbf{0.419}& 0.491 \\
		& 720 & 0.498 & 0.472 & 0.503 & 0.464 & \textbf{0.447}& 0.547 & 0.500 & 0.490 & 0.491 & 0.463 & \textbf{0.458}& 0.528 \\
		\midrule
		
		% ETTh2数据集
		\multirow{4}{*}{ETTh2} & 96  & 0.368 & 0.289 & 0.297 & 0.289 & 0.295 & \textbf{0.197} & 0.410 & 0.353 & 0.349 & 0.340& 0.354 & \textbf{0.313} \\
		& 192 & 0.422 & 0.383 & 0.380 & 0.370 & 0.351& \textbf{0.248} & 0.434 & 0.418 & 0.400 & 0.391 & 0.389& \textbf{0.352} \\
		& 336 & 0.471 & 0.448 & 0.428 & 0.422 & 0.351& \textbf{0.313} & 0.475 & 0.465 & 0.432 & 0.434 & 0.397& \textbf{0.395} \\
		& 720 & 0.474 & 0.605 & 0.427 & 0.429 & 0.388& \textbf{0.379} & 0.484 & 0.551 & 0.445 & 0.445 & 0.436& \textbf{0.435} \\
		\midrule
		
		% ETTm1数据集
		\multirow{4}{*}{ETTm1} & 96  & 0.481 & 0.299 & 0.334 & 0.309 & 0.297& \textbf{0.195} & 0.463 & 0.343 & 0.368 & 0.352 & 0.353 &\textbf{0.308} \\
		& 192 & 0.598 & 0.335 & 0.377 & 0.355  & 0.333 & \textbf{0.279} & 0.513 & 0.365 & 0.391 & 0.378  & 0.375 & \textbf{0.363} \\
		& 336 & 0.579 & 0.369 & 0.426 & 0.387  & \textbf{0.366} & 0.372 & 0.509 & \textbf{0.386} & 0.420 & 0.399  & 0.399 & 0.407 \\
		& 720 & 0.560 & 0.425 & 0.491 & 0.448  & \textbf{0.414} & 0.479 & 0.509 & \textbf{0.421} & 0.459 & 0.435 & 0.423& 0.466 \\
		\midrule
		
		% ETTm2数据集
		\multirow{4}{*}{ETTm2} & 96  & 0.255 & 0.167 & 0.180 & 0.171 & 0.158& \textbf{0.118} & 0.339 & 0.260 & 0.264 & 0.255 & 0.249& \textbf{0.230} \\
		& 192 & 0.281 & 0.224 & 0.250 & 0.238 & 0.215& \textbf{0.167} & 0.340 & 0.303 & 0.309 & 0.300 & 0.291& \textbf{0.279} \\
		& 336 & 0.339 & 0.281 & 0.311 & 0.301 & 0.263& \textbf{0.211} & 0.372 & 0.342 & 0.348 & 0.340 & 0.323& \textbf{0.319} \\
		& 720 & 0.422 & 0.397 & 0.412 & 0.401 & 0.348& \textbf{0.277} & 0.419 & 0.421 & 0.407 & 0.397 & 0.376& \textbf{0.367} \\
		\midrule
		
		% Exchange数据集
		\multirow{4}{*}{Exchange} & 96  & 0.197 & 0.081 & 0.086 & 0.089 &0.084 & \textbf{0.055} & 0.382 & 0.203 & 0.206 & 0.208 &0.202 & \textbf{0.148} \\
		& 192 & 0.300 & 0.157 & 0.177 & 0.196 &0.189 & \textbf{0.107} & 0.369 & 0.293 & 0.299 & 0.313 &0.310 & \textbf{0.224} \\
		& 336 & 0.509 & 0.305 & 0.331 & 0.354 &0.362 & \textbf{0.199} & 0.524 & 0.414 & 0.417 & 0.429 &0.436 & \textbf{0.322} \\
		& 720 & 1.447 & 0.643 & 0.847 & 0.894 &0.900 & \textbf{0.521} & 0.941 & 0.601 & 0.691 & 0.708 &0.717 & \textbf{0.578} \\
		\midrule
		
		% ILI数据集
		\multirow{4}{*}{ILI} & 24  & 3.483 & 2.215 & 2.695 & 1.971 &2.055& \textbf{0.570} & 1.287 & 1.081 & 1.703 & 0.868 &0.890& \textbf{0.458} \\
		& 36  & 3.103 & 1.963 & 2.563 & 2.068 &2.017& \textbf{0.749} & 1.148 & 0.963 & 1.051 & 0.934 &0.950& \textbf{0.558} \\
		& 48  & 2.669 & 2.130 & 2.567 & 1.925 &1.898& \textbf{1.180} & 1.085 & 1.024 & 1.048 & 0.884 &0.905& \textbf{0.710} \\
		& 60  & 2.770 & 2.368 & 2.625 & 1.920 &1.796& \textbf{1.282} & 1.125 & 1.096 & 1.068 & 0.892 &0.897& \textbf{0.738} \\
		\midrule
		
		% Weather数据集
		\multirow{4}{*}{Weather} & 96  & 0.266 & 0.176 & 0.174 & 0.168 &0.143& \textbf{0.107} & 0.336 & 0.237 & 0.214 & 0.209 &0.192& \textbf{0.144} \\
		& 192 & 0.307 & 0.220 & 0.221 & 0.220 &0.185& \textbf{0.140} & 0.367 & 0.282 & 0.254 & 0.254 &0.235& \textbf{0.190} \\
		& 336 & 0.359 & 0.265 & 0.278 & 0.276 &0.237& \textbf{0.169} & 0.395 & 0.319 & 0.296 & 0.296 &0.277& \textbf{0.228} \\
		& 720 & 0.419 & 0.323 & 0.358 & 0.353 &0.307& \textbf{0.225} & 0.428 & 0.362 & 0.347 & 0.347 &0.330& \textbf{0.283} \\

		%		% Electricity数据集
		%		\multirow{4}{*}{Electricity} & 96  & 0.201 & 0.140& 0.148 & 0.158 &\textbf{0.118} & 0.141 & 0.317 & 0.237 & 0.240 & 0.258 &\textbf{0.218}& 0.243 \\
		%		& 192 & 0.222 & 0.153 & 0.166 & 0.174 &\textbf{0.142}& 0.180 & 0.334 & 0.249 & 0.258 & 0.271  &\textbf{0.237}&0.273 \\
		%		& 336 & 0.231 & 0.169 & 0.178 & 0.192 &\textbf{0.156}& 0.224 & 0.338 & 0.267 & 0.271 & 0.289 &\textbf{0.252}& 0.303 \\
		%		& 720 & 0.254 & 0.203 & 0.209 & 0.233 &\textbf{0.179}& 0.309 & 0.361 & 0.301 & 0.299 & 0.321 &\textbf{0.278}& 0.360 \\
		%		%		\midrule
		
		% traffic数据集
		%		\multirow{4}{*}{traffic} & 96  & 0.613 & 0.410 & 0.395 & 0.416 &\textbf{0.340}& 0.618 & 0.388 & 0.282 & 0.268& 0.280 &\textbf{0.240} & 0.409 \\
		%		& 192 & 0.616 & 0.423 &0.417 & 0.435 & \textbf{0.343}& 0.722 & 0.382 & 0.287 & 0.276 & 0.288 &\textbf{0.250}& 0.437 \\
		%		& 336 & 0.622 & 0.436 & 0.433& 0.451 &\textbf{0.363} & 0.839 & 0.337 & 0.296 & 0.283 & 0.296 &\textbf{0.257}& 0.473 \\
		%		& 720 & 0.660 & 0.466& 0.467 & 0.484 &\textbf{0.393} & 0.963 & 0.408 & 0.315 & 0.302 & 0.314 &\textbf{0.27}1& 0.520 \\
		\bottomrule
	\end{tabular}
	\caption{Comparison of Model Performances}
	\label{tab:model_performance}
\end{table*}
\subsection{Comparative Results}
To verify the effectiveness of the proposed method, we conduct a systematic performance comparison between OCE-TS and five mainstream models trained with the traditional MSE loss function on multiple benchmark datasets.

Table~\ref{tab:model_performance} presents the comparative performance analysis. Given the relatively small size of the ILI dataset, its lookback window is configured as 104 timesteps with forecast horizons $\{24, 36, 48, 60\}$. Following conventional practice, other datasets adopt a fixed lookback window of 336 timesteps and forecast horizons $\{96, 192, 336, 720\}$. The experimental results show that OCE-TS achieves superior forecasting performance across most benchmark datasets.

To intuitively demonstrate the model’s fitting performance on two types of datasets, this paper visualizes the 720-step forecasting results by comparing the predicted values with the ground truth. Both the proposed method and the DLinear model use a 336-step historical sequence as input.

\begin{figure}[h]
	\centering
	\subfloat[DLinear]{\includegraphics[width=0.11\textwidth]{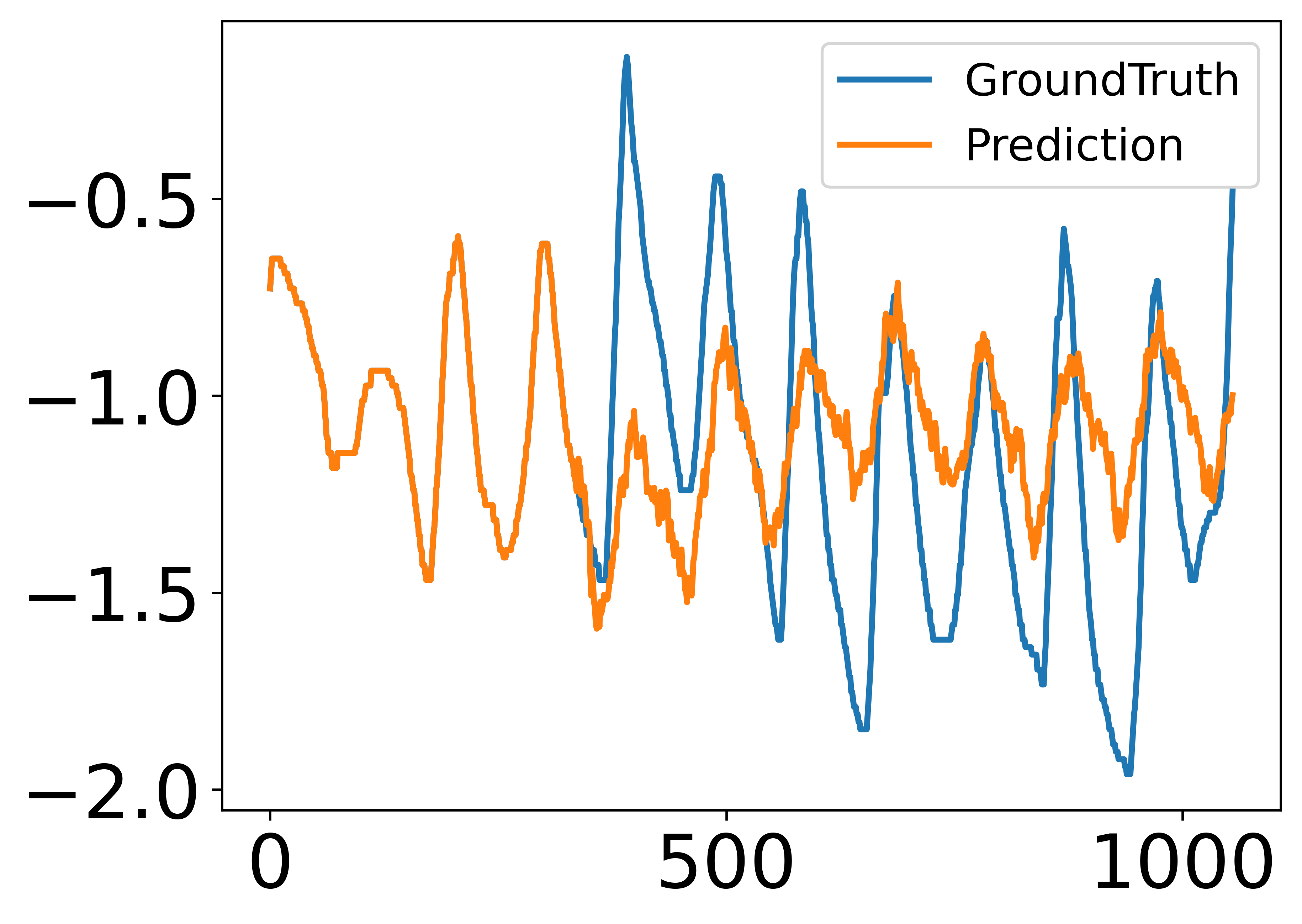}}
	\subfloat[Ours]{\includegraphics[width=0.11\textwidth]{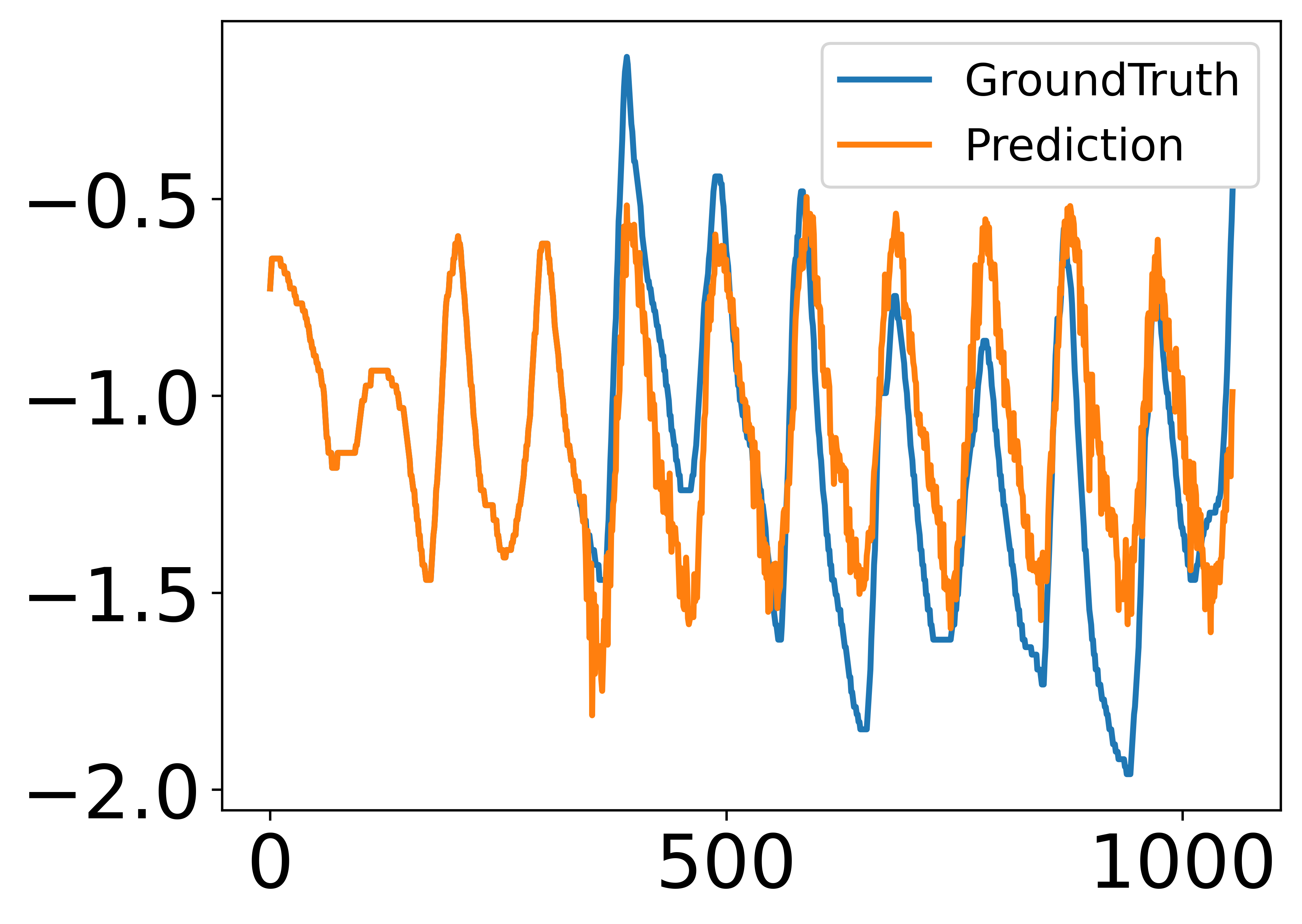}}
	\subfloat[DLinear]{\includegraphics[width=0.11\textwidth]{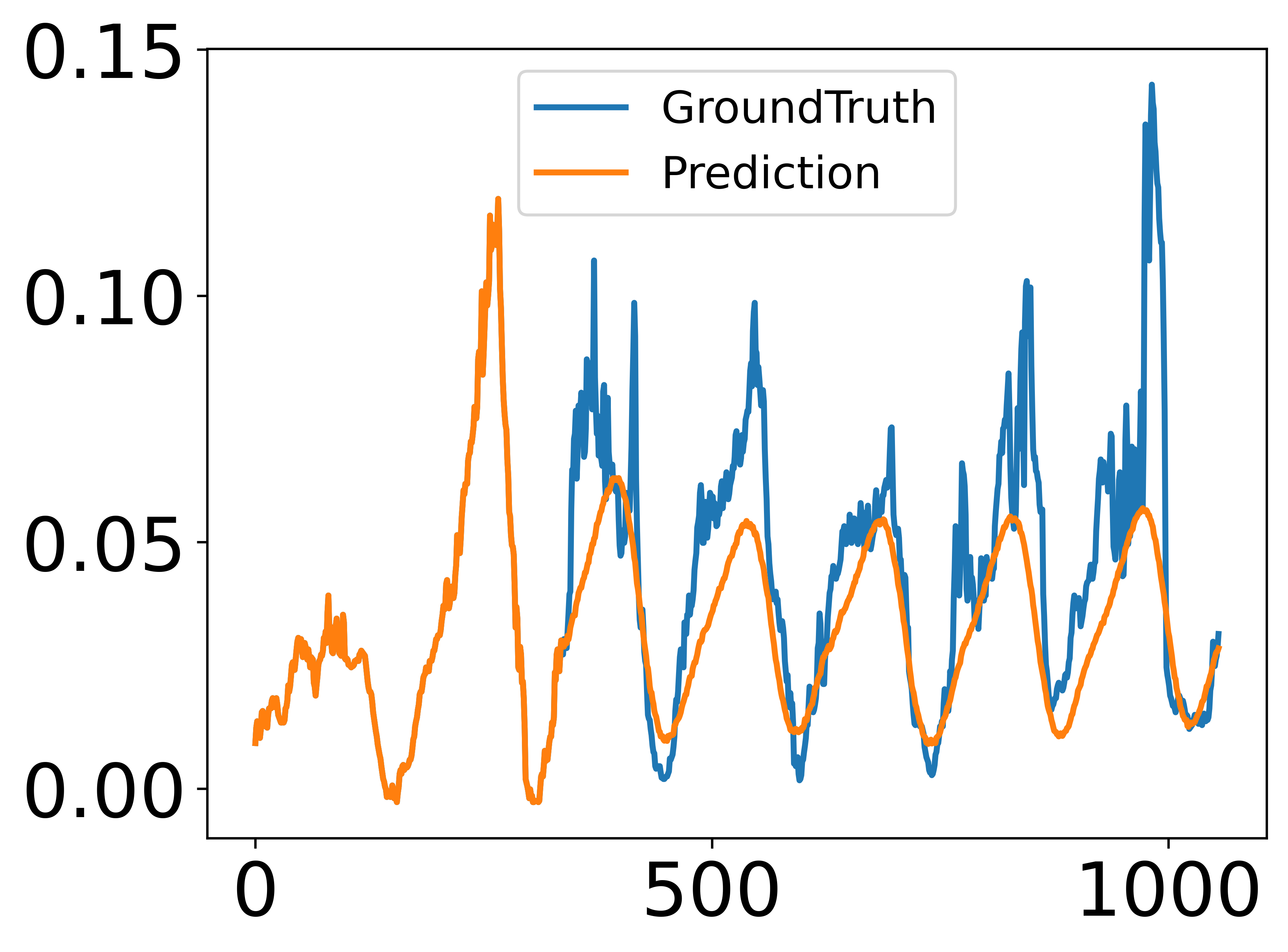}}
	\subfloat[Ours]{\includegraphics[width=0.11\textwidth]{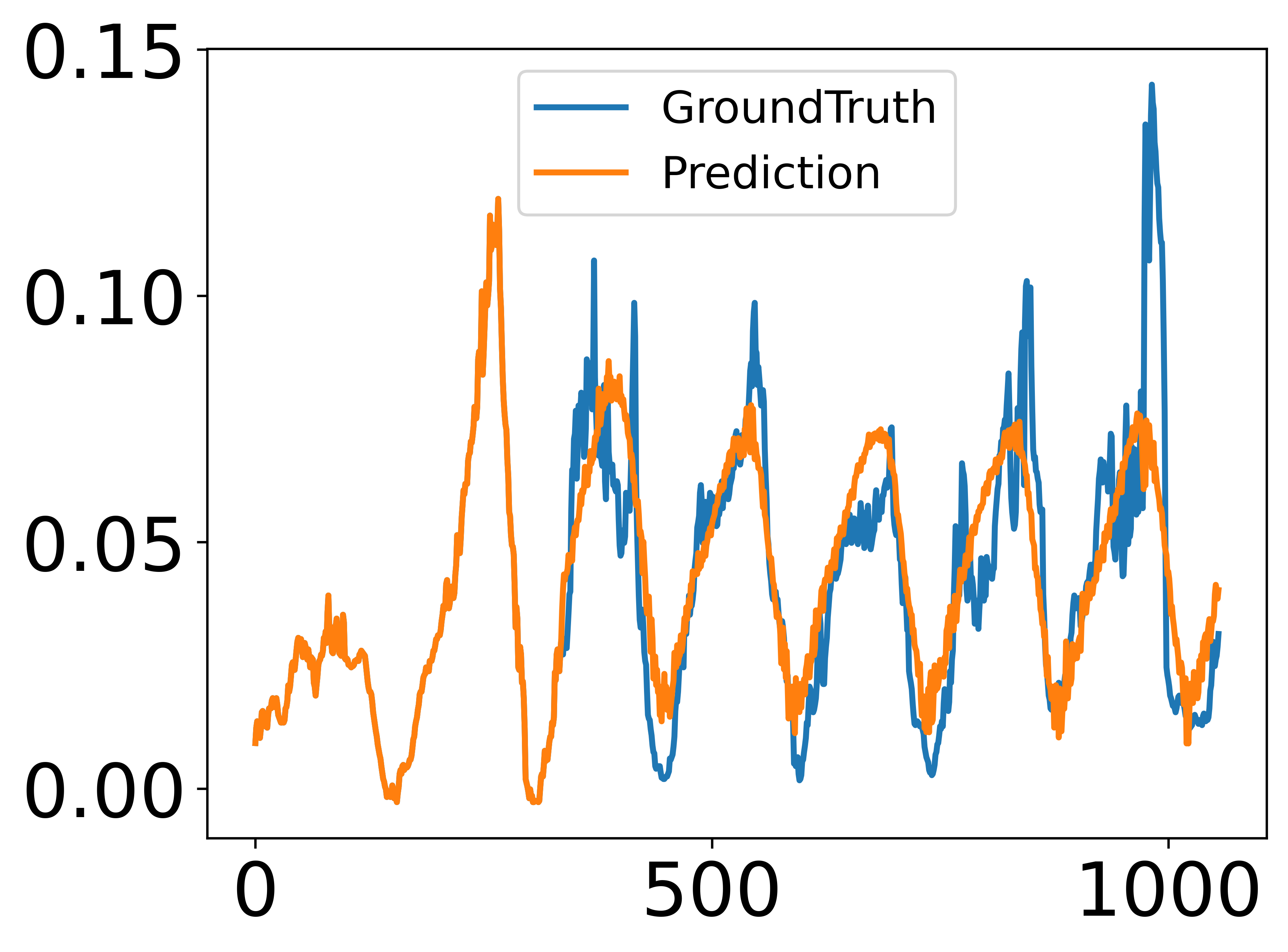} }
%	\caption{Model Fitting Results: DLinear vs. Ours}
\caption{Model Fitting Results of DLinear and Ours}
	\label{fig:model-comparison}
\end{figure}

As shown in Figure~\ref{fig:model-comparison}, the first two plots forecast results on ETTm2, while the latter two show predictions for Weather. Our method tracks ground truth better than DLinear, confirming the better performance of our approach.

\begin{table}[!ht]
	\centering
	\scriptsize  % 9磅罗马字体
	\setlength{\tabcolsep}{1mm}  % 进一步压缩列间距
	\begin{tabular}{lcccccc}
		\toprule
		\textbf{models} & \textbf{Autoformer} & \textbf{Dlinear} & \textbf{Itransformer} & \textbf{Timexer} & \textbf{TimeBridge} & \textbf{Ours} \\
		\midrule
		\textbf{ms/iter} &123.3  &19.3  &29  &17.2  &21  &34.9  \\
		\bottomrule
	\end{tabular}
	\caption{Per-Iteration Runtime by Model (ms)}
	\label{tab:time-per-iter-9pt}
	\vspace{-9pt}
\end{table}
Table~\ref{tab:time-per-iter-9pt} shows that OCE-TS achieves high accuracy with moderate cost, balancing efficiency and performance.

\subsection{Study of Different Probability Distributions}
This section analyzes the impact of the probability distribution on the performance of the ETTh dataset. Distributions are defined in Appendix~\ref{app:Distribution}.

\begin{figure}[!ht]
	\centering
	\includegraphics[width=0.85\linewidth]{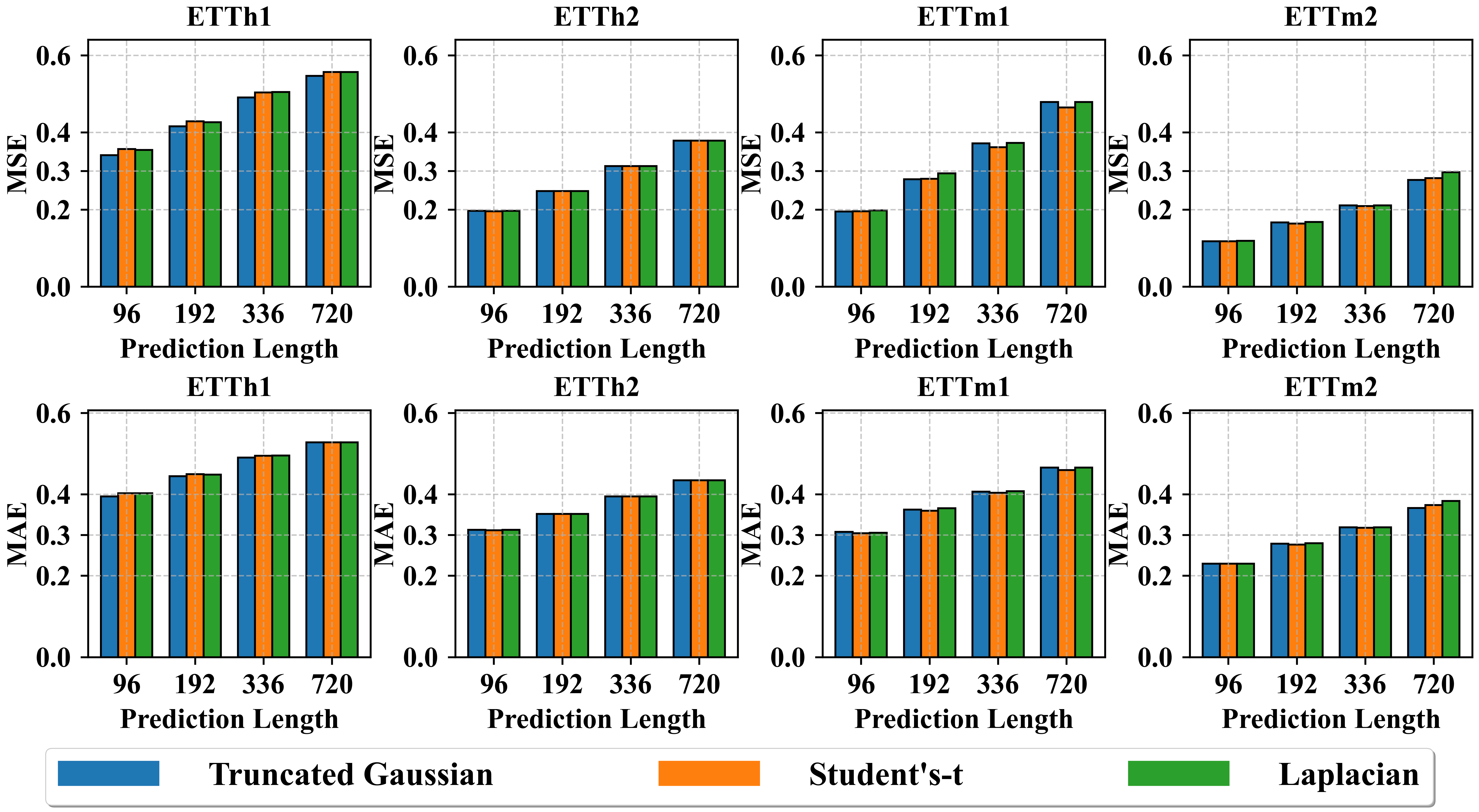}
	\caption{Distribution Comparison on ETT Datasets}
	\label{fig:dist-compare}
\end{figure}

As illustrated in Figure~\ref{fig:dist-compare}, the truncated Gaussian distribution achieves optimal performance on the ETTh1 dataset, demonstrating superior MSE and MAE metrics across all prediction horizons compared to both the Student's-t distribution and Laplacian distribution, with particularly notable advantages in long-term forecasting. However, the Student's-t distribution exhibits better performance on ETTm2, while all three distributions show comparable results on ETTh2 and ETTm1 without statistically significant differences. Hence, the selection of an appropriate probability distribution should be carefully determined according to specific task requirements and dataset characteristics.

\subsection{Lookback Window Sizes}
To evaluate the impact of lookback window size on forecasting performance, this study employs multi-scale sliding windows (48 / 96 / 192 / 336 / 504 / 720 timesteps) for 720-step-ahead prediction across multiple datasets, quantitatively analyzing the correlation between historical information scale and long-term forecasting accuracy.

\begin{figure}[!ht]
	\centering
	\begin{subfigure}[b]{\linewidth}
		\centering
		\includegraphics[width=\linewidth, height=\textheight, keepaspectratio]{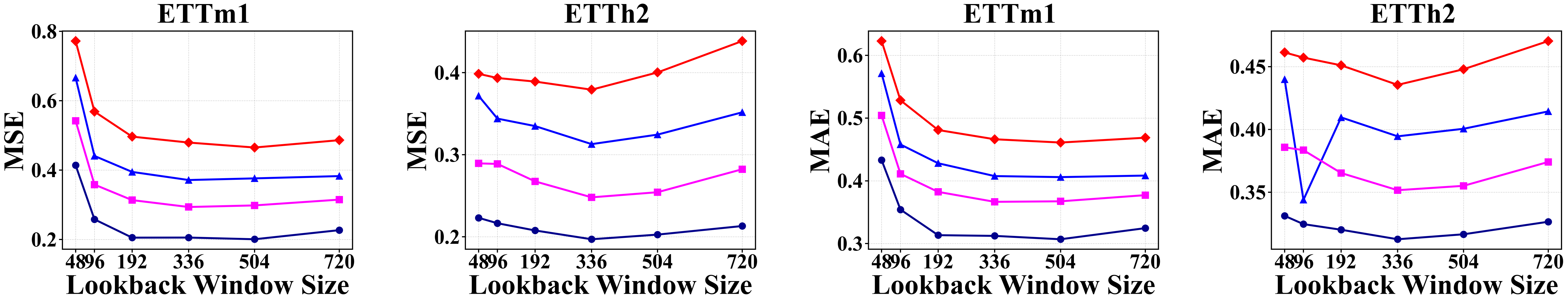}
	\end{subfigure}
	
	\begin{subfigure}[b]{\linewidth}
		\centering
		\includegraphics[width=\linewidth, height=\textheight, keepaspectratio]{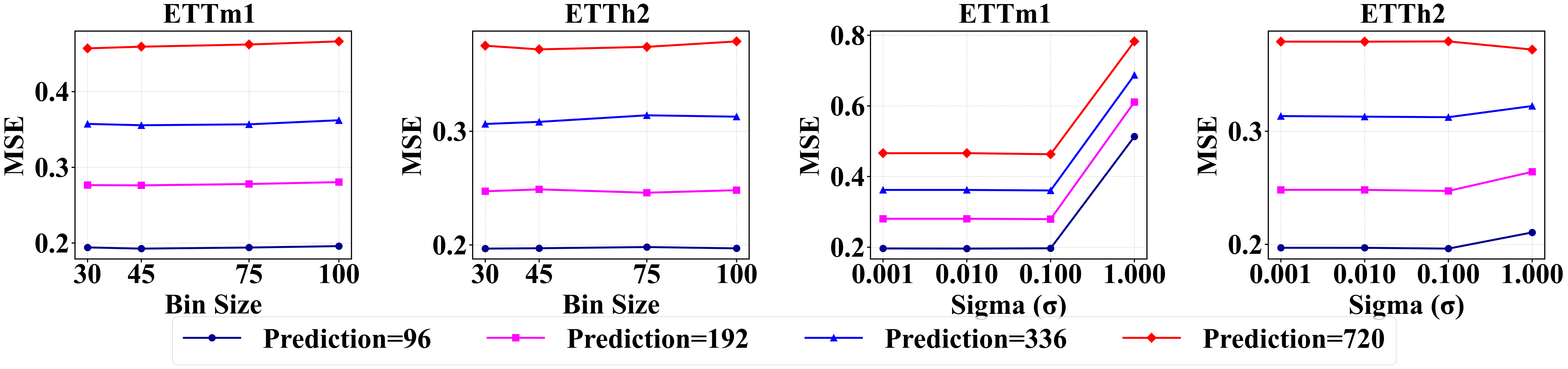}
	\end{subfigure}
	\caption{Lookback Window Sizes and Parameter Sensitivity}
	\label{fig:combined-results}
\end{figure}

 The first row of Figure~\ref{fig:combined-results} presents the performance of ETTm1 and ETTh2 under different lookback window lengths, the experiments show that increasing the lookback window significantly improves prediction accuracy for ETTm1 but has limited benefits for ETTh2, indicating that the optimization effect of historical information length depends on dataset characteristics. ETTm1 achieves optimal performance with $\geq$336 steps, while ETTh2 performs best at 192 steps. Therefore, the selection of the lookback window should balance accuracy and efficiency, with adjustments tailored to different datasets. Please refer to Appendix~\ref{app:Lookback} for complete experimental results and details.

\subsection{Parameter Sensitivity}
The OCE-TS method involves two key parameters: the number of bins (bin = K) and the standard deviation (\(\sigma\)). In our experimental setup, we consider bin \(\in \{30, 45, 75, 100\}\) and \(\sigma \in \{0.001, 0.01, 0.1, 1\}\), while keeping all other settings consistent with the baseline configuration. To systematically analyze the influence of these parameters, we examine the effect of bin by fixing \(\sigma = 0.01\), and analyze the impact of \(\sigma\) under a fixed bin of 100. Based on these results, we provide recommendations for parameter configuration.

The second row of Figure~\ref{fig:combined-results} shows the variations in model performance with respect to $\sigma$ values and binning parameters. The ETTm1 and ETTh2 datasets are insensitive to the number of bins but exhibit different levels of sensitivity to the standard deviation. Both datasets maintain stable performance across various bin settings; however, when $\sigma = 1$, the MSE on ETTm1 increases significantly, and ETTh2 also experiences a performance drop. Considering both stability and accuracy, we recommend setting the number of bins to 100 and the standard deviation $\sigma$ to 0.01. Please refer to Appendix~\ref{app:Parameter} for complete experimental results and details.

 \subsection{Comparison with CE}
This section compares the performance between the OCE loss and the CE loss. See Appendix~\ref{app:Loss} for complete experimental results and details.
Experimental results in Figure~\ref{fig:ce-oce-comparison} validate that OCE loss achieves superior metric performance (MAE / MSE) over CE loss on the Weather dataset.
\begin{figure}[ht]
	\centering
	% 第一张子图（CE与OCE的MSE对比）
	\begin{subfigure}[t]{0.48\columnwidth}
		\centering
		\includegraphics[width=0.9\textwidth]{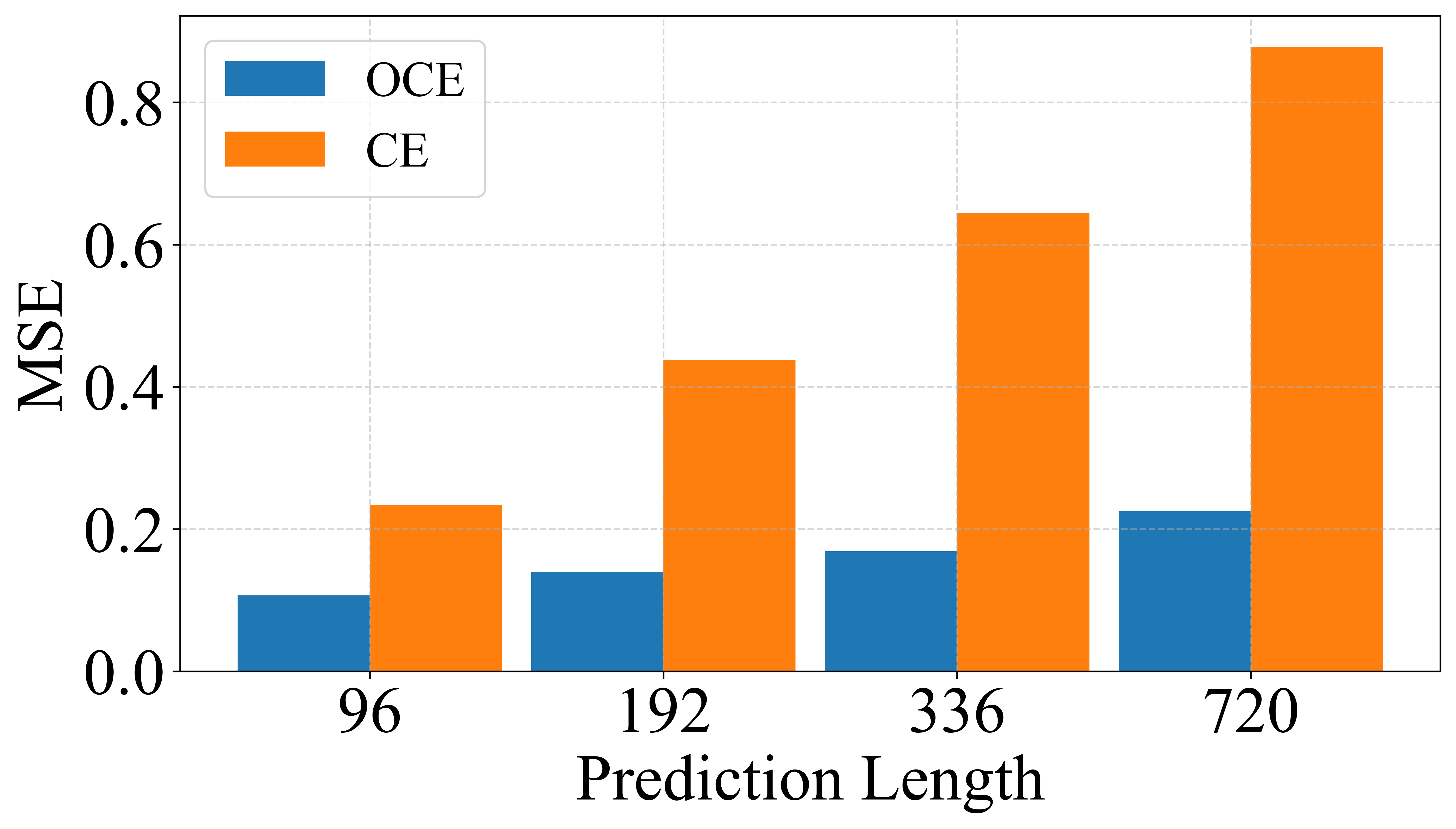}
		\caption{MSE (CE vs. OCE)}  % 明确标注两种损失函数
		\label{fig:mse-ce-oce}  % 标签体现“MSE+损失函数对比”
	\end{subfigure}
	\hfill
	% 第二张子图（CE与OCE的MAE对比）
	\begin{subfigure}[t]{0.48\columnwidth}
		\centering
		\includegraphics[width=0.9\textwidth]{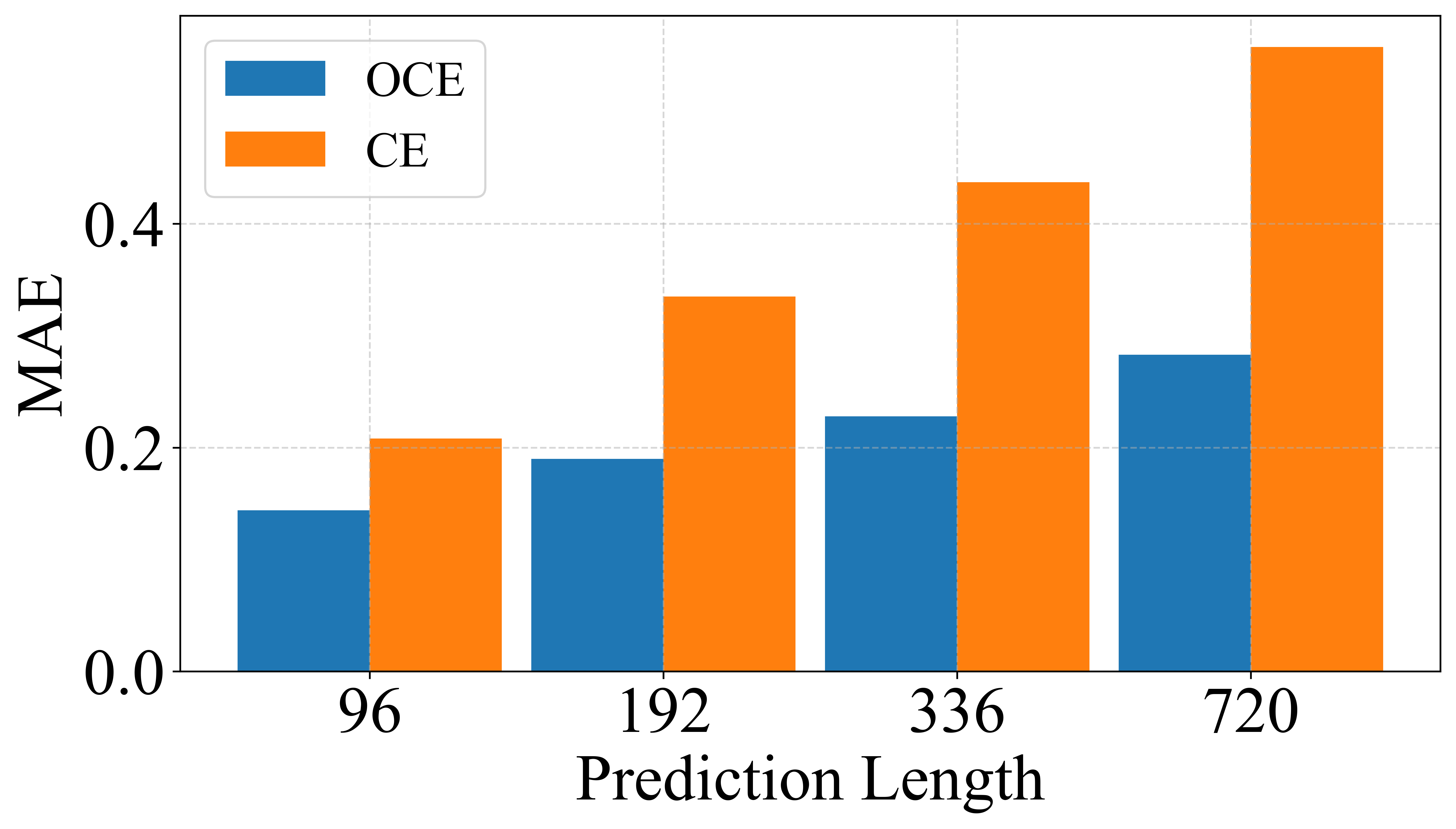}
		\caption{MAE (CE vs. OCE)}
		\label{fig:mae-ce-oce}
	\end{subfigure}
	\caption{The Comparison between CE Loss and OCE Loss}
	\label{fig:ce-oce-comparison}
\end{figure}

\subsection{Model Performance under Different SNRs}
This experiment compares the  OCE-TS method with DLinear under varying signal-to-noise ratios (SNRs) (\SI{-3}{dB}, \SI{0}{dB}, \SI{3}{dB}, \SI{10}{dB}, and \SI{20}{dB}), where additive Gaussian white noise is injected into inputs to evaluate robustness on the ETTm2 and Weather datasets.
Please refer to Appendix~\ref{app:SNR} for complete experimental results and details.

\begin{table}[h]
	\centering
	\scriptsize
	\setlength{\tabcolsep}{2.5pt} % 增加列间距
	\begin{tabular}{l|cccccccccc}
		\toprule
		\multirow{2}{*}{\textbf{Dataset}} & \multicolumn{2}{c}{\textbf{-3dB}} & \multicolumn{2}{c}{\textbf{0dB}} & \multicolumn{2}{c}{\textbf{3dB}} & \multicolumn{2}{c}{\textbf{10dB}} & \multicolumn{2}{c}{\textbf{20dB}} \\
		\cmidrule(lr){2-3} \cmidrule(lr){4-5} \cmidrule(lr){6-7} \cmidrule(lr){8-9}\cmidrule(lr){10-11}
		& MSE & Ours & MSE & Ours & MSE & Ours & MSE & Ours & MSE & Ours \\
		\midrule
		\textbf{ETTm2}& 0.440 & \textbf{0.275} & 0.431 & \textbf{0.274} &0.425  & \textbf{0.273} & 0.420 & \textbf{0.277}&0.420  & \textbf{0.280} \\
		\textbf{Weather}&0.342  & \textbf{0.227} & 0.339 & \textbf{0.225} &0.338  & \textbf{0.224} & 0.337 & \textbf{0.223}& 0.338 & \textbf{0.224} \\
		\bottomrule
	\end{tabular}
	%\caption{Impact of Noise Levels on MSE}
\caption{Comparison Under Different Noise Levels}
	\label{Table_noise}
\end{table}

The experimental results show that under different SNR levels, the performance of our method is superior to that of DLinear, and the impact of noise on the model's performance is minimal, demonstrating its robustness.

\subsection{Significance Analysis}
To evaluate the statistical significance of performance differences among the compared algorithms, we conduct a critical difference (CD) analysis based on both MAE and MSE metrics. The x-axis of Figure~\ref{fig:cd_analysis} represents the average rankings, and algorithms connected by horizontal lines indicate no statistically significant differences. Our approach exhibits prominent superiority, with its leftmost position on the CD axis confirming statistically significant advantages over counterparts. Please refer to Appendix~\ref{app:siga} for more details.

\begin{figure}[h]
	\centering
	% 第一张子图：CD 图的排名点分布
	\begin{subfigure}[t]{0.48\columnwidth}
		\centering
		\includegraphics[width=\textwidth]{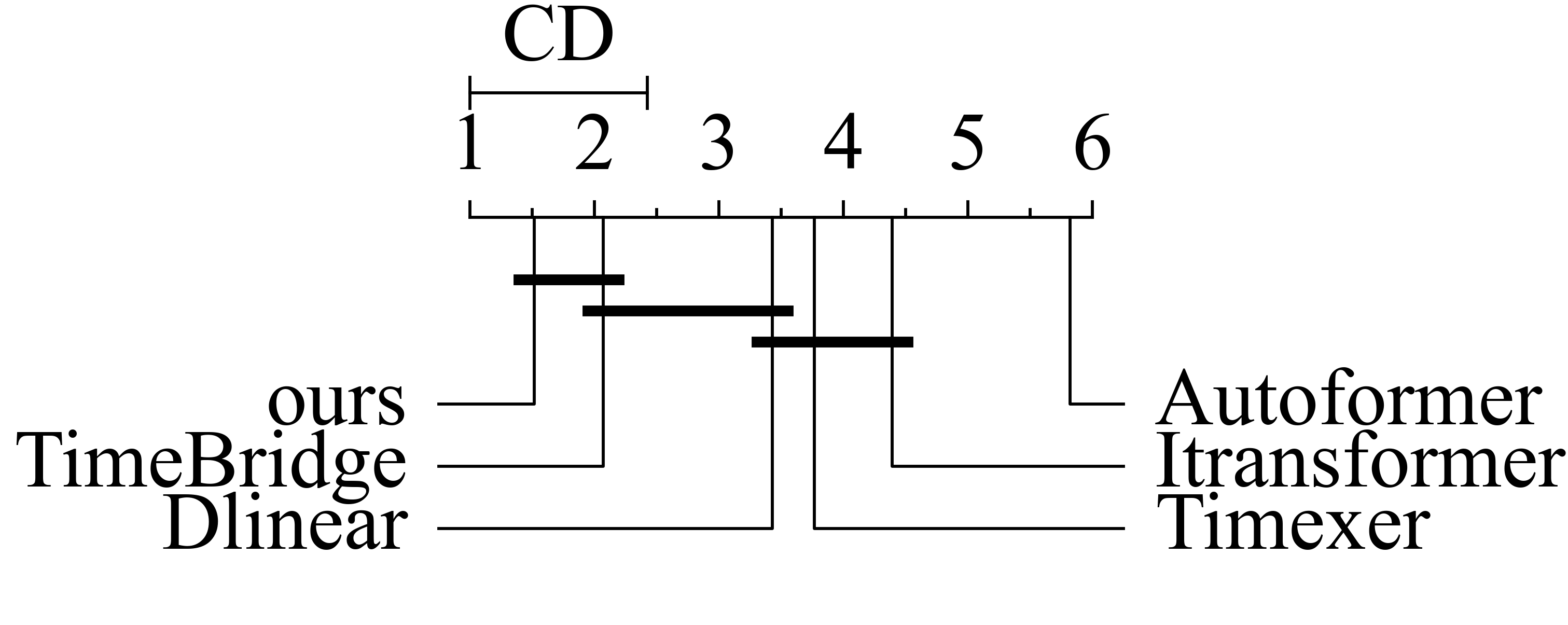}
		\caption{MSE}
		\label{fig:cd_diagram_mse}
	\end{subfigure}
	\hfill
	% 第二张子图：CD 图的显著性差异
	\begin{subfigure}[t]{0.48\columnwidth}
		\centering
		\includegraphics[width=\textwidth]{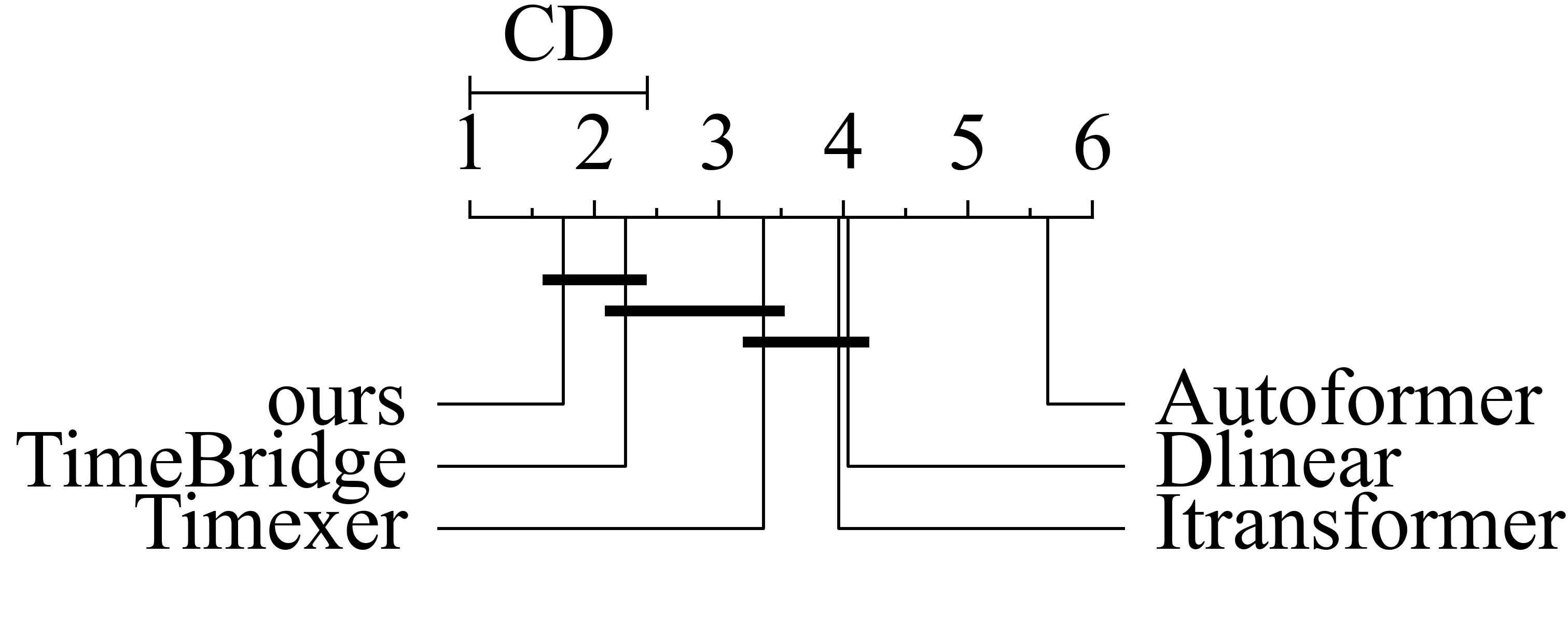}
		\caption{MAE}
		\label{fig:cd_diagram_mae}
	\end{subfigure}
	\caption{Critical Difference Analysis }
	\label{fig:cd_analysis}
\end{figure}

\section{Conclusion}
Temporal sequence prediction plays a crucial role across various domains, yet existing models generally suffer from two major limitations: insufficient uncertainty quantification and weak robustness. To address this, this paper proposes the OCE-TS framework, which synergistically integrates regression-as-classification with ordinal cross-entropy optimization to achieve more accurate predictive distribution modeling. Experiments on multiple benchmark datasets demonstrate that this approach significantly outperforms baseline models in both prediction accuracy and stability. Future work will explore more expressive distribution modeling techniques to capture complex uncertainty structures.

\bibliography{aaai2026}

\fi
\setlength{\leftmargini}{20pt}
\makeatletter\def\@listi{\leftmargin\leftmargini \topsep .5em \parsep .5em \itemsep .5em}
\def\@listii{\leftmargin\leftmarginii \labelwidth\leftmarginii \advance\labelwidth-\labelsep \topsep .4em \parsep .4em \itemsep .4em}
\def\@listiii{\leftmargin\leftmarginiii \labelwidth\leftmarginiii \advance\labelwidth-\labelsep \topsep .4em \parsep .4em \itemsep .4em}\makeatother

\setcounter{secnumdepth}{0}
\renewcommand\thesubsection{\arabic{subsection}}
\renewcommand\labelenumi{\thesubsection.\arabic{enumi}}

\newcounter{checksubsection}
\newcounter{checkitem}[checksubsection]

\newcommand{\checksubsection}[1]{%
	\refstepcounter{checksubsection}%
	\paragraph{\arabic{checksubsection}. #1}%
	\setcounter{checkitem}{0}%
}

\newcommand{\checkitem}{%
	\refstepcounter{checkitem}%
	\item[\arabic{checksubsection}.\arabic{checkitem}.]%
}
\newcommand{\question}[2]{\normalcolor\checkitem #1 #2 \color{blue}}
\newcommand{\ifyespoints}[1]{\makebox[0pt][l]{\hspace{-15pt}\normalcolor #1}}

\newpage
\clearpage
\appendix

% 创建横跨双栏的 Appendix 标题
\twocolumn[%
\begin{@twocolumnfalse}
	\begin{center}
		\Huge\textbf{Appendix}
	\end{center}
	\vspace{2em}
\end{@twocolumnfalse}
]

The appendix includes:
\begin{itemize}
	\item \textbf{A Definition and Proof.}
	\item \textbf{B Algorithm Framework.}
	\item \textbf{C More Experimental Results.}
\end{itemize}

\section {Definition and Proof}

\subsection{The Definition of Cross-Entropy Loss}
\label{app:CE}
\begin{definition}[Standard Cross-Entropy Loss (General Form)]
	Let $Y \in \{1,\dots,K\}$ be the true class label and $\hat{Y}$ the predicted class. Given the predicted probability distribution $\bm{q} = [q_1,\dots,q_K]$ where $q_k = \mathbb{P}(\hat{Y} = k)$ and the true probability distribution $\bm{p} = [p_1,\dots,p_K]$, the cross-entropy loss is:
	\begin{equation}\label{eq:ce_general}
		\mathcal{L}_{\text{CE}}(Y, \hat{Y}) = -\sum_{k=1}^K p_k \log q_k.
	\end{equation}
	For hard labels (one-hot encoding):
	\begin{equation}
		p_k = \mathbb{P}_{\text{true}}(Y = k) = \begin{cases}
			1 & \text{if } Y = k ,\\
			0 & \text{otherwise}.
		\end{cases}
	\end{equation}
	The cross-entropy loss simplifies to:
	\begin{equation}\label{eq:ce_onehot}
		\mathcal{L}_{\text{CE}}(Y, \hat{Y}) = -\log q_Y.
	\end{equation}
	where $q_Y$ is the predicted probability for the true class.
\end{definition}

\subsection{Influence Function and Related Work}
Influence functions (IF)~\cite{Hampel1974Influence} quantify the marginal impact of individual training samples on model parameters or predictions. First, for efficient computation,~\cite{pmlr-v70-koh17a} introduces stochastic Hessian inverse approximation for large models and~\cite{Schioppa2022Scaling} implements Arnoldi iteration to accelerate this for Transformers. Second, for interpretability and robustness, IFs identify influential training samples for debugging, anomaly detection, and robustness evaluation~\cite{Pruthi2020Estimating}, reveal model sensitivity to input perturbations~\cite{Basu2020Influence}, and aid adversarial detection and data quality assessment.

%\subsection{ The Definition of Influence Function}
\begin{definition}[Influence Function]
	Let \( T \) be a statistical functional mapping from a probability distribution to a real number (or vector), such as the Mean $ T(P) = \mathbb{E}_P[X] = \int x \, dP(x) $.
	Let \( P_\epsilon \) be the contaminated distribution with
	a mixture of the true distribution \( P \) and a Dirac delta distribution \( \delta_y \) at \( y \):
	\begin{equation}
		P_\epsilon = (1 - \epsilon)P + \epsilon \delta_y,
	\end{equation}
	where \( \epsilon \in [0, 1) \) is the contamination proportion. Intuitively, \( P_\epsilon \) represents a scenario where most of the data comes from \( P \), but a small fraction \( \epsilon \) is replaced by an outlier at \( y \).
\end{definition}

The influence function (IF) is a fundamental tool in robust statistics that quantifies the effect of infinitesimal contamination on a statistical functional \( T \)~\cite{pmlr-v70-koh17a}.
It is defined as the directional derivative of \( T \) at \( P \) in the direction of \( \delta_y \):
\begin{equation}
	\text{IF}(y; T, P) = \lim_{\epsilon \to 0} \frac{T(P_\epsilon) - T(P)}{\epsilon}.
\end{equation}
If \( \text{IF}(y; T, P) \) is unbounded for some \( y \), \( T \) is non-robust (e.g., the mean). If it is bounded, \( T \) is robust (e.g., the median).
IF describes how much \( T \) changes when a small fraction of data is perturbed at \( y \).

\subsection{ The Influence Function of Model Parameter}
In machine learning,
influence function analysis quantifies the contribution of individual training samples to model parameters or predictions, revealing data importance and enhancing model interpretability. It can identify critical samples (e.g., marginal data or outliers), detect noisy or biased data, and assess model robustness.

%\begin{align}
%&\frac{d}{d\epsilon}\left[(1 - \epsilon) \nabla_\theta \mathbb{E}_{P_n}[L(z, \hat{\theta}_\epsilon)] + \epsilon \nabla_\theta L(z, \hat{\theta}_\epsilon)\right] = 0 \nonumber \\
%&= -\nabla_\theta \mathbb{E}_{P_n}[L(z, \hat{\theta}_\epsilon)] + (1 - \epsilon)\frac{d}{d\epsilon}\nabla_\theta \mathbb{E}_{P_n}[L(z, \hat{\theta}_\epsilon)] \nonumber \\
%&\quad + \nabla_\theta L(z, \hat{\theta}_\epsilon) + \epsilon \frac{d}{d\epsilon}\nabla_\theta L(z, \hat{\theta}_\epsilon) = 0
%\end{align}

For a model parameterized by $\bm{\theta} \in \Theta$, trained via empirical risk minimization (ERM) with loss $L(\bm{z}, \bm{\theta})$, the empirical distribution $P_n = \frac{1}{n} \sum_{i=1}^n \delta_{\bm{z}_i}$. The parameter estimate $\hat{\bm{\theta}}$ is obtained by minimizing the expected loss over the empirical distribution $P_n$:
\begin{equation}
	\hat{\bm{\theta}} = \arg\min_{\bm{\theta}} \mathbb{E}_{\bm{z} \sim P_n} [L(\bm{z}, \bm{\theta})] = \arg\min_{\bm{\theta}} \frac{1}{n} \sum_{i=1}^n L(\bm{z}_i, \bm{\theta}).
\end{equation}
To compute the influence of a new point $\bm{z}$, consider a perturbed distribution:
\begin{equation}
	P_\epsilon = (1 - \epsilon)P_n + \epsilon \delta_{\bm{z}},
\end{equation}
where $\epsilon$ is a small weight added to the point $\bm{z}$. The new parameter estimate $\hat{\bm{\theta}}_\epsilon$ minimizes:
\begin{equation}
	\mathbb{E}_{\bm{z} \sim P_\epsilon}[L(\bm{z}, \bm{\theta})] = (1 - \epsilon) \mathbb{E}_{\bm{z} \sim P_n}[L(\bm{z}, \bm{\theta})] + \epsilon L(\bm{z}, \bm{\theta}).
\end{equation}

To study the sensitivity of the loss function to perturbations, we derive the derivative \(\left.\frac{d\hat{\bm{\theta}}_\epsilon}{d\epsilon}\right|_{\epsilon=0}\), thereby defining the influence function of $\bm{\theta}$.
Analyzing the behavior at $\epsilon = 0$ (i.e., the derivative at zero perturbation) is to study the model's sensitivity to infinitesimal data perturbations. According to the
first-order optimality condition and the implicit function theorem (IFT), we obtain
the influence function $\text{IF}(\bm{z}; \hat{\bm{\theta}}, P)$ is:
\begin{equation}
	\text{IF}(\bm{z}; \hat{\bm{\theta}}, P) = \left. \frac{d\hat{\bm{\theta}}_\epsilon}{d\epsilon} \right|_{\epsilon=0} = -H_{\hat{\bm{\theta}}}^{-1} \nabla_{\bm{\theta}} L(\bm{z}, \hat{\bm{\theta}}).
	\label{IFtheta}
\end{equation}
This quantifies how much $\hat{\bm{\theta}}$ changes when an infinitesimal weight is added at $\bm{z}$.

The details of Equation (\ref{IFtheta}) is obtained as follows.
The Implicit Function Theorem is used to derive how $\hat{\bm{\theta}}_\epsilon$ changes with $\epsilon$ at $\epsilon = 0$. Firstly, according to
the first-order optimality condition,
at $\hat{\bm{\theta}}_\epsilon$, the gradient of the perturbed risk is zero:
\begin{equation}
	(1 - \epsilon) \nabla_{\bm{\theta}} \mathbb{E}_{P_n}[L(\bm{z}, \hat{\bm{\theta}}_\epsilon)] + \epsilon \nabla_{\bm{\theta}} L(\bm{z}, \hat{\bm{\theta}}_\epsilon) = 0.
\end{equation}

Secondly, differentiate the above equality w.r.t. $\epsilon$, then apply the chain rule and evaluate at $\epsilon = 0$, we have:
\begin{align}
	&  -\nabla_{\bm{\theta}} \mathbb{E}_{P_n}[L(\bm{z}, \hat{\bm{\theta}})] + \nabla_{\bm{\theta}} L(\bm{z}, \hat{\bm{\theta}})  \\&
	\quad \quad +\mathbb{E}_{P_n}[\nabla_{\bm{\theta}}^2 L(\bm{z}, \hat{\bm{\theta}})] \cdot \left. \frac{d\hat{\bm{\theta}}_\epsilon}{d\epsilon} \right|_{\epsilon=0}
	= 0.
\end{align}
The first term vanishes because $\hat{\bm{\theta}}$ minimizes the empirical risk (gradient is zero).
Finally, solve for the derivative, we have:
\begin{equation}
	\left. \frac{d\hat{\bm{\theta}}_\epsilon}{d\epsilon} \right|_{\epsilon=0} = -H_{\hat{\bm{\theta}}}^{-1} \nabla_{\bm{\theta}} L(\bm{z}, \hat{\bm{\theta}}),
	\label{IFdef}
\end{equation}
where $H_{\hat{\bm{\theta}}} = \mathbb{E}_{P_n}[\nabla_{\bm{\theta}}^2 L(\bm{z}, \hat{\bm{\theta}})]$ is the empirical Hessian (assumed invertible).

\subsection{ Prerequisites for Influence Function Analysis}

For matrices $\bm{A} \in \mathbb{R}^{m \times n}$ and $\bm{B} \in \mathbb{R}^{p \times q}$, their Kronecker product is defined as the block matrix \cite{horn2012matrix}:
\[
\bm{A} \otimes \bm{B} = \begin{bmatrix}
	a_{11}\bm{B} & \cdots & a_{1n}\bm{B} \\
	\vdots & \ddots & \vdots \\
	a_{m1}\bm{B} & \cdots & a_{mn}\bm{B}
\end{bmatrix} \in \mathbb{R}^{mp \times nq}
\]

If $\bm{\Sigma}_X \in \mathbb{R}^{d \times d}$ and $\bm{P} \in \mathbb{R}^{K \times K}$ are invertible, then:
\[
(\bm{\Sigma}_X \otimes \bm{P})^{-1} = \bm{\Sigma}_X^{-1} \otimes \bm{P}^{-1}
\]

Let $\bm{A} \in \mathbb{R}^{n \times n}$ be a real symmetric matrix. Its spectral norm is defined as:
\begin{align}
	\|\bm{A}\|_2 & = \max_{\bm{x} \neq \bm{0}} \frac{\|\bm{Ax}\|_2}{\|\bm{x}\|_2} \\ \notag
	& = \max \{ |\lambda| : \lambda \text{ is an eigenvalue of } \bm{A} \},
\end{align}
That is:
\begin{equation}
	\|\bm{A}\|_2 = |\lambda_{\max}(\bm{A})|.
\end{equation}

For arbitrary matrices $\bm{A}$, $\bm{B}$:
\[
\|\bm{A} \otimes \bm{B}\|_2 = \|\bm{A}\|_2 \cdot \|\bm{B}\|_2
\]

If $\bm{X}$ is symmetric positive definite, the condition number of matrix $\bm{X} \in \mathbb{R}^{n \times n}$ measures its numerical stability:
\[
\kappa_2(\bm{X}) := \|\bm{X}\|_2 \cdot \|\bm{X}^{-1}\|_2 = \frac{\lambda_{\max}(\bm{X})}{\lambda_{\min}(\bm{X})}.
\]

\subsubsection{Bound on Probability Residual Norm}

\begin{lemma}
	\label{lem:prob_bound}
	Let $\bm{\sigma} \in \mathbb{R}^K$ be a probability vector from the softmax function (i.e., $\sigma_k \geq 0$, $\sum_{k=1}^K \sigma_k = 1$), and $\mathbf{e}_y \in \{0,1\}^K$ be the one-hot encoded ground-truth label. The Euclidean norm of the probability residual satisfies:
	\begin{equation}
		\|\bm{\sigma} - \mathbf{e}_y\|_2 \leq \sqrt{2},
		\label{eq:main_bound}
	\end{equation}
	with equality achieved when $\sigma_y = 0$ (i.e., the model assigns zero probability to the true class).
\end{lemma}

\begin{proof}
	Expand the squared norm:
	\begin{equation}
		\|\bm{\sigma} - \mathbf{e}_y\|_2^2 = \sum_{k=1}^K (\sigma_k - \delta_{k,y})^2 = (1 - \sigma_y)^2 + \sum_{k \neq y} \sigma_k^2,
		\label{eq:decomp}
	\end{equation}
	where $\delta_{k,y}$ is the Kronecker delta ($\delta_{k,y} = 1$ if $k = y$, else $0$).
	
	Since $\sum_{k \neq y} \sigma_k = 1 - \sigma_y$ and $\sigma_k \in [0,1]$, the sum of squares $\sum_{k \neq y} \sigma_k^2$ is maximized when all $\sigma_k$ for $k \neq y$ are equal (by convexity).
	
	The worst case occurs when $\sigma_y = 0$ and $\sigma_k = \frac{1}{K-1}$ for $k \neq y$:
	\begin{align}
		\|\bm{\sigma} - \mathbf{e}_y\|_2^2 & = (1 - 0)^2 + \sum_{k \neq y} \left(\frac{1}{K-1}\right)^2 ,\\
		& = 1 + \frac{1}{K-1} \leq 2 \quad (\forall K \geq 2).
		\label{eq:worst_case}
	\end{align}
\end{proof}

\subsection{ IF for MSE Loss}

We consider the standard linear regression model $y = \bm{x}^\top\bm{\theta} + \epsilon$, where $\bm{x} \in \mathbb{R}^d$ is the $d$-dimensional feature vector, $\bm{\theta} \in \mathbb{R}^d$ is the $d$-dimensional parameter vector, $\epsilon \in \mathbb{R}$ is the scalar noise term, and $y \in \mathbb{R}$ is the scalar response.

The mean squared error loss function for a data point $\bm{z} = (\bm{x},y)$ is given by:
\begin{equation}
	L(\bm{z},\bm{\theta}) = \frac{1}{2}(y - \bm{x}^\top\bm{\theta})^2 \in \mathbb{R}
\end{equation}

The gradient and Hessian of the loss function take the forms:
\begin{align}
	\nabla_\theta L(\bm{z},\bm{\theta}) &= -\bm{x}(y - \bm{x}^\top\bm{\theta}) \in \mathbb{R}^d, \\
	\nabla_\theta^2 L(\bm{z},\bm{\theta}) & = \bm{x}\bm{x}^\top \in \mathbb{R}^{d\times d},
\end{align}
where the expected Hessian $\bm{H}_\theta = \mathbb{E}[\bm{x}\bm{x}^\top] \in \mathbb{R}^{d\times d}$ corresponds to the $d\times d$ feature covariance matrix.

The influence function for MSE can be derived as:
\begin{equation}
	\text{IF}(\bm{z};\hat{\bm{\theta}}) = \left(\mathbb{E}[\bm{x}\bm{x}^\top]\right)^{-1}\bm{x}(y - \bm{x}^\top\hat{\bm{\theta}}) \in \mathbb{R}^d
\end{equation}
which quantifies the $d$-dimensional effect of individual data points on parameter estimates.

By the submultiplicative property of spectral norm, we have the upper Bound for $\|\text{IF}_{\text{MSE}}\|_2$:
\begin{align*}
	&\|\text{IF}_{\text{MSE}}\|_2 = \|\bm{\Sigma}_X^{-1} \bm{x} (y - \bm{x}^\top\bm{\theta})\|_2 \\
	&\leq \|\bm{\Sigma}_X^{-1}\|_2 \cdot \|\bm{x}\|_2 \cdot |y - \bm{x}^\top\bm{\theta}|.
\end{align*}
By the definition of spectral norm, we have the lower Bound for $\|\text{IF}_{\text{MSE}}\|_2$:
\begin{align*}
	&\|\text{IF}_{\text{MSE}}\|_2 = \|\bm{\Sigma}_X^{-1} \bm{x} (y - \bm{x}^\top\bm{\theta})\|_2 \\
	&\geq \sigma_{\min}(\bm{\Sigma}_X^{-1}) \cdot \|\bm{x}\|_2 \cdot |y - \bm{x}^\top\bm{\theta}| \\
	&= \frac{\|\bm{x}\|_2 \cdot |y - \bm{x}^\top\bm{\theta}|}{\|\bm{\Sigma}_X\|_2} \quad \text{(since $\sigma_{\min}(\bm{A}^{-1}) = 1/\|\bm{A}\|_2$)}.
\end{align*}

The norm of the influence function $\|\text{IF}_{\text{MSE}}\|_2$ becomes large under several key conditions:

First, when samples have large prediction errors $\lvert y - \bm{x}^\top\bm{\theta} \rvert \gg 0$ (such as outliers), both the upper and lower bounds increase substantially, leading to significant growth in $\|\text{IF}_{\text{MSE}}\|_2$.

Second, the bounds scale linearly with the feature norm $\|\bm{x}\|_2 \gg 0$, meaning that samples with extreme feature values (known as high leverage points) exert disproportionate influence on the estimator.

Third, an ill-conditioned covariance matrix, characterized by $\|\bm{\Sigma}_X^{-1}\|_2 \gg 0$ or equivalently $\lambda_{\min}(\bm{\Sigma}_X) \approx 0$, significantly affects the influence function. This condition primarily relates to the shape of the data distribution and occurs when features exhibit near-linear dependence (collinearity). The ill-conditioning amplifies the upper bound of the influence function through the term $\|\bm{\Sigma}_X^{-1}\|_2$.

Fourth, a concentrated data distribution, indicated by $\|\bm{\Sigma}_X\|_2 \ll 1$ or equivalently $\lambda_{\max}(\bm{\Sigma}_X) \ll 1$, also increases the influence function's sensitivity. This condition reflects the overall scale of the data, showing that all features have small magnitudes (low variance). The concentration effect primarily impacts the lower bound via $\|\bm{\Sigma}_X\|_2$, making even moderate residuals more influential when the data variance is small.

These conditions interact in important ways. The ill-conditioned covariance matrix focuses on the relative scaling between features (anisotropic stretching), while the concentrated distribution concerns the absolute scale of the data (uniform shrinkage). When both conditions occur simultaneously - with $\lambda_{\min}(\bm{\Sigma}_X) \approx 0$ and $\lambda_{\max}(\bm{\Sigma}_X) \ll 1$ - their combined effect can make the estimator particularly sensitive to individual data points. Understanding these relationships helps explain why certain data characteristics lead to larger influence function values and provides insights for improving model robustness through techniques like regularization and data preprocessing.

\subsection{ IF for Classification}
\label{app:IF}
Consider a $K$-class classification problem with a parametric softmax model. Let $\bm{x} \in \mathbb{R}^d$ be an input feature vector and $y \in \{1,...,K\}$ its corresponding class label. The model's predicted probabilities are given by:
\begin{equation}
	p(y|\bm{x}; \bm{\beta}) = \sigma(\bm{x}^T \bm{\beta})_y = \frac{e^{\bm{x}^T \bm{\beta}_y}}{\sum_{k=1}^K e^{\bm{x}^T \bm{\beta}_k}},
\end{equation}
where $\bm{\beta} = [\bm{\beta}_1,...,\bm{\beta}_K] \in \mathbb{R}^{d \times K}$ contains the model parameters. The cross-entropy loss for a single observation $\bm{z} = (\bm{x}, y)$ is:
\begin{equation}
	L(\bm{z}, \bm{\beta}) = -\sum_{k=1}^K \mathbb{I}_{y=k} \log \sigma(\bm{x}^T \bm{\beta})_k,
\end{equation}
where $\mathbb{I}_{y=k}$ denotes the indicator function that equals 1 when $y=k$ and 0 otherwise. This loss function is widely used in classification tasks due to its desirable properties, including differentiability and appropriateness for probability estimation.
From Eq. (\ref{IFdef}), computing the influence function (IF) requires both the first-order derivative (gradient) and the Hessian matrix.

\subsubsection{Gradient Derivation}
The gradient of the cross-entropy loss with respect to the model parameters $\bm{\beta}$ can be computed as follows. For a single observation $\bm{z} = (\bm{x},y)$, by the chain rule, we have:
\begin{align}
	\nabla_{\bm{\beta}} L(\bm{z}, \bm{\beta}) &= \nabla_{\bm{\beta}} \left[ -\log \sigma(\bm{x}^T \bm{\beta})_y \right] ,\\ \notag
	&= -\frac{1}{\sigma(\bm{x}^T \bm{\beta})_y} \nabla_{\bm{\beta}} \sigma(\bm{x}^T \bm{\beta})_y,
\end{align}

The gradient of the softmax function has two cases. For the true class $y$:
\begin{equation}
	\frac{\partial \sigma(\bm{x}^T \bm{\beta})_y}{\partial \bm{\beta}_y} = \sigma(\bm{x}^T \bm{\beta})_y (1 - \sigma(\bm{x}^T \bm{\beta})_y) \bm{x},
\end{equation}

For any other class $k \neq y$:
\begin{equation}
	\frac{\partial \sigma(\bm{x}^T \bm{\beta})_y}{\partial \bm{\beta}_k} = -\sigma(\bm{x}^T \bm{\beta})_y \sigma(\bm{x}^T \bm{\beta})_k \bm{x},
\end{equation}

Combining these partial derivatives, we obtain the full gradient:
\begin{equation}
	\nabla_{\bm{\beta}} L(\bm{z}, \bm{\beta}) = \bm{x} \left( \sigma(\bm{x}^T \bm{\beta}) - \mathbf{e}_y \right) \in \mathbb{R}^{d \times K},
\end{equation}
where $\mathbf{e}_y$ is the one-hot vector with 1 at position $y$ and 0 elsewhere. This compact form reveals several important properties: the gradient magnitude scales linearly with the input features $\bm{x}$; the direction depends on the discrepancy between predictions $\sigma(\bm{x}^T \bm{\beta})$ and ground truth $\mathbf{e}_y$; for correctly classified examples with high confidence ($\sigma(\bm{x}^T \bm{\beta})_y \approx 1$), the gradient approaches zero.

\subsubsection{Derivation of the Hessian Matrix}
The Hessian matrix $H_{\bm{\beta}}$ is a second-order derivative matrix of dimension $(dK) \times (dK)$. It can be partitioned into $K \times K$ block matrices, where each block $H_{ij} \in \mathbb{R}^{d \times d}$ corresponds to the interaction between parameters $\bm{\beta}_i$ and $\bm{\beta}_j$:
\begin{equation}
	H_{\bm{\beta}} = \begin{bmatrix}
		\frac{\partial^2 L}{\partial \bm{\beta}_1 \partial \bm{\beta}_1} & \cdots & \frac{\partial^2 L}{\partial \bm{\beta}_1 \partial \bm{\beta}_K} \\
		\vdots & \ddots & \vdots \\
		\frac{\partial^2 L}{\partial \bm{\beta}_K \partial \bm{\beta}_1} & \cdots & \frac{\partial^2 L}{\partial \bm{\beta}_K \partial \bm{\beta}_K}
	\end{bmatrix}
\end{equation}

For any $i,j \in \{1,\dots,K\}$, we compute the block matrix $H_{ij} = \frac{\partial^2 L}{\partial \bm{\beta}_i \partial \bm{\beta}_j}$.

First, consider the case where $i = j = y$ (true class). The Hessian block is derived through:
\begin{align}
	H_{yy} &= \frac{\partial}{\partial \bm{\beta}_y} \left[ \bm{x} \left( \sigma(\bm{x}^\top \bm{\beta})_y - 1 \right) \right], \\
	&= \bm{x} \bm{x}^\top \sigma(\bm{x}^\top \bm{\beta})_y (1 - \sigma(\bm{x}^\top \bm{\beta})_y),
\end{align}
This expression captures the curvature information specific to the correct classification decision.

Next, for the case where either $i = y, j \neq y$ or $i \neq y, j = y$, the off-diagonal blocks take the form:
\begin{align}
	H_{yk} &= \frac{\partial}{\partial \bm{\beta}_k} \left[ \bm{x} \left( \sigma(\bm{x}^\top \bm{\beta})_y \right) \right], \\
	&= -\bm{x} \bm{x}^\top \sigma(\bm{x}^\top \bm{\beta})_y \sigma(\bm{x}^\top \bm{\beta})_k, \quad k \neq y
\end{align}
These terms quantify how parameter adjustments for incorrect classes affect the true class gradient.

Finally, for cases where $i,j \neq y$ (both non-true classes), the Hessian blocks exhibit a symmetric pattern:
\begin{equation}
	H_{ij} = \begin{cases}
		\bm{x} \bm{x}^\top \sigma(\bm{x}^\top \bm{\beta})_i (1 - \sigma(\bm{x}^\top \bm{\beta})_i), & i = j \\
		-\bm{x} \bm{x}^\top \sigma(\bm{x}^\top \bm{\beta})_i \sigma(\bm{x}^\top \bm{\beta})_j, & i \neq j
	\end{cases}
\end{equation}
revealing how incorrect classes interact with each other in parameter space.

The complete Hessian matrix admits an elegant Kronecker product representation. By defining $\bm{P}  = \text{diag}(\sigma(\bm{x}^\top \bm{\beta})) - \sigma(\bm{x}^\top \bm{\beta}) \sigma(\bm{x}^\top \bm{\beta})^\top \in \mathbb{R}^{K \times K}$ as the covariance matrix of softmax outputs, we obtain:
\begin{equation}
	H_{\bm{\beta}} = \nabla_{\bm{\beta}}^2 L(\bm{z}, \bm{\beta}) = \bm{x} \bm{x}^\top \otimes\bm{P},
\end{equation}
where $\otimes$ denotes the Kronecker product.

This compact formulation explicitly shows the Hessian's positive semi-definite nature, as both $\bm{x} \bm{x}^\top$ and the softmax covariance matrix $\bm{P}$ are themselves positive semi-definite. The Kronecker structure separates the input feature interactions (through $\bm{x} \bm{x}^\top$) from the class probability relationships (through $\bm{P}$), providing insight into the loss landscape's fundamental geometry.

The expected Hessian over $x$ is:
\begin{align}
	\mathbb{E}[H_{\bm{\beta}}] \nonumber \\
	\makebox[0pt][r]{${}={}$} \mathbb{E} &\left[ \bm{x} \bm{x}^\top \otimes \left( \text{diag}(\sigma(\bm{x}^\top \bm{\beta})) - \sigma(\bm{x}^\top \bm{\beta}) \sigma(\bm{x}^\top \bm{\beta})^\top \right) \right].
\end{align}

\subsubsection{IF for Cross Entropy Loss}
Based on the gradient and Hessian matrix of the cross-entropy loss, the complete mathematical expression of the influence function is as follows:
\begin{equation}
	\text{IF}(\bm{z}; \bm{\beta}) = -\mathbf{H}_{\bm{\beta}}^{-1} \cdot \text{vec}\left( \nabla_{\bm{\beta}} L(\mathbf{z}, \bm{\beta}) \right) \in \mathbb{R}^{dK \times 1},
\end{equation}
where $\text{vec}(\bm{\alpha})\in \mathbb{R}^{dK \times 1}$ is the matrix vectorization operation for
$\bm{\alpha} \in \mathbb{R}^{d \times K}$.

The gradient term is given by:
\begin{equation}
	\nabla_{\bm{\beta}} L(\bm{z}, \bm{\beta}) = \bm{x} \left( \sigma(\bm{x}^\top \bm{\beta}) - \bm{e}_y \right) \in \mathbb{R}^{d \times K},
\end{equation}
where $\bm{x} \in \mathbb{R}^d$ is the input feature vector, $\sigma(\bm{x}^\top \bm{\beta}) \in \mathbb{R}^K$ represents the model's predicted softmax probabilities, and $\bm{e}_y \in \mathbb{R}^K$ is the one-hot encoding of the true label.

The Hessian inverse term takes the form:
\begin{equation}
	\bm{H}_{\bm{\beta}}^{-1} = \left( \mathbb{E}\left[ \bm{x}\bm{x}^\top \otimes \bm{P} \right] \right)^{-1},
\end{equation}
\begin{equation}
	\bm{P} = \text{diag}(\sigma(\bm{x}^\top \bm{\beta})) - \sigma(\bm{x}^\top \bm{\beta})\sigma(\bm{x}^\top \bm{\beta})^\top,
\end{equation}
Here $\bm{P} \in \mathbb{R}^{K \times K}$ denotes the covariance matrix of softmax outputs, and $\otimes$ is the Kronecker product that couples the feature space ($d \times d$) with the class space ($K \times K$).

By using the submultiplicative property of spectral norm  and Kronecker product norm, we have the
upper Bound for $\|\text{IF}_{\text{CE}}\|_2$:
\begin{align}
	&\|\text{IF}_{\text{CE}}\|_2 = \|(\mathbb{E}[\bm{x}\bm{x}^\top \otimes \bm{P}])^{-1} (\bm{x} \otimes (\bm{\sigma} - \bm{e}_y))\|_2 \\
	&\leq \|(\mathbb{E}[\bm{x}\bm{x}^\top \otimes \bm{P}])^{-1}\|_2 \cdot \|\bm{x} \otimes (\bm{\sigma} - \bm{e}_y)\|_2 \quad   \\
	&= \|(\bm{\Sigma}_X \otimes \bm{P})^{-1}\|_2 \cdot \|\bm{x}\|_2 \cdot \|\bm{\sigma} - \bm{e}_y\|_2 \\
	&= \|\bm{\Sigma}_X^{-1} \otimes \bm{P}^{-1}\|_2 \cdot \|\bm{x}\|_2 \cdot \|\bm{\sigma} - \bm{e}_y\|_2 \\
	&= \|\bm{\Sigma}_X^{-1}\|_2 \cdot \|\bm{P}^{-1}\|_2 \cdot \|\bm{x}\|_2 \cdot \|\bm{\sigma} - \bm{e}_y\|_2 \quad  \\
	&= \frac{1}{\lambda_{\min}(\bm{\Sigma}_X)} \cdot \frac{1}{\lambda_{\min}(\bm{P})} \cdot \|\bm{x}\|_2 \cdot \|\bm{\sigma} - \bm{e}_y\|_2 \\
	&\leq \frac{\sqrt{2} \|\bm{x}\|_2}{\lambda_{\min}(\bm{\Sigma}_X) \lambda_{\min}(\bm{P})} \quad \text{(since $\|\bm{\sigma} - \bm{e}_y\|_2 \leq \sqrt{2}$)}.
\end{align}

By the Matrix-vector norm inequality, we have the lower Bound for $\|\text{IF}_{\text{CE}}\|_2$:
\begin{align}
	&\|\text{IF}_{\text{CE}}\|_2 = \|(\bm{\Sigma}_X \otimes \bm{P})^{-1} (\bm{x} \otimes (\bm{\sigma} - \bm{e}_y))\|_2 \\
	&\geq \frac{\|\bm{x} \otimes (\bm{\sigma} - \bm{e}_y)\|_2}{\|\bm{\Sigma}_X \otimes \bm{P}\|_2}  \\
	&= \frac{\|\bm{x}\|_2 \cdot \|\bm{\sigma} - \bm{e}_y\|_2}{\|\bm{\Sigma}_X\|_2 \cdot \|\bm{P}\|_2}  .
\end{align}

The norm of the influence function $\|\text{IF}_{\text{CE}}\|_2$ becomes large under several key conditions:

First, the influence scales linearly with the input feature norm $\|\bm{x}\|_2$, indicating that samples with larger feature magnitudes (high-leverage points) disproportionately affect the model parameters.

Second, both bounds depend critically on the spectrum of $\bm{\Sigma}_X$. The upper bound grows as $\lambda_{\min}(\bm{\Sigma}_X)^{-1}$, showing extreme sensitivity to collinear features ($\lambda_{\min} \approx 0$), while the lower bound shrinks with $\|\bm{\Sigma}_X\|_2$, revealing that concentrated data distributions (small eigenvalues) make all samples more influential. This highlights the importance of regularization for stable influence patterns.

Third, the term $\lambda_{\min}(\bm{P})^{-1}$ creates a prediction confidence paradox: samples where the model is overconfident ($\lambda_{\min}(\bm{P}) \to 0$) can have unbounded influence. This reveals a fundamental tension in softmax models - high confidence predictions may indicate vulnerable decision boundaries overly sensitive to specific training points, regardless of prediction accuracy.

\subsection{Influence Function Growth Rate Comparison}
\begin{theorem}
	\label{IFratio}
	Suppose that the covariance matrix $\bm{\Sigma}_X = \mathbb{E}[\bm{x}\bm{x}^\top]$ is positive definite with $\lambda_{\min}(\bm{\Sigma}_X) > 0$ and
	the expected Softmax matrix $\bm{P}$ is positive definite with $\lambda_{\min}(\bm{P}) > 0$.
	Assume the input features have finite second moments ($\mathbb{E}[\|\bm{x}\|_2^2] < \infty$) and the sample size exceeds the feature dimension ($n > d$).
	For non-degenerate predictions where the MSE residual is non-zero ($y - \bm{x}^\top\bm{\theta} \neq 0$) and the CE probability deviation is non-trivial ($\bm{\sigma} - \bm{e}_y \neq \bm{0}$), the influence function ratio $R = \|\mathrm{IF}_{\mathrm{CE}}\|_2/\|\mathrm{IF}_{\mathrm{MSE}}\|_2$ satisfies the two-sided bound:
	\begin{equation}
		\frac{\|\bm{\sigma} - \bm{e}_y\|_2}{\kappa_2(\bm{\Sigma}_X)\lambda_{\max}(\bm{P})|y - \bm{x}^\top\bm{\theta}|}
		\leq R \leq
		\frac{\sqrt{2}\kappa_2(\bm{\Sigma}_X)}{\lambda_{\min}(\bm{P})|y - \bm{x}^\top\bm{\theta}|}
	\end{equation}
	where $\kappa_2(\bm{\Sigma}_X) = \|\bm{\Sigma}_X\|_2\|\bm{\Sigma}_X^{-1}\|_2$ represents the spectral condition number.
\end{theorem}

\begin{proof}
	Using the above bounds, we derive the ratio $R = \frac{\|\text{IF}_{\text{CE}}\|_2}{\|\text{IF}_{\text{MSE}}\|_2}$:
	The upper bound is:
	\begin{align*}
		R &\leq \frac{\frac{\sqrt{2} \|\bm{x}\|_2}{\lambda_{\min}(\bm{\Sigma}_X) \lambda_{\min}(\bm{P})}}{\frac{\|\bm{x}\|_2 \cdot |y - \bm{x}^\top\bm{\theta}|}{\|\bm{\Sigma}_X\|_2}}
		= \frac{\sqrt{2} \|\bm{\Sigma}_X\|_2}{\lambda_{\min}(\bm{\Sigma}_X) \lambda_{\min}(\bm{P}) |y - \bm{x}^\top\bm{\theta}|} \\
		&= \frac{\sqrt{2} \kappa_2(\bm{\Sigma}_X)}{\lambda_{\min}(\bm{P}) |y - \bm{x}^\top\bm{\theta}|} \quad \text{(where $\kappa_2(\bm{\Sigma}_X) = \frac{\|\bm{\Sigma}_X\|_2}{\lambda_{\min}(\bm{\Sigma}_X)}$)}
	\end{align*}
	and the lower bound is:
	\begin{align*}
		R &\geq \frac{\frac{\|\bm{x}\|_2 \cdot \|\bm{\sigma} - \bm{e}_y\|_2}{\|\bm{\Sigma}_X\|_2 \cdot \|\bm{P}\|_2}}{\|\bm{\Sigma}_X^{-1}\|_2 \cdot \|\bm{x}\|_2 \cdot |y - \bm{x}^\top\bm{\theta}|} \\
		&= \frac{\|\bm{\sigma} - \bm{e}_y\|_2}{\|\bm{\Sigma}_X\|_2 \cdot \|\bm{\Sigma}_X^{-1}\|_2 \cdot \|\bm{P}\|_2 \cdot |y - \bm{x}^\top\bm{\theta}|} \\
		&= \frac{\|\bm{\sigma} - \bm{e}_y\|_2}{\kappa_2(\bm{\Sigma}_X) \lambda_{\max}(\bm{P}) |y - \bm{x}^\top\bm{\theta}|}.
	\end{align*}
	
	Thus we obtain the final combined bounds:
	\[
	\frac{\|\bm{\sigma} - \bm{e}_y\|_2}{\kappa_2(\bm{\Sigma}_X) \lambda_{\max}(\bm{P}) |y - \bm{x}^\top\bm{\theta}|} \leq R \leq \frac{\sqrt{2} \kappa_2(\bm{\Sigma}_X)}{\lambda_{\min}(\bm{P}) |y - \bm{x}^\top\bm{\theta}|}.
	\]
\end{proof}
The cross-entropy (CE) loss becomes more stable than mean squared error (MSE) when the ratio of influence functions satisfies $R \leq 1$. This stability condition holds when the data covariance matrix is well-conditioned and model predictions avoid extreme confidence. Mathematically, the critical threshold occurs when the spectral condition number $\kappa_2(\bm{\Sigma}_X)$ is bounded by:
\begin{equation}
	\label{stability}
	\kappa_2(\bm{\Sigma}_X) \leq \frac{\lambda_{\min}(\bm{P})|y - \bm{x}^\top\bm{\theta}|}{\sqrt{2}}.
\end{equation}

From Eq. (\ref{stability}), the stability of cross-entropy (CE) loss is primarily determined by two key mathematical quantities:

First, the condition number $\kappa_2(\bm{\Sigma}_X)$ of the feature covariance matrix characterizes the orthogonality of the input features. When $\kappa_2(\bm{\Sigma}_X) \approx 1$, the input features are approximately orthogonal (if the mean is zero and the variances are normalized, the covariance matrix is a scaled identity matrix).
The gradient descent optimization process is more stable because the curvature (reflected by the eigenvalues of the Hessian matrix) is similar across different directions, allowing a single learning rate to work well for all dimensions.
The sensitivity to input perturbations is reduced because the local geometry of the loss function is more uniform.

Second, the minimum eigenvalue $\lambda_{\min}(\bm{P})$ of the softmax covariance matrix fundamentally characterizes prediction confidence geometry - it vanishes when the model becomes overconfident in any class (If $\sigma(\bm{x}^T \bm{\beta})$ is a one-hot vector for example, $\sigma(\bm{x}^T \bm{\beta}) = \bm{e}_k$, a direct calculation shows that $\bm{P} = 0$, and thus $\lambda_{\min}(\bm{P}) = 0$), indicating degenerate curvature in the loss landscape. This spectral property directly controls influence function stability, with smaller $\lambda_{\min}(\bm{P})$ values amplifying sample influence.

\section{Algorithm Framework}

\subsection{The OCE-TS Architecture}
\begin{figure}[!ht]
	\centering
	\includegraphics[width=0.9\linewidth]{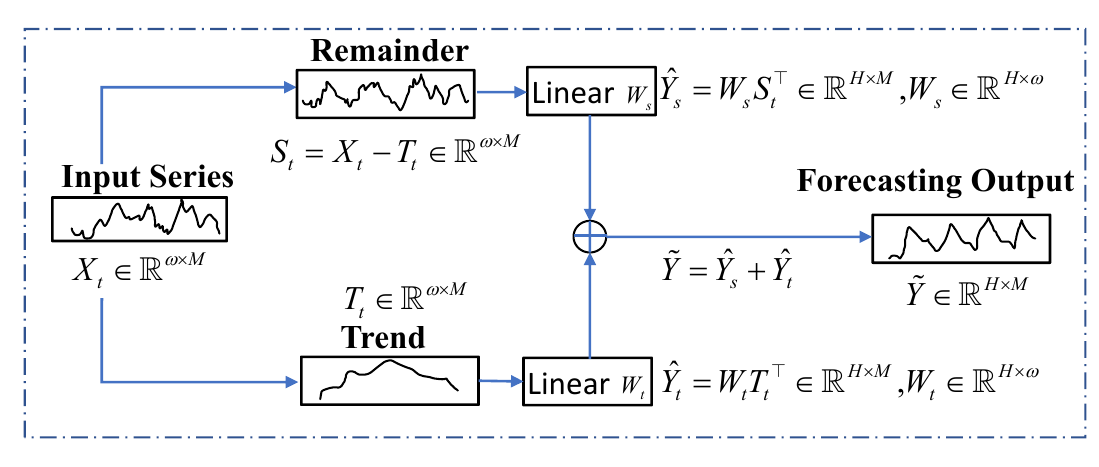}
	\caption{The framework of the DLinear model }
	\label{fig:dlinear_architecture}
\end{figure}

The OCE-TS algorithm is developed by modifying the the Decomposition-Linear (DLinear) Model. As shown in Figure~\ref{fig:dlinear_architecture}, the DLinear \cite{zeng2023transformers} framework employs a decomposition-based approach for time series forecasting by separating the input series into trend and residual components. The procedure of DLinear is outlined as follows:
Given an input lookback window $\boldsymbol{X}_t = \{\boldsymbol{x}_t, \boldsymbol{x}_{t-1}, \dots, \boldsymbol{x}_{t-w+1}\} \in \mathbb{R}^{w \times M}$ consisting of $w$ historical observations with $M$ features, the model first decomposes the series into a trend component $\boldsymbol{T}_t = \text{MA}_k(\boldsymbol{X}_t)$, obtained via moving average smoothing with window size $k$, and a residual component $\boldsymbol{S}_t = \boldsymbol{X}_t - \boldsymbol{T}_t$. The trend is then processed through weights $\boldsymbol{W}_t \in \mathbb{R}^{H \times w}$ to capture global patterns, yielding predictions $\hat{\boldsymbol{Y}}_t = \boldsymbol{W}_t \boldsymbol{T}_t^\top$, while the residuals are transformed via weights $\boldsymbol{W}_s \in \mathbb{R}^{H \times w}$ to model local fluctuations, producing $\hat{\boldsymbol{Y}}_s = \boldsymbol{W}_s \boldsymbol{S}_t^\top$.

The final $H$-step prediction $\tilde{\boldsymbol{Y}} \in \mathbb{R}^{H \times M}$ combines both components:
\begin{equation}
	\tilde{\boldsymbol{Y}} = \underbrace{\boldsymbol{W}_t \cdot \text{MA}_k(\boldsymbol{X}_t^\top)}_{\text{Trend}} + \underbrace{\boldsymbol{W}_s \cdot (\boldsymbol{X}_t^\top - \text{MA}_k(\boldsymbol{X}_t^\top))}_{\text{Residual}}.
\end{equation}

This dual-component design enables DLinear to simultaneously capture long-term trends and short-term variations while maintaining computational efficiency through linear operations.

At the output layer of the DLinear model, we introduce a softmax classification head to convert the predicted values $\tilde{\mathbf{Y}}$ into a probability distribution $\bm{q} \in \mathbb{R}^K$ over $K$ ordinal categories. For ground truth values $\mathbf{Y}$, we generate target probability distributions $\bm{p} \in \mathbb{R}^K$ using a Truncated Gaussian Distribution. The model parameters (including trend weights $\mathbf{W}_t$ and residual weights $\mathbf{W}_s$) are trained end-to-end by optimizing the OCE loss function, which minimizes the divergence between predicted probabilities $\bm{q}$ and target distributions $\bm{p}$.

\subsection{The OCE-TS Algorithm}
The framework of our proposed OCE-TS algorithm is presented in Algorithm~\ref{alg:oce-ts}.
\begin{algorithm}[!h]
	\caption{OCE-TS Framework}
	\label{alg:oce_ts_batch}
	\begin{algorithmic}[1]
		\Require $\mathcal{D} = \{\bm{X}_{t_i}, Y_{t_i+j}\}$: Training dataset
		\Require $K$: Number of bins
		\Require $[a,b]$: Target value range
		\Require $\sigma$: Standard deviation of truncated Gaussian
		\Require $w$: Input window size
		\Require $H$: Prediction horizon
		\Require $\eta$: Learning rate
		\Require $B$: Batch size
		
		\Ensure Trained model parameters $\theta^*$
		
		\State \textbf{1. Target-to-Probability Transformation}
		\For{each observation $Y_{t+j} \in \mathcal{D}$}
		\State Partition $[a,b]$ into $K$ bins $\mathcal{B}_k = [\ell_k, u_k)$ where $\ell_k = a + (k-1)\Delta$, $u_k = a + k\Delta$, $\Delta = \frac{b-a}{K}$
		\State Compute $p_k^{(j)} = \frac{1}{2Z}\left[\mathrm{erf}\left(\frac{u_k - Y_{t+j}}{\sigma\sqrt{2}}\right) - \mathrm{erf}\left(\frac{\ell_k - Y_{t+j}}{\sigma\sqrt{2}}\right)\right]$
		\State Store $\bm{p}^{(j)} \gets [p_1^{(j)},\dots,p_K^{(j)}]^\top$
		\EndFor
		
		\State \textbf{2. Batch Training}
		\Repeat
		\State Sample batch $\{(\bm{X}_{t_i}^{(b)}, \bm{Y}_{t_i+j}^{(b)})\}_{b=1}^B$ from $\mathcal{D}$
		\For{$b \gets 1$ \textbf{to} $B$}
		\State $\bm{T}_{t_i}^{(b)} \gets \frac{1}{l}\sum_{k=1}^l \bm{X}_{t_i-k}^{(b)}$ \Comment{MA with window size $l$}
		\State $\bm{S}_{t_i}^{(b)} \gets \bm{X}_{t_i}^{(b)} - \bm{T}_{t_i}^{(b)}$ \Comment{Residual}
		\State $\tilde{\bm{Y}}^{(b)} \gets \bm{W}_t(\bm{T}_{t_i}^{(b)})^\top + \bm{W}_s(\bm{S}_{t_i}^{(b)})^\top$ \Comment{Projection}
		\For{$j \gets 1$ \textbf{to} $H$}
		\State $\bm{q}^{(b,j)} \gets \mathrm{softmax}(\bm{W}_o \bm{\tilde{Y}}_{t_i+j}^{(b)} + \bm{b}_o)$
		\State $\mathcal{L}^{(b,j)} \gets -\sum_{k=1}^K p_k^{(b,j)} \log q_k^{(b,j)}$
		\EndFor
		\EndFor
		\State $\theta \gets \theta - \eta \nabla_\theta\left(\frac{1}{BH}\sum_{b=1}^B\sum_{j=1}^H \mathcal{L}^{(b,j)}\right)$
		\Until{convergence}
		
		\State \textbf{3. Probability-to-Value Reconstruction}
		\For{each test window $\bm{X}_t$}
		\For{$j \gets 1$ \textbf{to} $H$}
		\State Compute $\bm{\hat{Y}}_{t+j} \gets \sum_{k=1}^K v_k q_k^{(j)}$
		\EndFor
		\EndFor
	\end{algorithmic}
	\label{alg:oce-ts}
\end{algorithm}

\section{More Experimental Results}
\label{app:experimental results}
\subsection{ Dataset Descriptions}
\label{app:dataset}
Table~\ref{tab:dataset_stats} presents the detailed characteristics of the experimental datasets used in this study.
\begin{table}[!h]
	\centering
	\setlength{\tabcolsep}{6pt}
	\renewcommand{\arraystretch}{0.8}
	\begin{tabular}{l c c}
		\toprule
		\textbf{Dataset} & \textbf{Feature} & \textbf{Timesteps} \\
		\midrule
		ETTh1 & 7 & 17420 \\
		ETTh2 & 7 & 17420 \\
		ETTm1 & 7 & 69680 \\
		ETTm2 & 7 & 69680 \\
		Exchange & 8 & 7588 \\
		ILI & 7 & 966 \\
		Weather & 21 & 52696 \\
		%		Electricity & 321 & 26304 \\
		%       Traffic & 862 & 17544 \\
		\bottomrule
	\end{tabular}
	\caption{Dataset Statistics}
	\label{tab:dataset_stats}
\end{table}

\subsection{The Definition of Mean Absolute Error}
\label{app:MAE}
Mean Absolute Error (MAE) is a commonly used evaluation metric in time series forecasting, mathematically defined as:
\begin{equation}
	\text{MAE} = \frac{1}{H}\sum_{j=1}^H \|\boldsymbol{y}_{t+j} - \hat{\boldsymbol{y}}_{t+j}\|_1,
\end{equation}
where $\boldsymbol{y}_{t+j}$ and $\hat{\boldsymbol{y}}_{t+j}$ represent the ground truth and predicted vectors at time step $t+j$, respectively.

\subsection{Data Normalization}
\label{app:Model}

The model processes input time series normalized to the range \([-1, 1]\) or \([0, 1]\) during training and validation (denormalized during testing), with the support \([a, b]\) of the truncated Gaussian distribution matching the normalization range.

The normalization range is determined based on the data distribution: if all values are positive, normalization is performed to \([0, 1]\); otherwise, normalization is performed to \([-1, 1]\). In this experiment, the ETT series datasets are normalized to \([0, 1]\), while other datasets are normalized to \([-1, 1]\).

Let $\bm{x}$ be a time series sequence, with its minimum and maximum values denoted as $\bm{x}_{\min}$ and $\bm{x}_{\max}$ respectively. For normalization to the interval $[-1, 1]$, the minimum and maximum values are computed along the temporal dimension for each sample and feature. The normalization formula is:
\begin{equation}
	\label{eq:norm_neg1}
	\bm{x}_{\text{norm}} = 2 \cdot \frac{\bm{x} - \bm{x}_{\min}}{\bm{x}_{\max} - \bm{x}_{\min} + \varepsilon} - 1,
\end{equation}
where \(\epsilon\) is a small constant to prevent division by zero. The corresponding denormalization formula is:
\begin{equation}
	\label{eq:denorm_neg1}
	\bm{x} = \frac{\bm{x}_{\text{norm}} + 1}{2} \cdot \left( \bm{x}_{\max} - \bm{x}_{\min} + \varepsilon \right) + \bm{x}_{\min}.
\end{equation}

For normalization to the interval $[0, 1]$, the same minimum and maximum values $\bm{x}_{\min}$ and $\bm{x}_{\max}$ are computed along the temporal dimension for each sample and feature. The normalization is performed as:
\begin{equation}
	\label{eq:norm_01}
	\bm{x}_{\text{norm}} = \frac{\bm{x} - \bm{x}_{\min}}{\bm{x}_{\max} - \bm{x}_{\min} + \varepsilon}.
\end{equation}

The denormalization formula is:
\begin{equation}
	\label{eq:denorm_01}
	\bm{x} = \bm{x}_{\text{norm}} \times (\bm{x}_{\max} - \bm{x}_{\min} + \varepsilon) + \bm{x}_{\min}.
\end{equation}

The normalization operation for $\bm{X}_{t_i}^{(b)}$ and $\bm{Y}_{t_i+j}^{(b)}$ directly use the maximum and minimum values of the input sequence $\bm{X}_{t_i}^{(b)}$ and $\bm{Y}_{t_i+j}^{(b)}$, respectively. For denormalization, the predicted values $\bm{\hat{Y}}_{t+j}$ strictly employ the maximum and minimum values of the input sequence $\bm{X}_{t_i}^{(b)}$, thereby ensuring numerical stability during the training process.

\subsection{Distributions}
\label{app:Distribution}
This section introduces two alternative probability distributions, which are incorporated into the proposed framework to replace the truncated Gaussian distribution for validation and comparison purposes. We present the formal definitions of both distributions used in our analysis and describe the methodology for generating Figure~\ref{fig:dist-compare}.

\subsubsection{Student's t-distribution}

The Student's $t$-distribution with $\nu$ degrees of freedom has probability density function:
\begin{equation}
	f(y|\nu) = \frac{\Gamma\left(\frac{\nu+1}{2}\right)}{\sqrt{\nu \pi} \, \Gamma\left(\frac{\nu}{2}\right)} \left(1 + \frac{y^2}{\nu}\right)^{-\frac{\nu+1}{2}},
\end{equation}
where $\Gamma(\cdot)$ is the gamma function, $\nu > 0$ represents the degrees of freedom, and the distribution converges to the standard normal distribution as $\nu \to \infty$.

For the observation \( Y_{t+j} = y_c \), we model the predictive distribution using a truncated \( t \)-distribution with \(\nu=5\) degrees of freedom, centered on the interval \([a,b]\):
\begin{equation}\label{eq:T_pdf}
	p(y) =
	\begin{cases}
		\frac{1}{Z\sigma}f\left(\frac{y-y_c}{\sigma}\bigg|5\right), & x \in [a,b] \\
		0, & \text{otherwise}
	\end{cases},
\end{equation}
The choice of \(\nu=5\) provides moderate tail heaviness, allowing the model to better handle outliers while retaining a near-Gaussian shape.
Here, \( f(\cdot|\nu) \) denotes the PDF of a standard \( t \)-distribution with \(\nu=5\) degrees of freedom, and \(\sigma\) is the scale parameter.

The normalization constant \( Z \) ensures that the total probability integrates to one over the interval \([a,b]\). It is computed by the difference of cumulative distribution function (CDF) values:
\begin{equation}\label{eq:T_Z}
	Z = F_5\left(\frac{b-y_c}{\sigma}\right) - F_5\left(\frac{a-y_c}{\sigma}\right),
\end{equation}
where \( F_5(\cdot) \) is the CDF of the standard \( t \)-distribution with 5 degrees of freedom.

For discretized bins \([\ell_k, u_k]\) within \([a,b]\), the probability mass assigned to the \(k\)-th bin at horizon \(j\) is given by:
\begin{equation}\label{eq:T_pk}
	p_k^{(j)} = \frac{1}{Z} \left[ F_5\left(\frac{u_k - y_c}{\sigma}\right) - F_5\left(\frac{\ell_k - y_c}{\sigma}\right) \right].
\end{equation}

In evaluating the impact of different distributions on experimental performance, when the Gaussian distribution is replaced with the Student's $t$-distribution, the equations in the main text, namely Eqs.~(\ref{eq:G_pdf}), (\ref{eq:G_Z}), and (\ref{eq:bin_prob_calc}), are respectively replaced by Equations~(\ref{eq:T_pdf}), (\ref{eq:T_Z}), and (\ref{eq:T_pk}).

\subsubsection{Laplace Distribution}

The Laplace distribution (double exponential) with location $\mu$ and scale $\lambda > 0$ has the probability density function:
\begin{equation}
	f(y|\mu,\lambda) = \frac{1}{2\lambda}\exp\left(-\frac{|y-\mu|}{\lambda}\right),
\end{equation}
where $\mu \in \mathbb{R}$ is the location parameter and $\lambda$ controls the spread. The distribution is symmetric about $\mu$ and has heavier tails than the normal distribution. The scale parameter $\lambda$ relates to the standard deviation $\sigma$ through:
\begin{equation}
	\sigma = \sqrt{2}\lambda \quad \text{or equivalently} \quad \lambda = \frac{\sigma}{\sqrt{2}}.
\end{equation}

To ensure a fair comparison across distributional assumptions, we align the variance of the Laplace distribution with that of the Gaussian. Specifically, we set
\begin{equation}
	\lambda = \frac{\sigma}{\sqrt{2}} \approx 0.00707,
\end{equation}
given \(\sigma = 0.01\), so that both distributions share the same scale and yield consistent predictive behavior.

For the observation \( Y_{t+j} = y_c \), we model the predictive distribution using a truncated Laplace distribution with scale parameter \(\lambda\), centered on the interval \([a,b]\):
\begin{equation}\label{eq:L_pdf}
	p(y) =
	\begin{cases}
		\frac{1}{Z\lambda}f\left(\frac{y - y_c}{\lambda} \bigg| \mu=0 \right), &  \in [a,b] \\
		0, & \text{otherwise}
	\end{cases},
\end{equation}
Here, \( f(\cdot|\mu) \) denotes the probability density function (PDF) of a Laplace distribution with location parameter \(\mu=0\).

The normalization constant \( Z \) ensures the total probability integrates to one over \([a,b]\), computed by the difference of the Laplace cumulative distribution function (CDF) values:
\begin{equation}\label{eq:L_Z}
	Z = F_{\text{Laplace}}\left(\frac{b - y_c}{\lambda}\right) - F_{\text{Laplace}}\left(\frac{a - y_c}{\lambda}\right),
\end{equation}
where \( F_{\text{Laplace}}(\cdot) \) is the CDF of the standard Laplace distribution.

For discretized bins \([\ell_k, u_k]\) within \([a,b]\), the probability mass assigned to the \(k\)-th bin at horizon \(j\) is given by:
\begin{equation}\label{eq:L_pk}
	\small
	p_k^{(j)} = \frac{1}{Z} \left[ F_{\text{Laplace}}\left(\frac{u_k - y_c}{\lambda}\right) - F_{\text{Laplace}}\left(\frac{\ell_k - y_c}{\lambda}\right) \right].
\end{equation}

To evaluate the impact of different distributions on experimental performance, when the Gaussian distribution is replaced by the Laplace distribution, the equations in the main text, namely Equations~(\ref{eq:G_pdf}), (\ref{eq:G_Z}), and (\ref{eq:bin_prob_calc}), are respectively replaced with Equations~(\ref{eq:L_pdf}), (\ref{eq:L_Z}), and (\ref{eq:L_pk}).

\begin{figure*}[h]
	\centering
	
	% 第一行，4张图
	\begin{subfigure}[b]{0.23\linewidth}
		\centering
		\includegraphics[width=\linewidth]{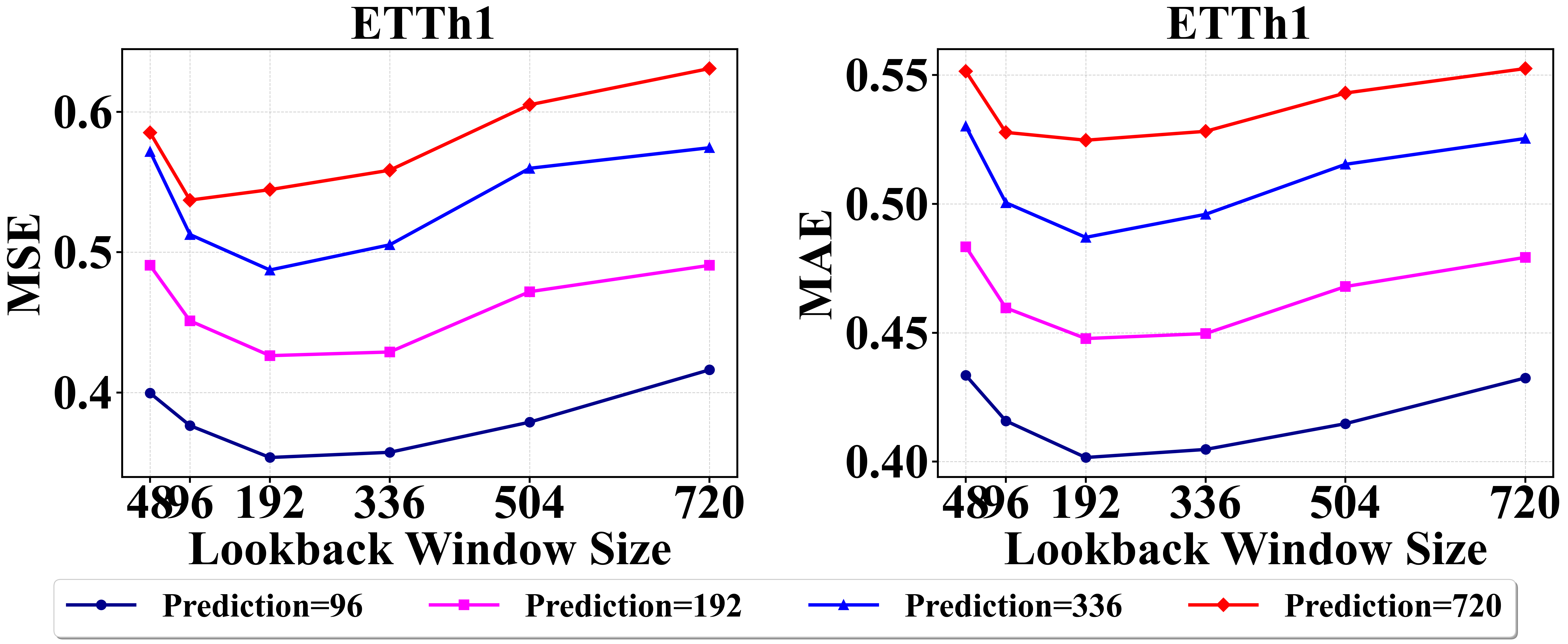}
		\caption{ETTh1}
	\end{subfigure}
	\hfill
	\begin{subfigure}[b]{0.23\linewidth}
		\centering
		\includegraphics[width=\linewidth]{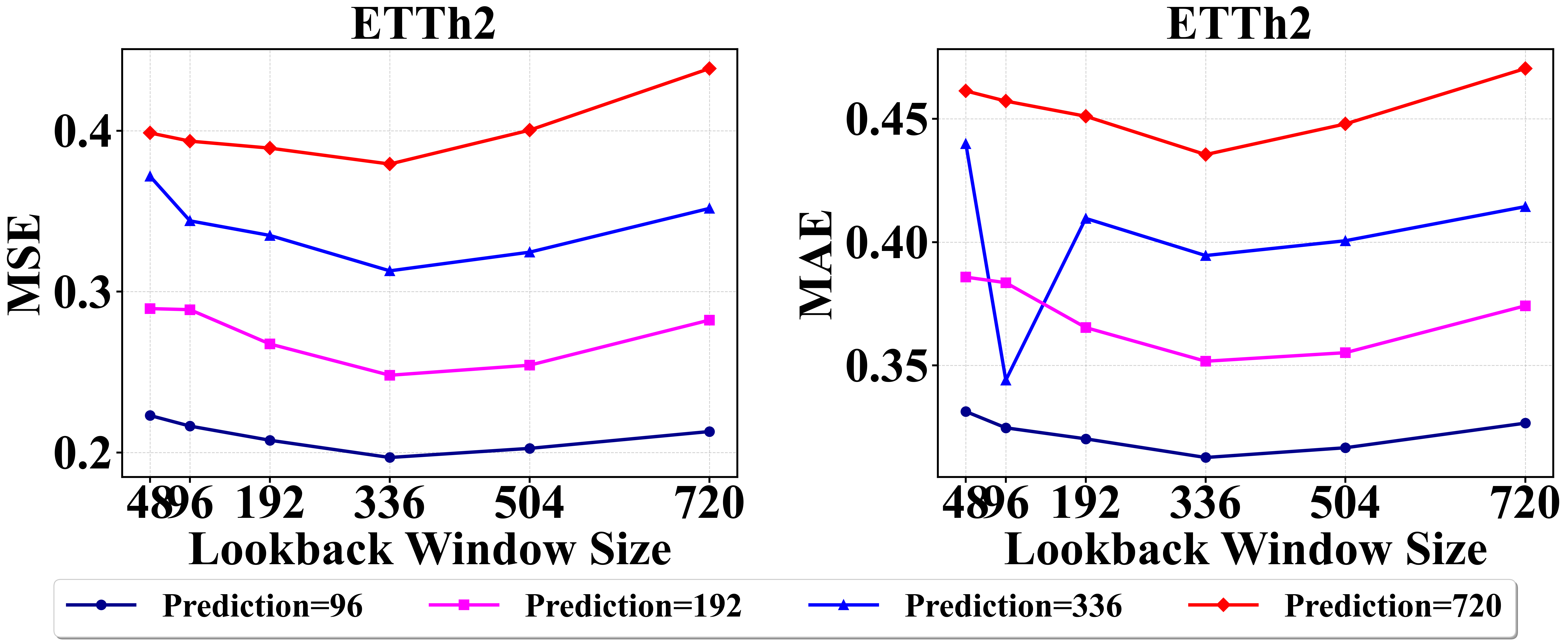}
		\caption{ETTh2}
	\end{subfigure}
	\hfill
	\begin{subfigure}[b]{0.23\linewidth}
		\centering
		\includegraphics[width=\linewidth]{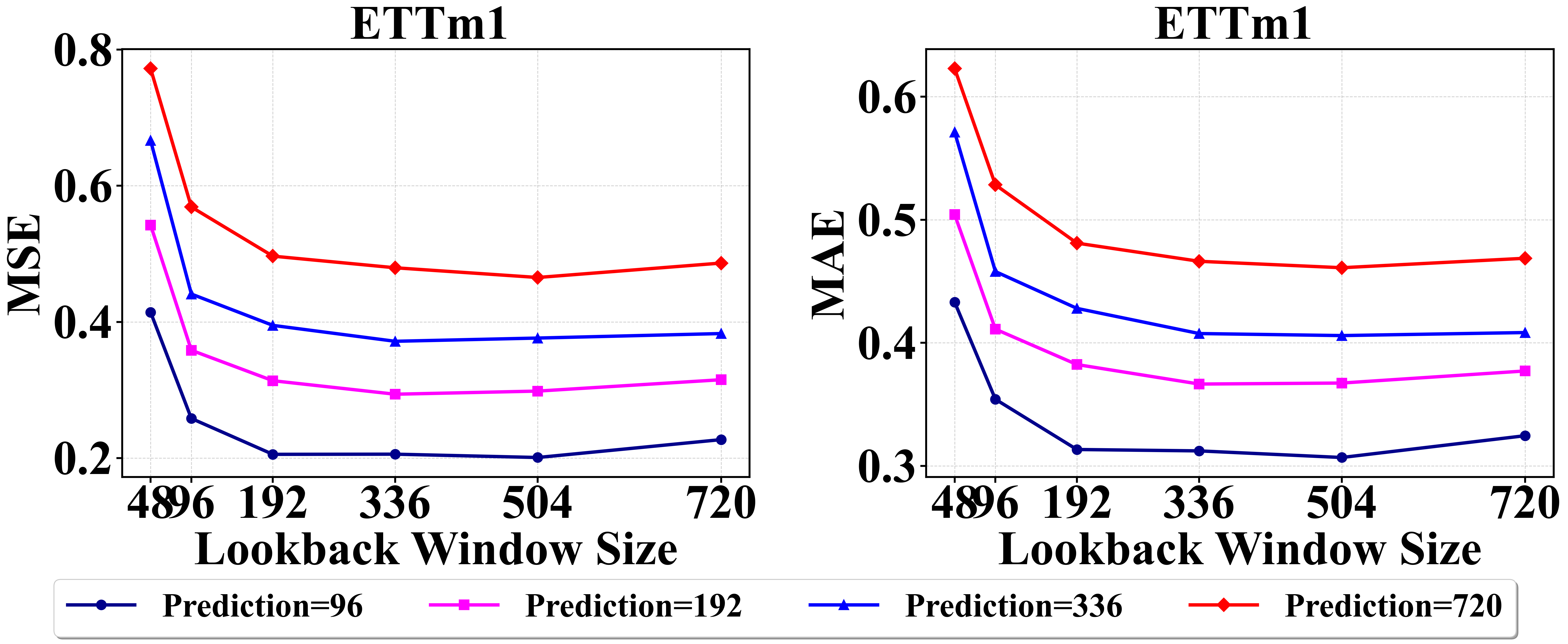}
		\caption{ETTm1}
	\end{subfigure}
	\hfill
	\begin{subfigure}[b]{0.23\linewidth}
		\centering
		\includegraphics[width=\linewidth]{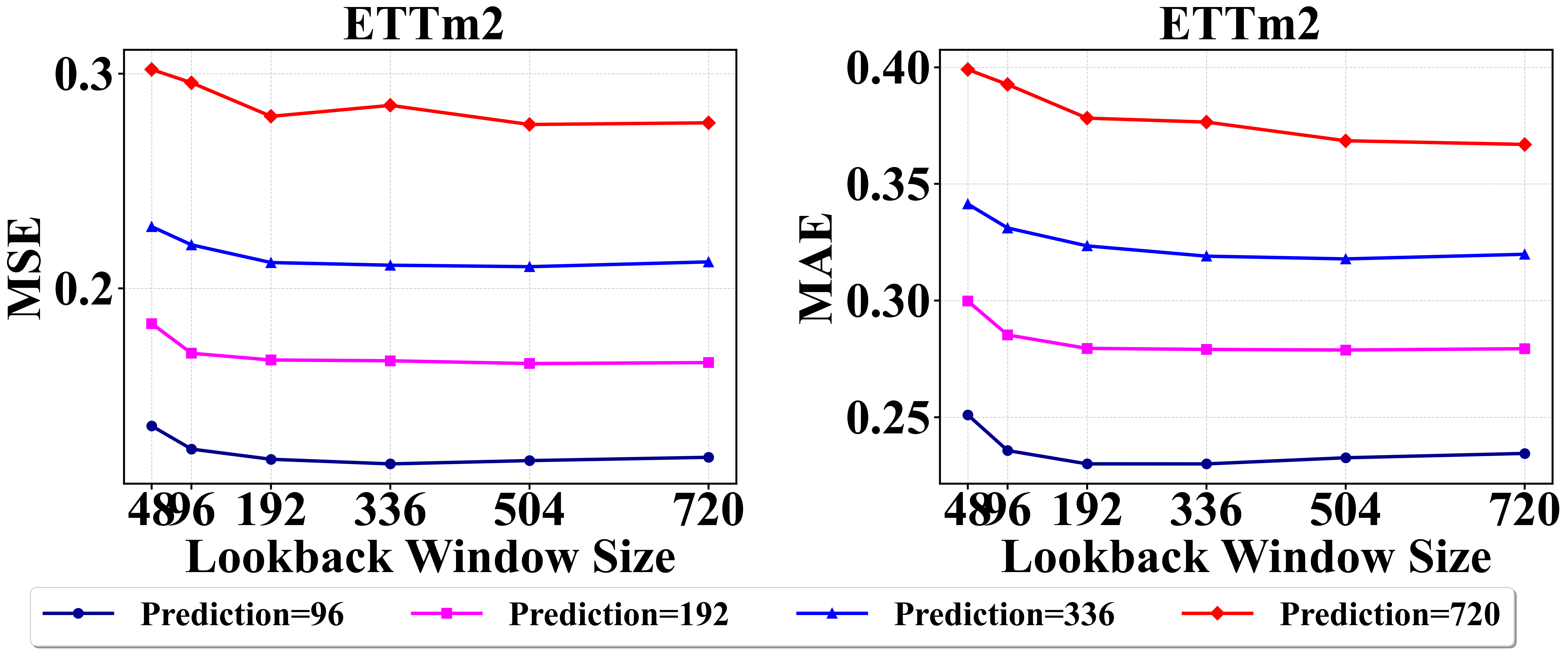}
		\caption{ETTm2}
	\end{subfigure}
	
	\vspace{0.8em}
	
	% 第二行，3张图居中
	\hspace*{\fill}
	\begin{subfigure}[b]{0.3\linewidth}
		\centering
		\includegraphics[width=\linewidth]{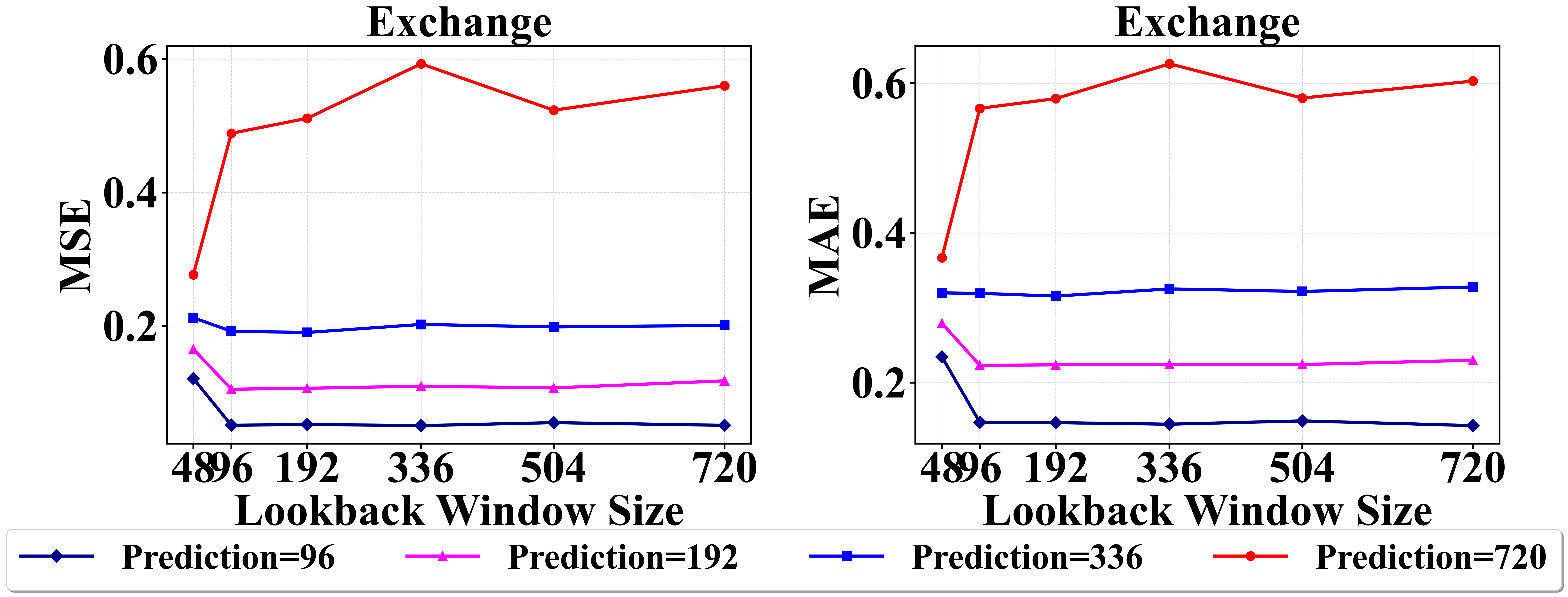}
		\caption{Exchange}
	\end{subfigure}
	\hfill
	\begin{subfigure}[b]{0.3\linewidth}
		\centering
		\includegraphics[width=\linewidth]{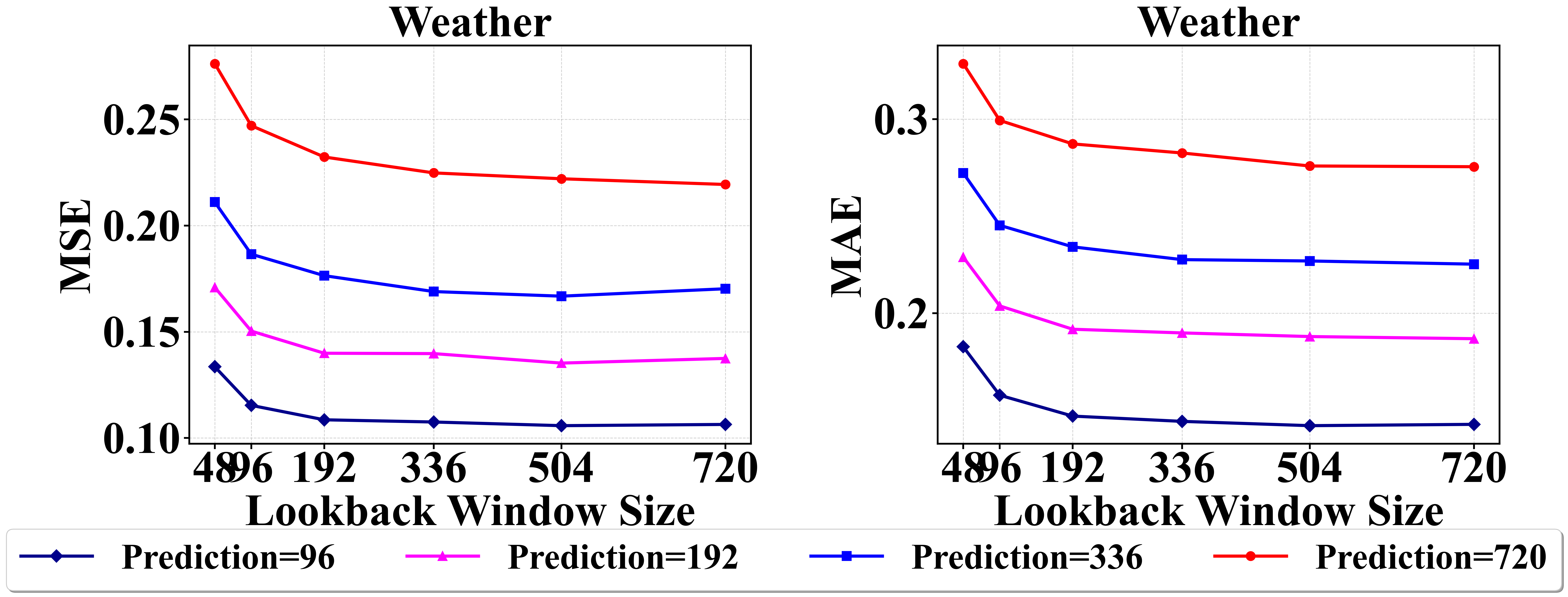}
		\caption{Weather}
	\end{subfigure}
	\hfill
	\begin{subfigure}[b]{0.3\linewidth}
		\centering
		\includegraphics[width=\linewidth]{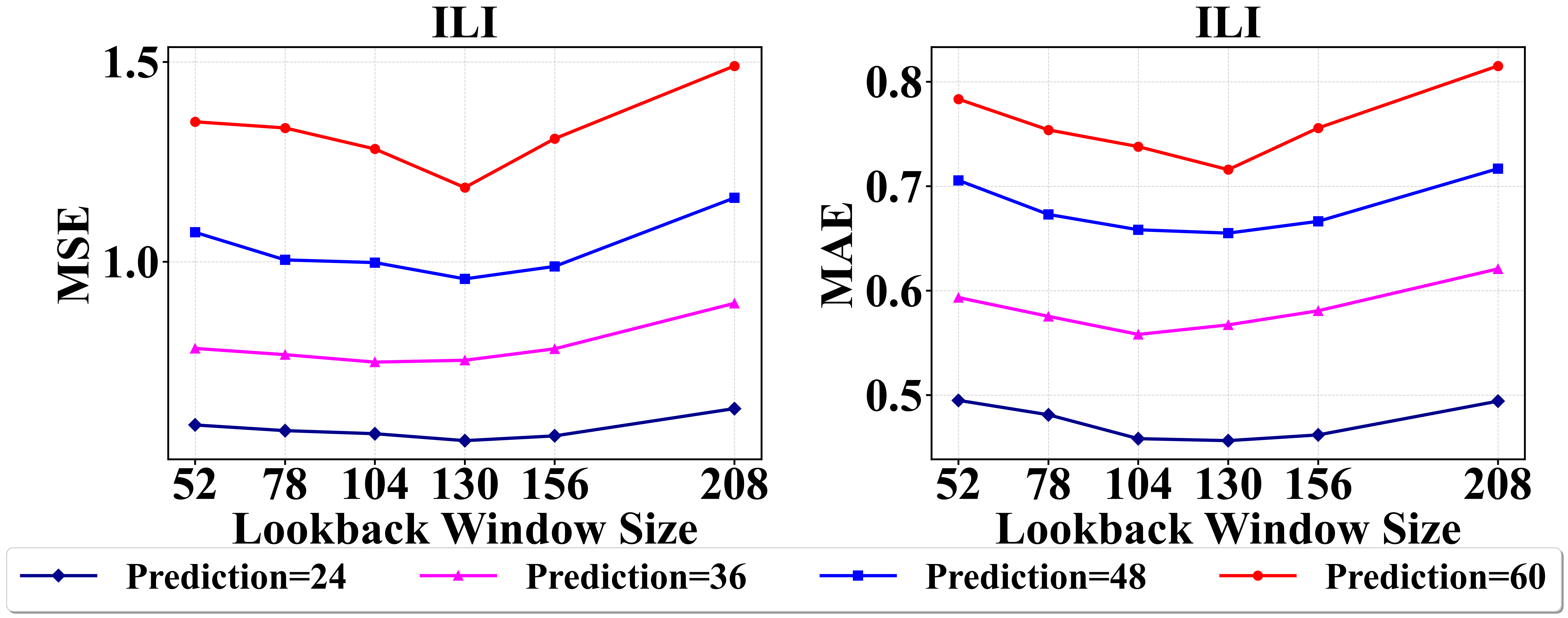}
		\caption{ILI}
	\end{subfigure}
	\hspace*{\fill}
	
	\caption{MSE and MAE results on seven datasets: ETTh1, ETTh2, ETTm1, ETTm2, Exchange, Weather, and ILI.}
	\label{fig:lookback-all}
\end{figure*}

\begin{figure*}[!ht]
	\centering
	% 第一行（4个子图）
	\begin{subfigure}[t]{0.23\textwidth}
		\centering
		\includegraphics[width=\linewidth]{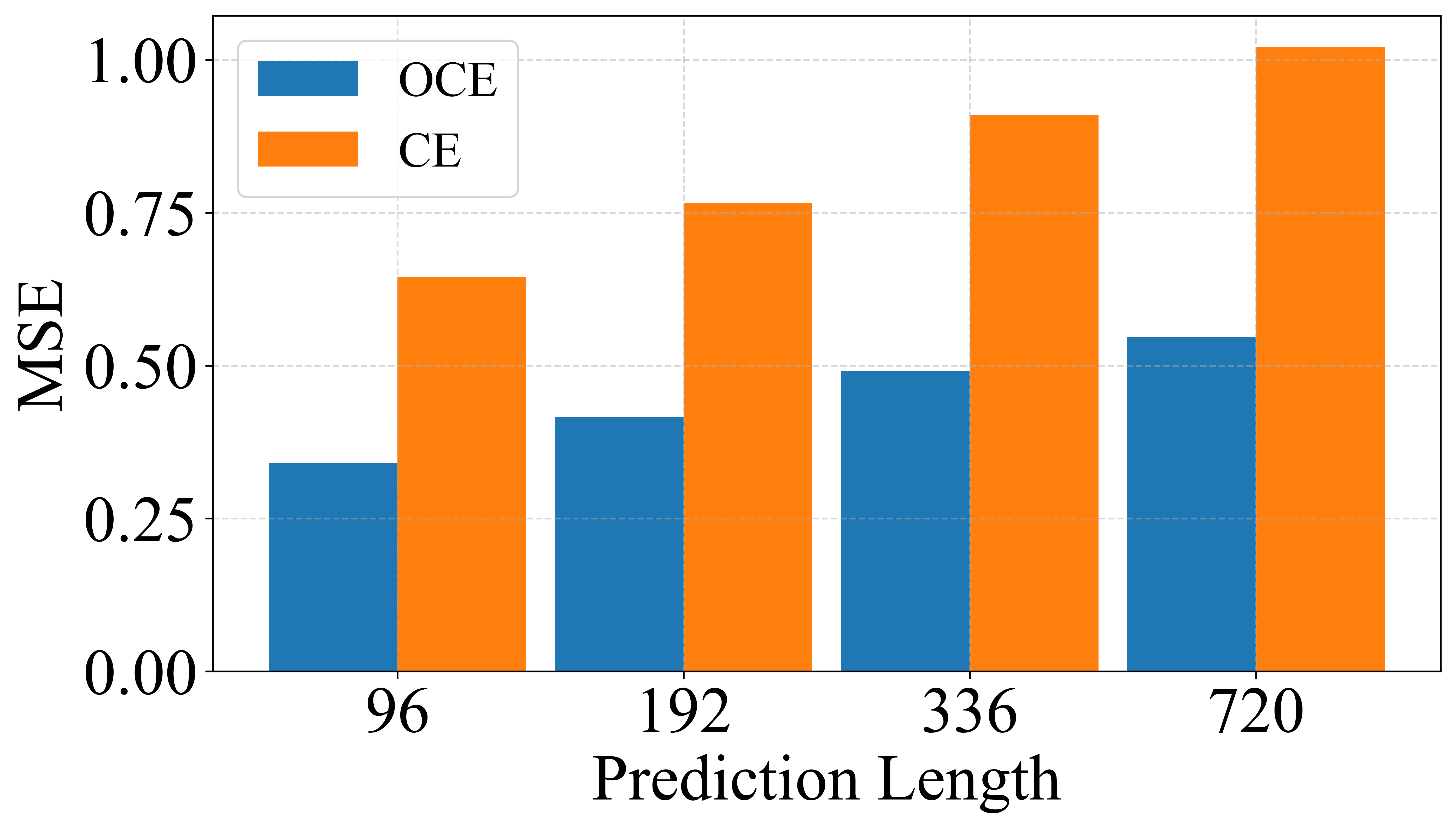}
		\caption{MSE (ETTh1)}
		\label{fig:etth1_mse}
	\end{subfigure}
	\hfill
	\begin{subfigure}[t]{0.23\textwidth}
		\centering
		\includegraphics[width=\linewidth]{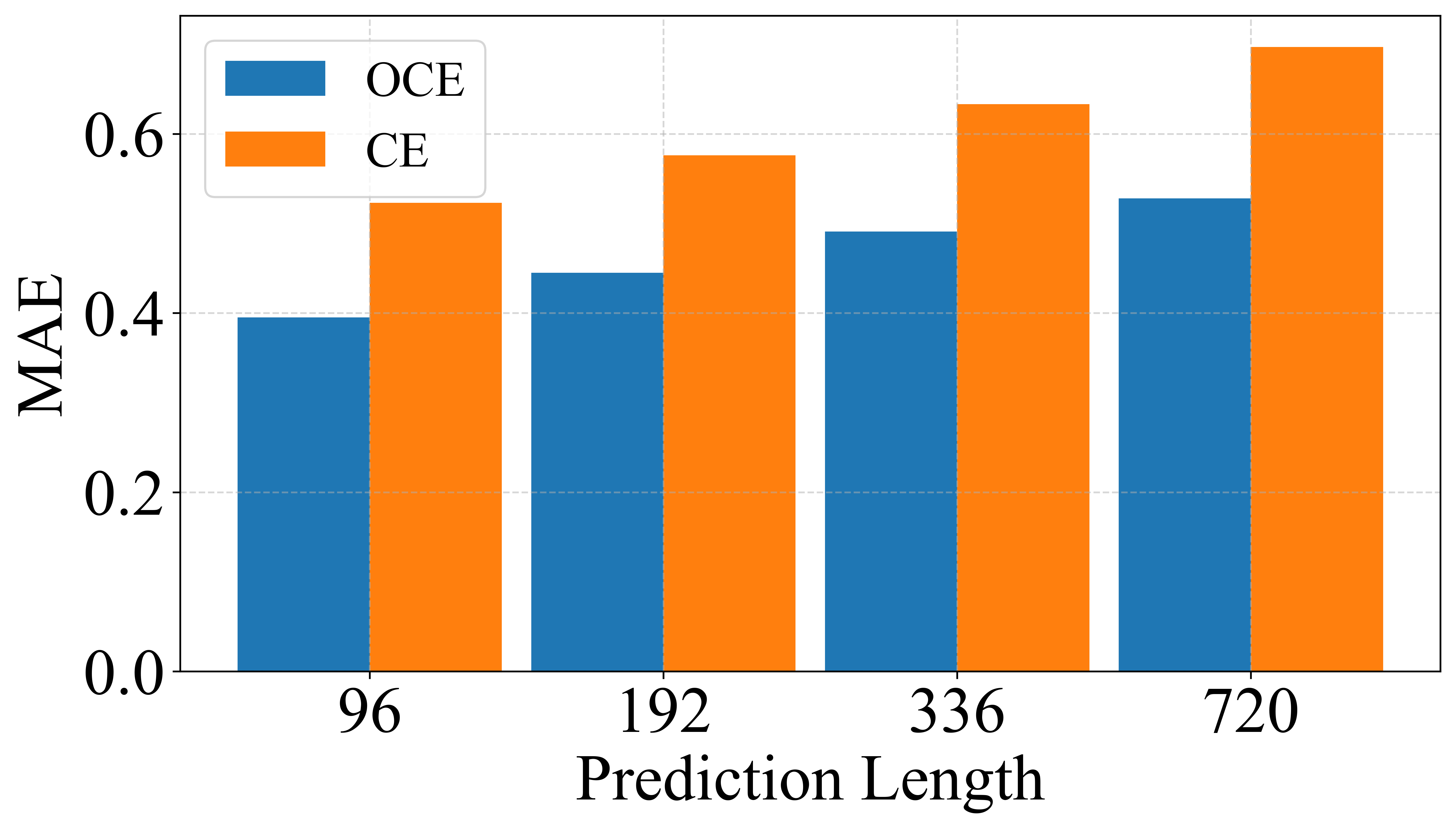}
		\caption{MAE (ETTh1)}
		\label{fig:etth1_mae}
	\end{subfigure}
	\hfill
	\begin{subfigure}[t]{0.23\textwidth}
		\centering
		\includegraphics[width=\linewidth]{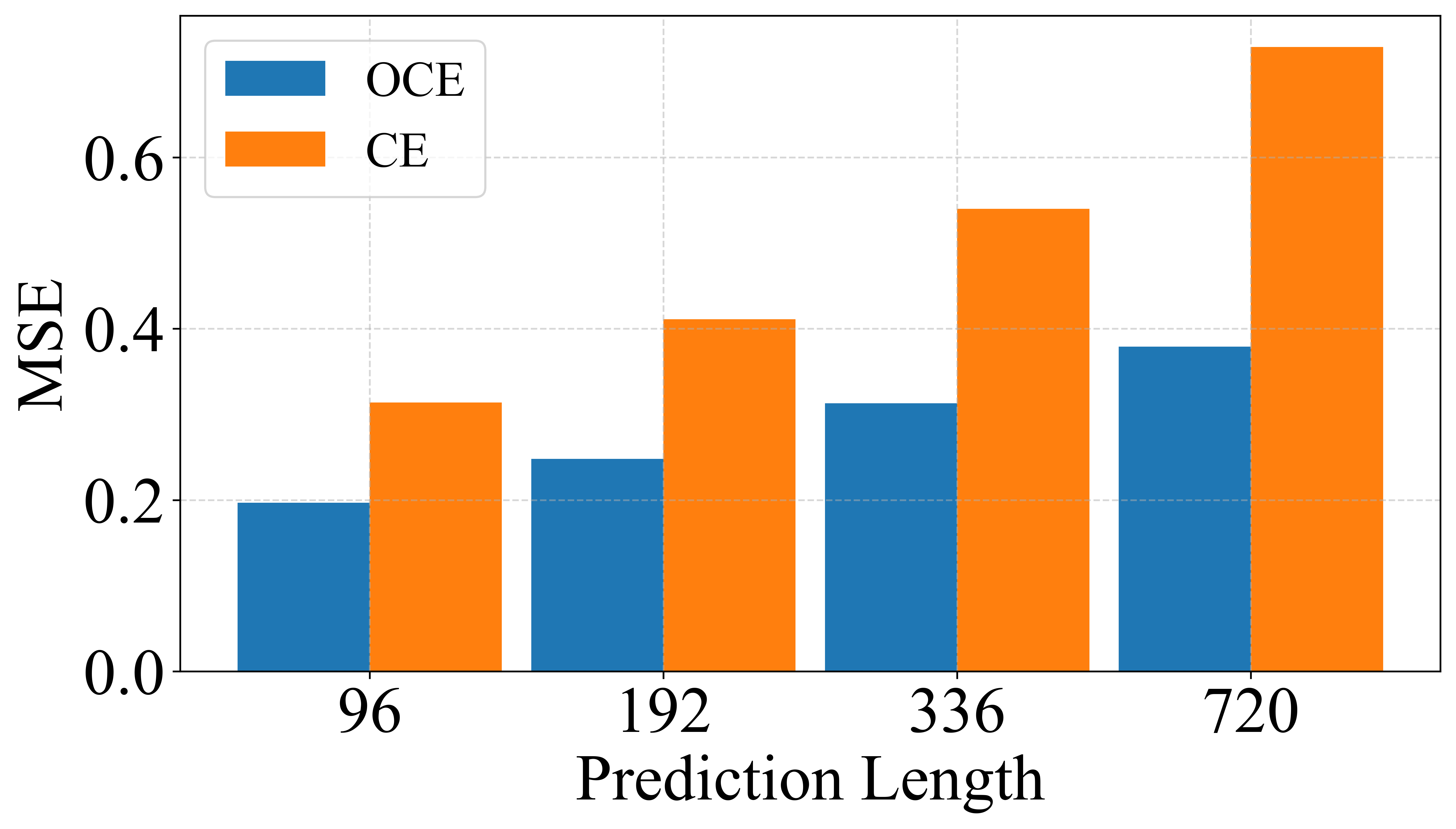}
		\caption{MSE (ETTh2)}
		\label{fig:etth2_mse}
	\end{subfigure}
	\hfill
	\begin{subfigure}[t]{0.23\textwidth}
		\centering
		\includegraphics[width=\linewidth]{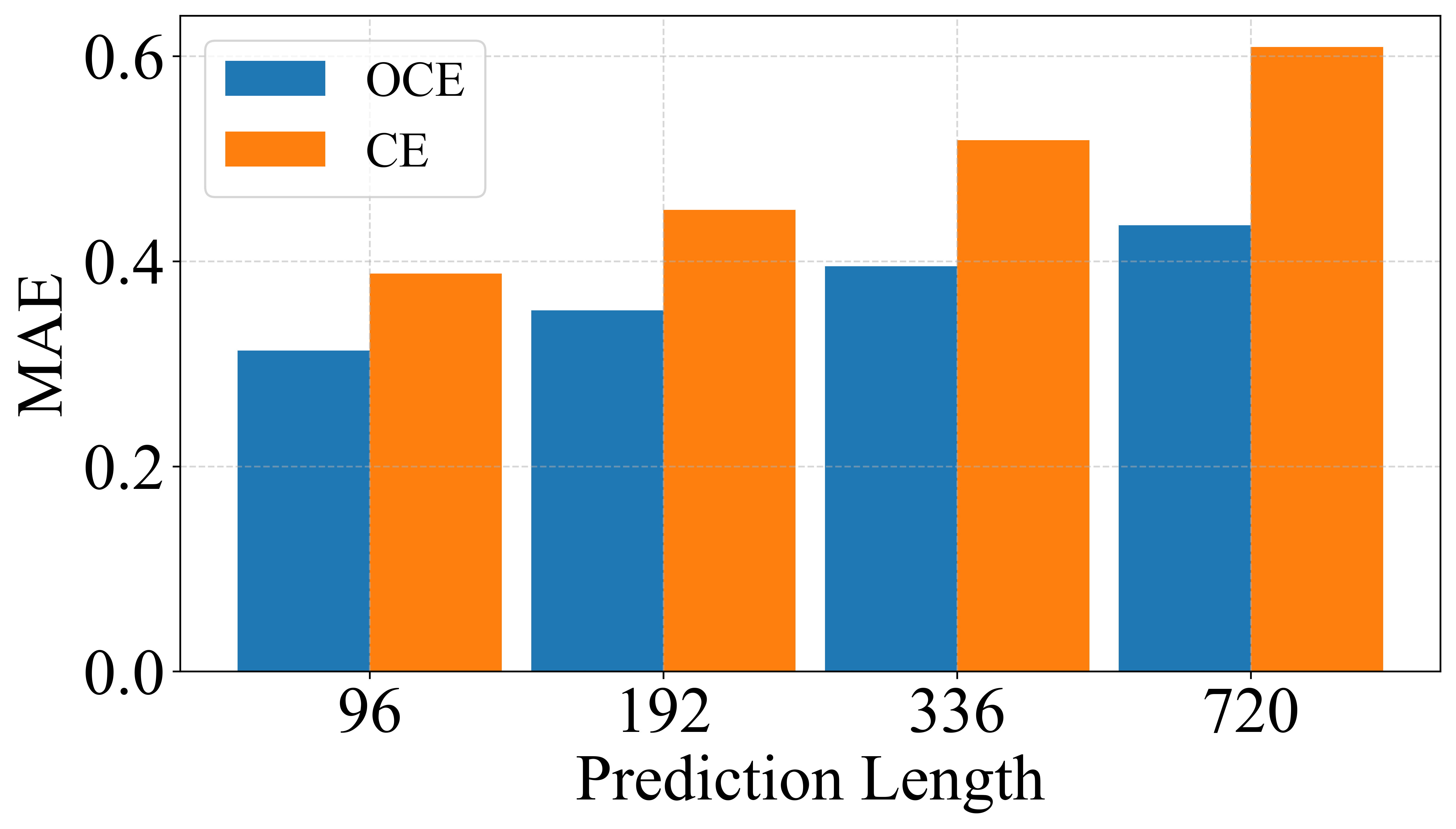}
		\caption{MAE (ETTh2)}
		\label{fig:etth2_mae}
	\end{subfigure}
	
	\vspace{0.5em}
	
	% 第二行（4个子图）
	\begin{subfigure}[t]{0.23\textwidth}
		\centering
		\includegraphics[width=\linewidth]{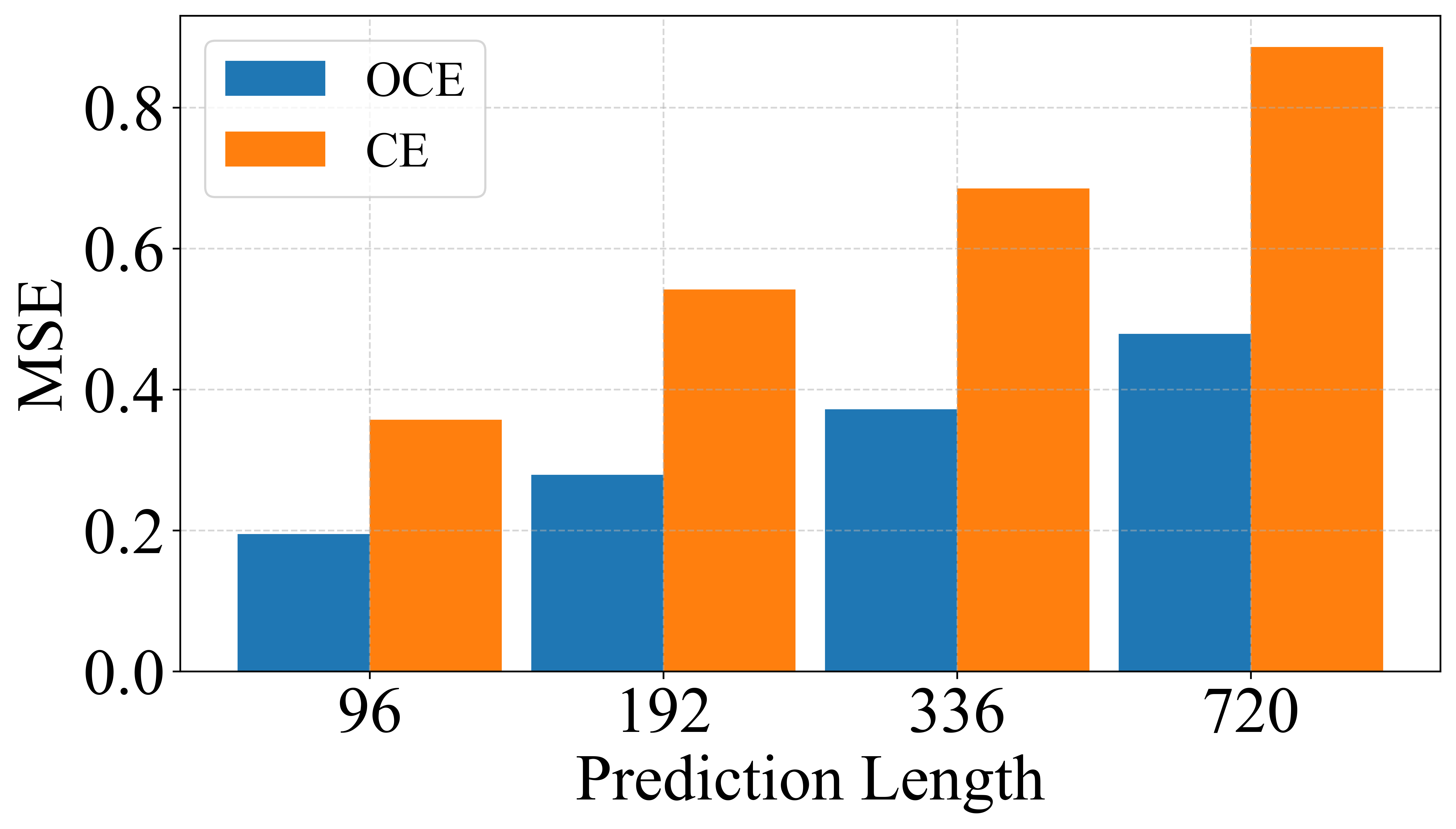}
		\caption{MSE (ETTm1)}
		\label{fig:ettm1_mse}
	\end{subfigure}
	\hfill
	\begin{subfigure}[t]{0.23\textwidth}
		\centering
		\includegraphics[width=\linewidth]{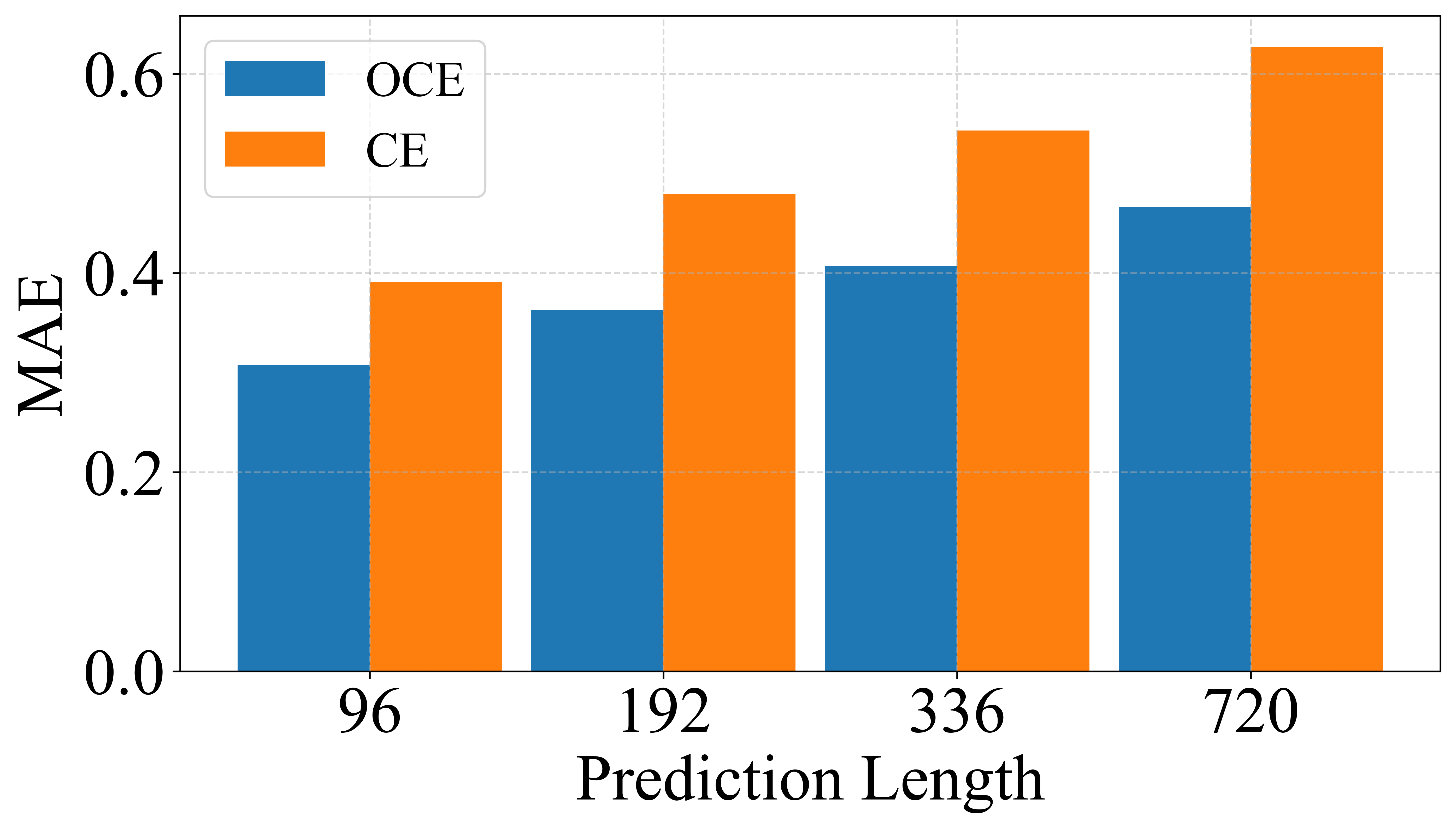}
		\caption{MAE (ETTm1)}
		\label{fig:ettm1_mae}
	\end{subfigure}
	\hfill
	\begin{subfigure}[t]{0.23\textwidth}
		\centering
		\includegraphics[width=\linewidth]{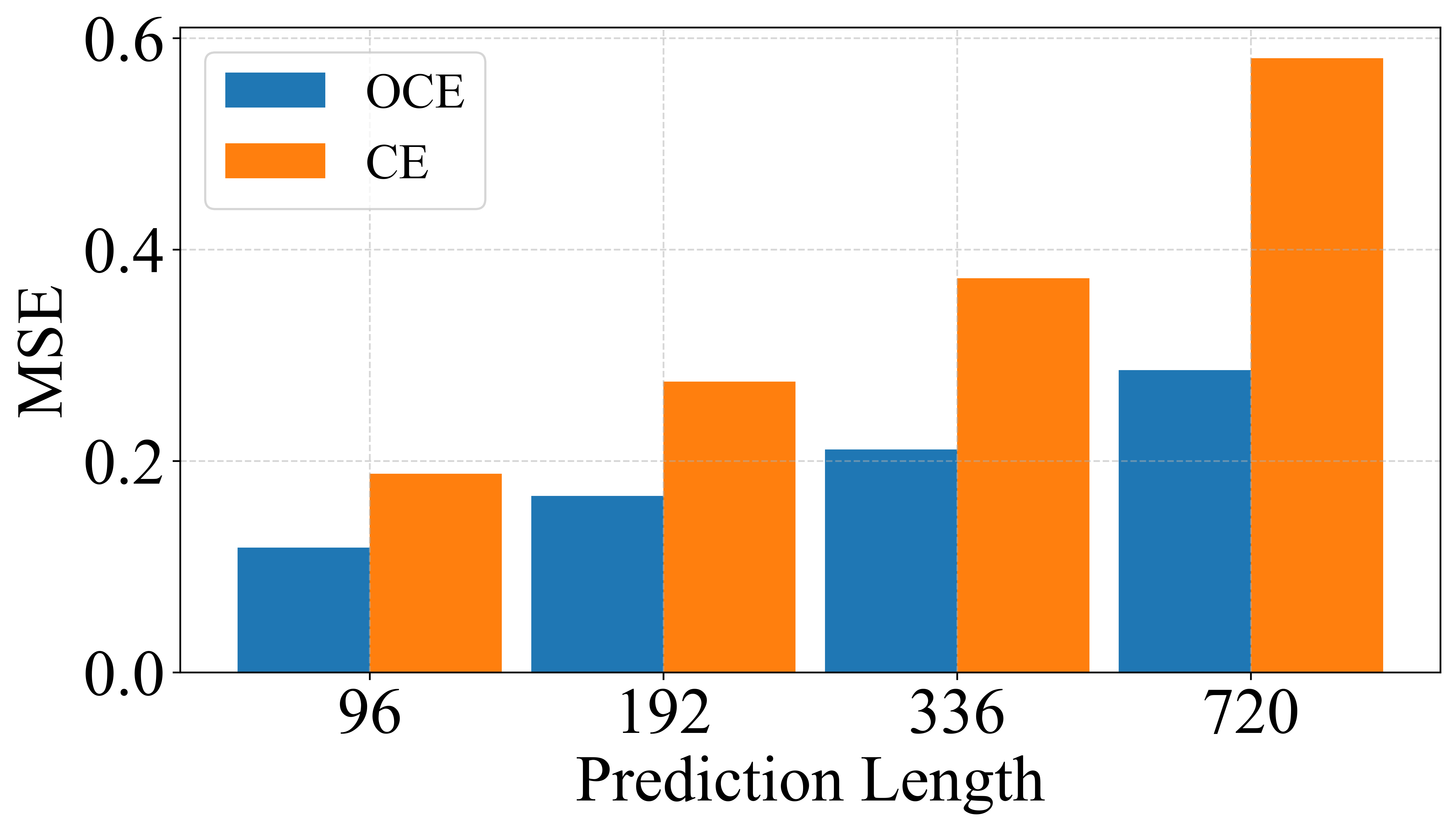}
		\caption{MSE (ETTm2)}
		\label{fig:ettm2_mse}
	\end{subfigure}
	\hfill
	\begin{subfigure}[t]{0.23\textwidth}
		\centering
		\includegraphics[width=\linewidth]{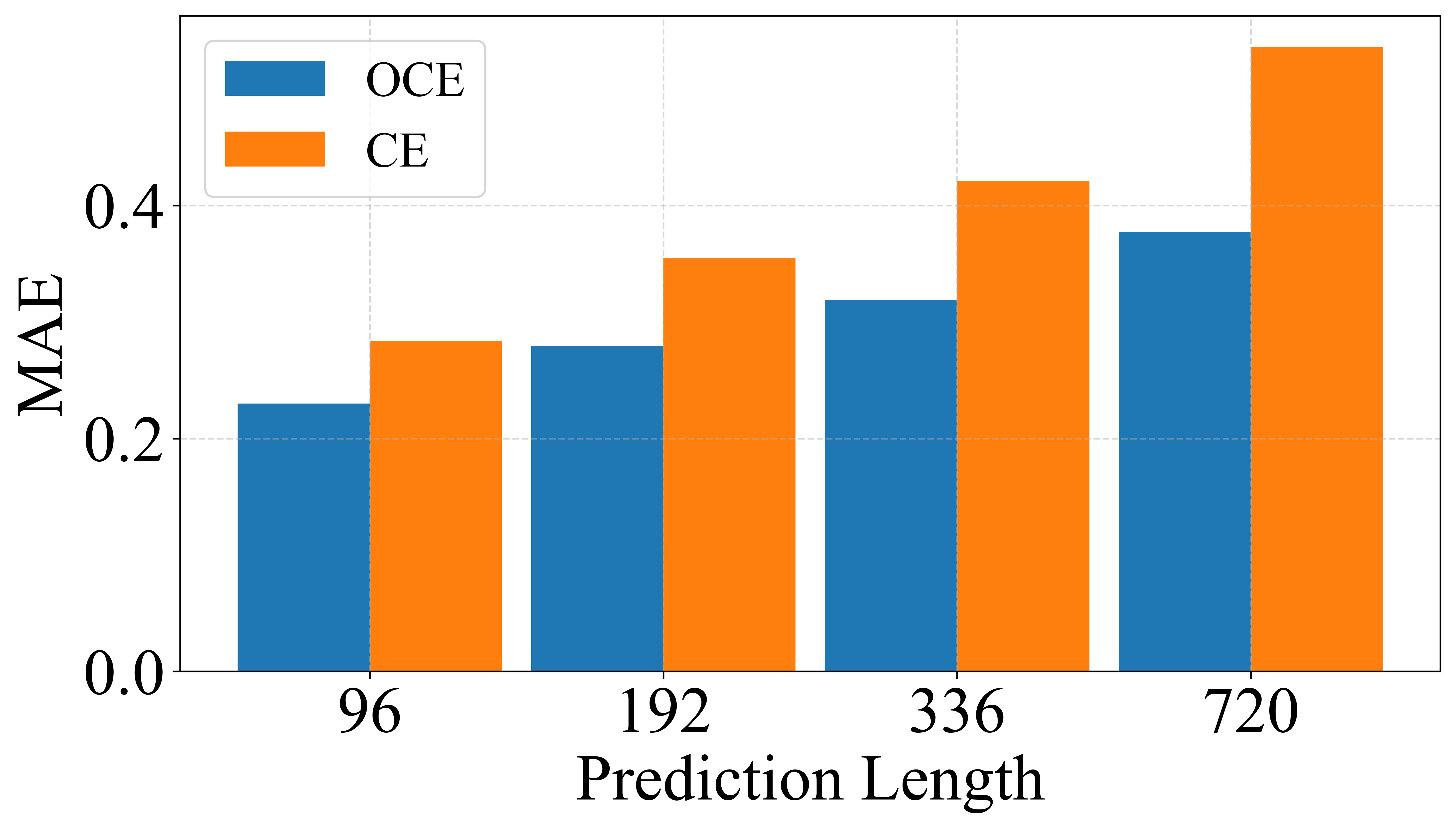}
		\caption{MAE (ETTm2)}
		\label{fig:ettm2_mae}
	\end{subfigure}
	
	\vspace{0.5em}
	
	% 第三行（4个子图）
	\begin{subfigure}[t]{0.23\textwidth}
		\centering
		\includegraphics[width=\linewidth]{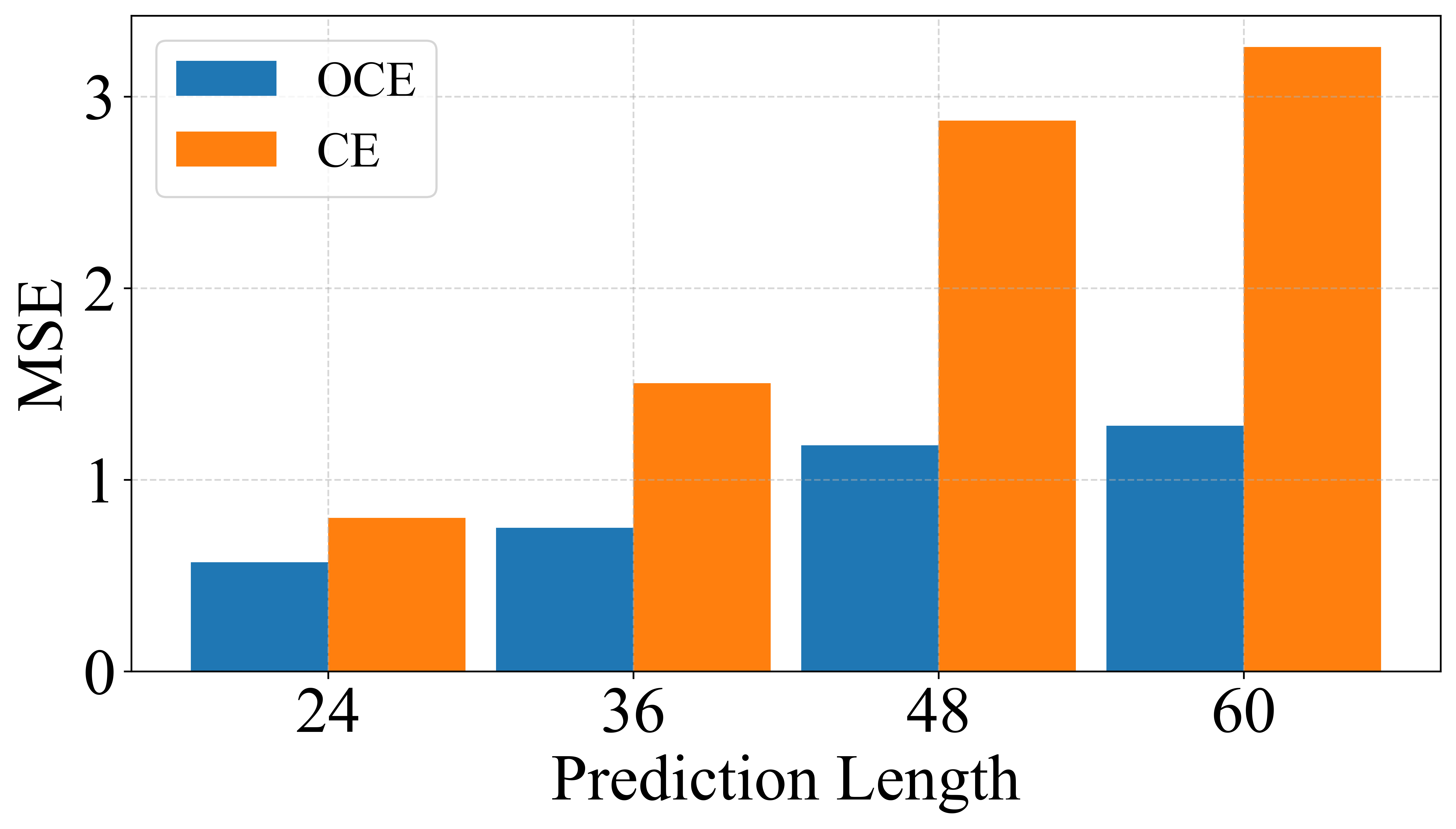}
		\caption{MSE (ILI)}
		\label{fig:ili_mse}
	\end{subfigure}
	\hfill
	\begin{subfigure}[t]{0.23\textwidth}
		\centering
		\includegraphics[width=\linewidth]{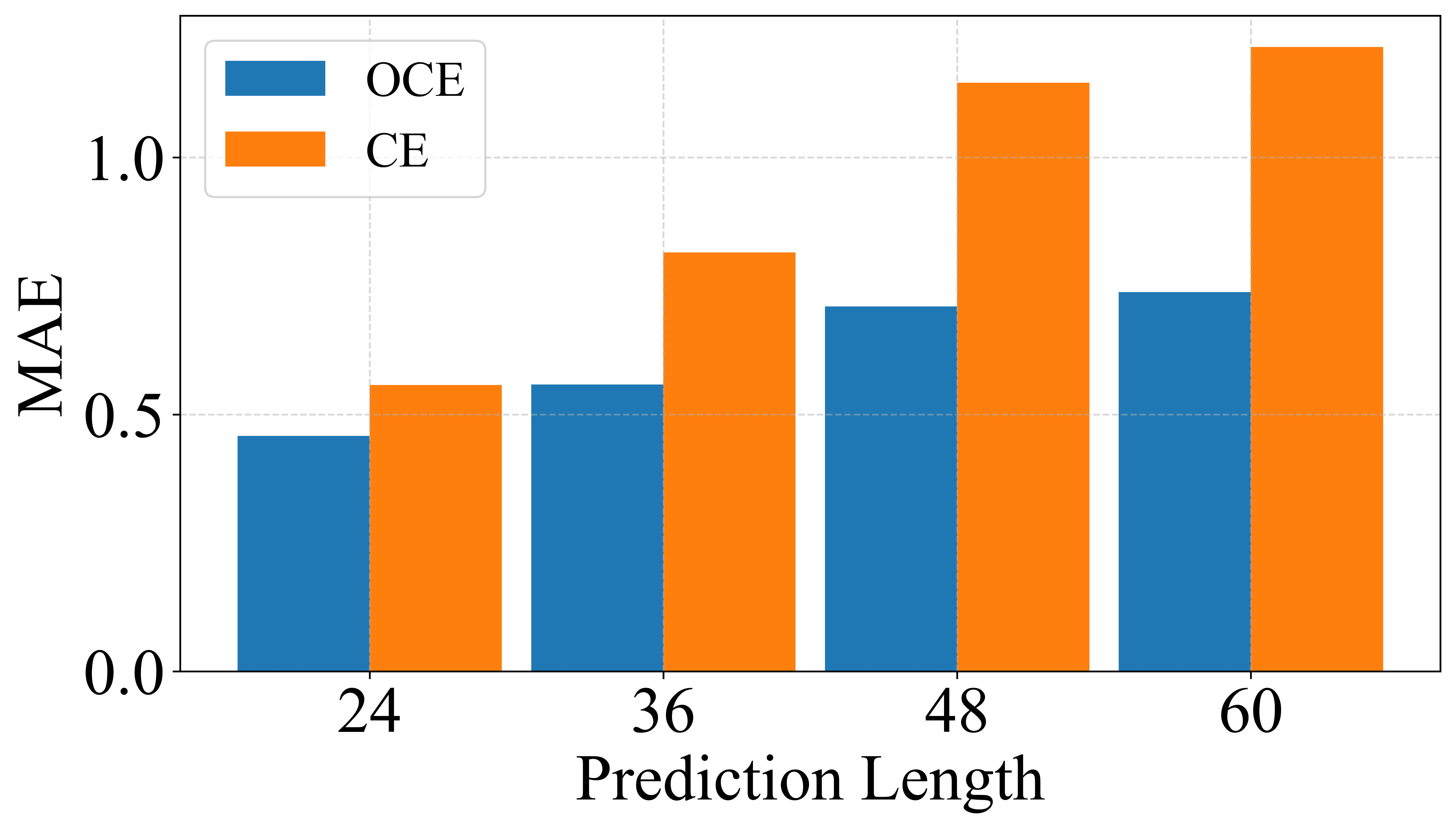}
		\caption{MAE (ILI)}
		\label{fig:ili_mae}
	\end{subfigure}
	\hfill
	\begin{subfigure}[t]{0.23\textwidth}
		\centering
		\includegraphics[width=\linewidth]{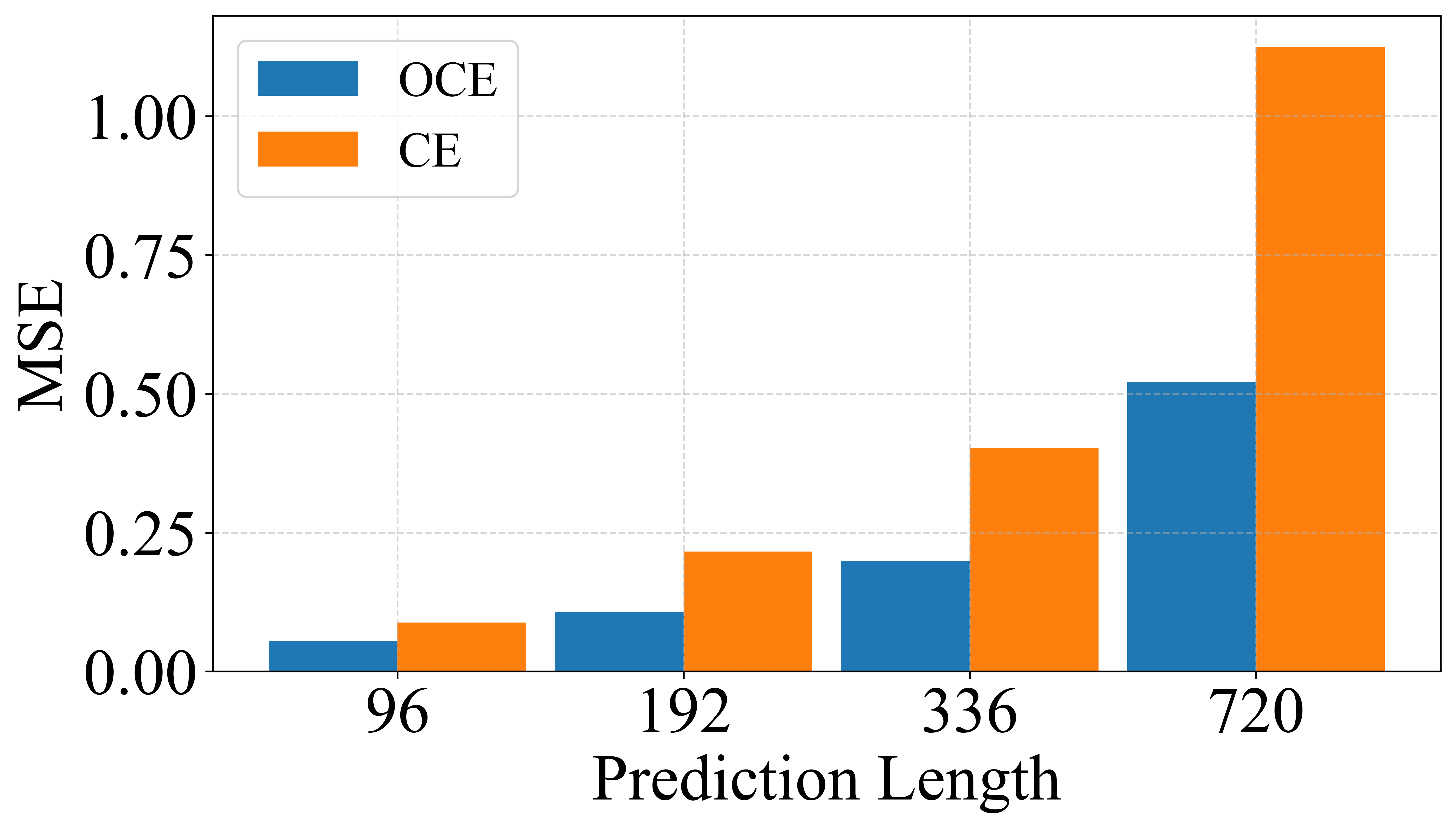}
		\caption{MSE (Exchange)}
		\label{fig:exchange_mse}
	\end{subfigure}
	\hfill
	\begin{subfigure}[t]{0.23\textwidth}
		\centering
		\includegraphics[width=\linewidth]{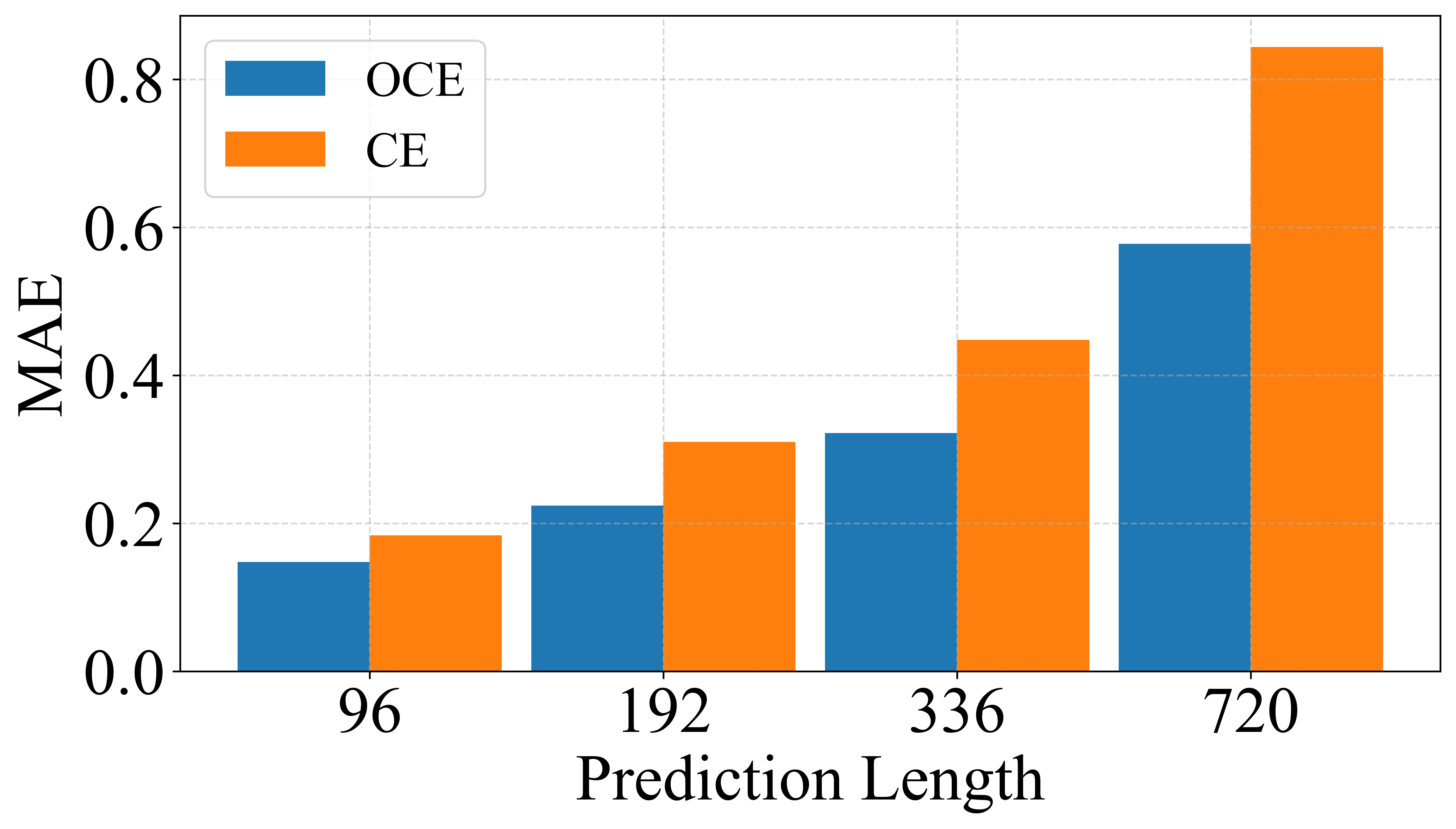}
		\caption{MAE (Exchange)}
		\label{fig:exchange_mae}
	\end{subfigure}
	
	\caption{Performance Comparison Between CE and OCE Loss Functions Across Six Datasets. }
	\label{fig:ce-oce-comparison-full}
\end{figure*}

\begin{table*}[!ht]
	\centering
	\scriptsize
	\setlength{\tabcolsep}{2.5pt}
	\renewcommand{\arraystretch}{1.02}
	\begin{tabular}{l|c|cccc|cccc|cccc|cccc|cccc}
		\toprule
		\multirow{2}{*}{\textbf{Dataset}} & \multirow{2}{*}{\textbf{Pred}} & \multicolumn{4}{c}{\textbf{-3dB}} & \multicolumn{4}{c}{\textbf{0dB}} & \multicolumn{4}{c}{\textbf{3dB}} & \multicolumn{4}{c}{\textbf{10dB}} & \multicolumn{4}{c}{\textbf{20dB}} \\
		\cmidrule(lr){3-6} \cmidrule(lr){7-10} \cmidrule(lr){11-14} \cmidrule(lr){15-18} \cmidrule(lr){19-22}
		& & \multicolumn{2}{c}{DLinear} & \multicolumn{2}{c}{Ours} & \multicolumn{2}{c}{DLinear} & \multicolumn{2}{c}{Ours} & \multicolumn{2}{c}{DLinear} & \multicolumn{2}{c}{Ours} & \multicolumn{2}{c}{DLinear} & \multicolumn{2}{c}{Ours} & \multicolumn{2}{c}{DLinear} & \multicolumn{2}{c}{Ours} \\
		& & MSE & MAE & MSE & MAE & MSE & MAE & MSE & MAE & MSE & MAE & MSE & MAE & MSE & MAE & MSE & MAE & MSE & MAE & MSE & MAE \\
		\midrule
		\multirow{4}{*}{\textbf{ETTm2}} & 96 & 0.186 & 0.284 & \textbf{0.117} & \textbf{0.229} & 0.182 & 0.280 & \textbf{0.116} & \textbf{0.228} & 0.176 & 0.274 & \textbf{0.116} & \textbf{0.227} & 0.172 & 0.269 & \textbf{0.118} & \textbf{0.229} & 0.172 & 0.268 & \textbf{0.118} & \textbf{0.230} \\
		& 192 & 0.255 & 0.338 & \textbf{0.163} & \textbf{0.278} & 0.246 & 0.328 & \textbf{0.162} & \textbf{0.277} & 0.241 & 0.323 & \textbf{0.162} & \textbf{0.276} & 0.235 & 0.317 & \textbf{0.163} & \textbf{0.277} & 0.235 & 0.316 & \textbf{0.164} & \textbf{0.277} \\
		& 336 & 0.285 & 0.345 & \textbf{0.207} & \textbf{0.318} & 0.280 & 0.340 & \textbf{0.206} & \textbf{0.317} & 0.277 & 0.337 & \textbf{0.206} & \textbf{0.317} & 0.276 & 0.336 & \textbf{0.207} & \textbf{0.317} & 0.277 & 0.337 & \textbf{0.209} & \textbf{0.318} \\
		& 720 & 0.440 & 0.453 & \textbf{0.275} & \textbf{0.373} & 0.431 & 0.446 & \textbf{0.274} & \textbf{0.372} & 0.425 & 0.442 & \textbf{0.273} & \textbf{0.371} & 0.420 & 0.438 & \textbf{0.277} & \textbf{0.372} & 0.420 & 0.437 & \textbf{0.280} & \textbf{0.374} \\
		\midrule
		\multirow{4}{*}{\textbf{Weather}} & 96 & 0.186 & 0.253 & \textbf{0.109} & \textbf{0.149} & 0.181 & 0.247 & \textbf{0.109} & \textbf{0.147} & 0.178 & 0.242 & \textbf{0.108} & \textbf{0.145} & 0.176 & 0.237 & \textbf{0.107} & \textbf{0.145} & 0.175 & 0.236 & \textbf{0.107} & \textbf{0.144} \\
		& 192 & 0.224 & 0.285 & \textbf{0.139} & \textbf{0.193} & 0.226 & 0.291 & \textbf{0.138} & \textbf{0.191} & 0.224 & 0.288 & \textbf{0.137} & \textbf{0.190} & 0.221 & 0.284 & \textbf{0.136} & \textbf{0.189} & 0.221 & 0.283 & \textbf{0.136} & \textbf{0.189} \\
		& 336 & 0.272 & 0.330 & \textbf{0.173} & \textbf{0.232} & 0.269 & 0.326 & \textbf{0.171} & \textbf{0.230} & 0.267 & 0.324 & \textbf{0.170} & \textbf{0.229} & 0.266 & 0.321 & \textbf{0.168} & \textbf{0.228} & 0.266 & 0.321 & \textbf{0.169} & \textbf{0.228} \\
		& 720 & 0.342 & 0.388 & \textbf{0.227} & \textbf{0.283} & 0.339 & 0.384 & \textbf{0.225} & \textbf{0.281} & 0.338 & 0.382 & \textbf{0.224} & \textbf{0.280} & 0.337 & 0.381 & \textbf{0.223} & \textbf{0.279} & 0.338 & 0.381 & \textbf{0.224} & \textbf{0.280} \\
		\bottomrule
	\end{tabular}
	\caption{MSE and MAE Comparison of DLinear and OCE-TS Under Different Noise Levels}
	\label{Table_noise}		
\end{table*}

\subsection{Lookback Window Sizes}
\label{app:Lookback}
Figure~\ref{fig:lookback-all} presents the prediction performance across seven different datasets under varying lookback window sizes.
For the ILI dataset, prediction lengths are set as $\{24, 36, 48, 60\}$ and lookback window sizes are selected from $\{52, 78, 104, 130, 156, 208\}$. For the remaining datasets, prediction lengths are chosen from $\{96, 192, 336, 720\}$, with lookback windows drawn from $\{48, 96, 192, 336, 504, 720\}$.

The experimental results demonstrate that the impact of different lookback window sizes on prediction performance exhibits significant dataset dependency. For most datasets including ETT and Weather, increasing the lookback window effectively reduces prediction errors (MSE / MAE), indicating that longer historical data helps capture long-term dependencies, though with diminishing returns - for instance, the ETTh series shows limited improvement when the window exceeds 336. In contrast, more volatile datasets like Exchange and period-sensitive datasets like ILI show greater sensitivity to window size, requiring case-specific optimization. Furthermore, short-term predictions (e.g., 96 / 192) generally outperform long-term forecasts (e.g., 336 / 720). While expanding the lookback window typically enhances performance, it necessitates balancing computational costs against diminishing returns. We recommend using medium-sized windows (192 / 336) as a baseline for specific data characteristics, followed by experimental fine-tuning. These findings provide guidance for balancing historical information utilization in time series forecasting.

\subsection{Parameter Sensitivity}
\label{app:Parameter}
The OCE-TS method incorporates two critical hyperparameters: the number of discrete bins (\textit{bin}) and standard deviation ($\sigma$). Our comprehensive parameter sensitivity analysis examines the complete parameter space defined by $\textit{bin} \in \{30,45,75,100\}$ and $\sigma \in \{0.001,0.01,0.1,1\}$, with all other experimental configurations held constant relative to the baseline model. To precisely evaluate individual parameter effects, we implement a controlled variable approach consisting of two analytical phases: (1) performance impact assessment of \textit{bin} variations with fixed $\sigma=0.01$, followed by (2) regulatory effect analysis of $\sigma$ adjustments with fixed $\textit{bin}=100$. This systematic parameter investigation is conducted on the ETT time series dataset collection, employing both MSE and MAE as primary evaluation metrics.

\begin{figure}[!ht]
	\centering
	\includegraphics[width=\linewidth]{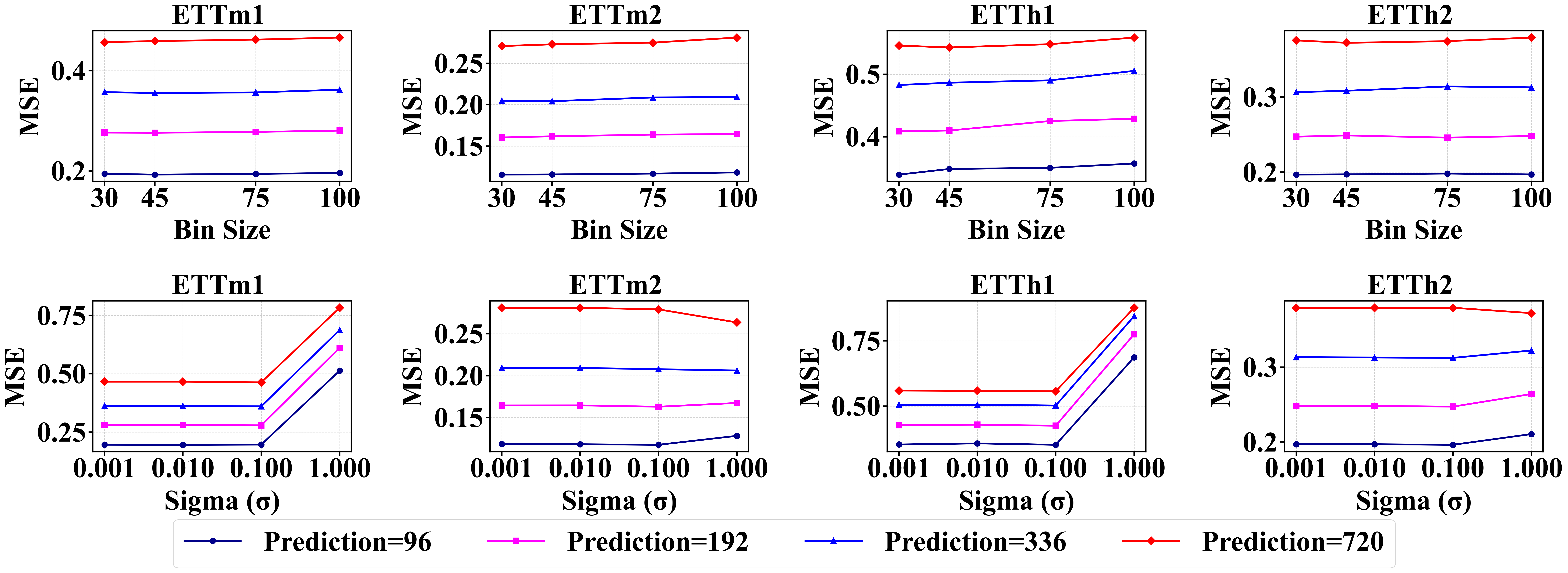}
	\caption{Comparison of MSE performance across different bin sizes (top row) and sigma values (bottom row) for four datasets. The plots show results for prediction lengths 96, 192, 336, and 720.}
	\label{fig:results}
\end{figure}

Experimental results demonstrate that the adjustment of bin size has relatively limited impact on model performance, whereas the value of standard deviation ($\sigma$) significantly affects the prediction results. The default parameter setting of $\text{bin}=75$ and $\sigma=0.01$ achieves consistently strong performance across most datasets and prediction lengths. It is recommended to first fix $\sigma=0.01$ and tune $\text{bin}$ within the range $[75, 100]$. For the ETTh series or long-term forecasting tasks, increasing $\text{bin}$ to 100 may yield better results. If performance fluctuations are observed, further fine-tuning $\sigma$ within the range $[0.005, 0.05]$ is suggested to enhance stability and robustness.

\subsection{Comparison of Loss Functions}
\label{app:Loss}
To provide a more comprehensive comparison between the CE and OCE loss functions, Figure~\ref{fig:ce-oce-comparison-full} presents the complete visualization results across six benchmark datasets under both MSE and MAE metrics.

The experimental results demonstrate that the proposed OCE loss function (marked in blue) significantly outperforms the baseline CE loss function (marked in yellow) on both MSE and MAE metrics. This performance advantage shows consistent improvement across all benchmark datasets, validating the superiority of the OCE loss function for time series forecasting tasks.

\subsection{Model Performance under Different SNRs}
\label{app:SNR}
To assess model robustness under noisy conditions, we conduct experiments by injecting additive Gaussian white noise with variance $\sigma^2 = P_{\text{signal}}/10^{\text{SNR}_{\text{dB}}/10}$ into inputs, evaluating performance across five SNR levels: \SI{-3}{dB}, \SI{0}{dB}, \SI{3}{dB}, \SI{10}{dB}, and \SI{20}{dB}. Here, $P_{\text{signal}}$ is the average signal power, computed as the signal’s mean squared value, used to set the noise variance $\sigma^2$ for the target SNR. The comparative analysis between our method and DLinear is performed on both ETTm2 and Weather datasets, with comprehensive results across multiple prediction lengths summarized in Table~\ref{Table_noise} (MSE / MAE metrics).

As the noise intensity increases (i.e., lower SNR), the performance of DLinear degrades significantly. In contrast, our method maintains consistently lower MSE and MAE across all noise levels and forecast horizons, demonstrating stronger noise robustness.

\subsection{Significance Analysis}  \label{app:siga}
The Nemenyi post-hoc test determines whether two algorithms differ significantly by comparing their average rank difference to the critical distance:
\begin{equation}
	CD = q_\alpha \sqrt{\frac{k(k+1)}{6N}},
\end{equation}
where $q_\alpha$ is the critical value from the Studentized range distribution at significance level $\alpha$, $k$ is the number of compared algorithms, and $N$ is the number of datasets. The performance difference is considered statistically significant when the observed rank difference exceeds $CD$.

The experiments compared OCE-TS with five baseline methods including Autoformer, Dlinear, and TimeBridge across seven benchmark datasets, using both MAE and MSE as evaluation metrics. We set the significance level at 0.05 for this analysis.

%\bibliography{aaai2026}

\end{document}